\documentclass[a4paper,10pt]{article}
\usepackage{stmaryrd}
\usepackage{amsfonts}
\usepackage{bbm}
\usepackage{amscd}
\usepackage{mathrsfs}
\usepackage{latexsym,amssymb,amsmath,amscd,amscd,amsthm,amsxtra}
\usepackage[dvips]{graphicx}
\usepackage[utf8]{inputenc}
\usepackage[T1]{fontenc}
\usepackage{lmodern}
\usepackage{amssymb}
\usepackage[all]{xy}
\usepackage{nicefrac,mathtools,enumitem}
\usepackage{microtype}
\usepackage{CJK}
\usepackage{amssymb}
\usepackage{epstopdf}
\usepackage{graphicx}
\usepackage{graphics}
\usepackage[T1]{fontenc}
\usepackage[utf8]{inputenc}
\usepackage{authblk}
\usepackage{subfigure}
\usepackage{caption}
\usepackage{framed}
\usepackage{mathrsfs}
\usepackage[noend]{algpseudocode}
\usepackage{algorithmicx,algorithm}
\usepackage{framed}
\usepackage{booktabs}
\usepackage{threeparttable}
\usepackage{color}

\usepackage{hyperref}

\numberwithin{equation}{section}
\newtheorem{thm}{Theorem}
\newtheorem{lem}{Lemma}

\newtheorem{pro}{Proposition}

\newtheorem{rmk}{Remark}

\newcommand {\emptycomment}[1]{}

\newcommand{\be }{\begin{equation}}
\newcommand{\ee }{\end{equation}}

\allowdisplaybreaks

\def\bea{\begin{eqnarray}}
\def\eea{\end{eqnarray}}
\def\be{\begin{equation}}
\def\ee{\end{equation}}
\def\blm{\begin{lem}}
\def\elm{\end{lem}}

\def\bea{\begin{eqnarray}}
	\def\eea{\end{eqnarray}}
\def\be{\begin{equation}}
	\def\ee{\end{equation}}
\def\blm{\begin{lem}}
	\def\elm{\end{lem}}
\def\bes{\begin{eqnarray*}}
	\def\ees{\end{eqnarray*}}

\begin{document}
\title{
{Estimates on Learning Rates for Multi-Penalty Distribution Regression}
 }\vspace{2mm}
\date{June 2020 (first version)}
\author{Zhan Yu\thanks{Corresponding author}\\ \small Department of Mathematics, Hong Kong Baptist University\\ \small 224 Waterloo Road, Kowloon Tong, Hong Kong\\ \small  Email: zhanyu@hkbu.edu.hk; mathyuzhan@gmail.com  % "\small" is optional
	\and Daniel W. C. Ho\\ \small Department of Mathematics, City University of Hong Kong\\ 
	\small 83 Tat Chee Avenue. Kowloon Tong, Hong Kong\\
	\small  Email: madaniel@cityu.edu.hk
}

%\author{ Zhan Yu\thanks{Corresponding author}, Daniel W. C. Ho \emph{ }\emph{ }}
%\date{}
%\affil[]{Department of Mathematics, City University of Hong Kong}
%\affil[]{Kowloon, Hong Kong}
%\affil[]{zhanyu2-c@my.cityu.edu.hk}
%\affil[]{madaniel@cityu.edu.hk}

\maketitle

	\begin{abstract}
	This paper is concerned with functional learning by utilizing two-stage
	sampled distribution regression. We study a multi-penalty regularization
	algorithm for distribution regression in the framework of learning theory.
	The algorithm aims at regressing to real-valued outputs from probability
	measures. The theoretical analysis of distribution regression is far from
	maturity and quite challenging since only second-stage samples are observable
	in practical settings. In our algorithm, to transform information of distribution
	samples, we embed the distributions to a reproducing kernel Hilbert space
	$\mathcal{H}_{K}$ associated with Mercer kernel $K$ via mean embedding
	technique. One of the primary contributions of this work is the introduction
	of a novel multi-penalty regularization algorithm, which is able to capture
	more potential features of distribution regression. Optimal learning rates
	of the algorithm are obtained under mild conditions. The work also derives
	learning rates for distribution regression in the hard learning scenario
	$f_{\rho }\notin \mathcal{H}_{K}$, which has not been explored in the existing
	literature. Moreover, we propose a new distribution-regression-based distributed
	learning algorithm to face large-scale data or information challenges arising
	from distribution data. The optimal learning rates are derived for the
	distributed learning algorithm. By providing new algorithms and showing
	their learning rates, the work improves the existing literature in various
	aspects.
\end{abstract}
%
%\begin{abstract}[class=graphical]
%\begin{figure} \includegraphics{<aid>fab} \end{figure} \abstext{}
%\end{abstract}
%
%\begin{abstract}[class=author-highlights,title=Highlights]
%\begin{itemize} %\item
%\end{itemize}
%\end{abstract}
%
%%% Options: [class=MSC|KWD]
%\begin{keyword}[class=MSC]
%\kwd{}\kwd{}
%\end{keyword}
%
%\begin{keyword}
%\kwd{???}%\kwd{}\kwd{}
%\end{keyword}

%
%%*************** Text entry area ******************%%
%s1 #&#
\section{Introduction}
\label{sec1}

In the era of big data, functional data or matrix-valued data have become
increasingly prevalent in practical applications within machine learning
and statistics. Instead of scalar data settings, these types of data present
more challenges and difficulties in handling intricate information. As
a result, there is a need for improved regression methods to effectively
address these issues. The development of a suitable regression scheme for
these types of problems has become a highly desirable goal. Recently, a
regression method called distribution regression has emerged as an effective
tool to handle data defined on some appropriate Banach spaces (\cite{fang},
\cite{dr2}, \cite{twostage}, \cite{only}). Specifically, the input data
are (probability) distributions on a compact metric space
$\widetilde{X}$. The proposed method contains two sampling stages to learn
the regressor from the distributions to the real-valued outputs. The first-stage
sample is made up of distributions, and the second-stage sample is drawn
according to these distributions. For the first stage, we define the data
set as
$\widetilde{D}=\{(x_{i},y_{i})\}_{i=1}^{|D|}\subset X\times Y$, in which
$|D|$ is the cardinality of $\widetilde{D}$, and each pair
$(x_{i},y_{i})$ is i.i.d. sampled from a meta distribution. $X$ is the
input space of probability measures on $\widetilde{X}$, and
$Y=\mathbb R$ is the output space equipped with the standard Euclidean
metric. For the second stage, the elements in the sample set
$\widehat{D}=\big\{(\{x_{i,s}\}_{s=1}^{d_{i}},y_{i})\big\}_{i=1}^{|D|}$
are obtained from the distributions $\{x_{i}\}_{i=1}^{|D|}$ accordingly,
where $x_{i,j}\in \widetilde{X}$. Many important machine learning and statistical
operations have been found to be directly related to the distribution regression
scheme mentioned above, for example, the multi-instance learning and point
estimation problems without analytical solution.

In this paper, we consider a mean-embedding-based ridge regression method
for distribution regression. Let $H=H(\widetilde{K})$ be a reproducing
kernel Hilbert space (RKHS) with
$\widetilde{K}:\widetilde{X}\times \widetilde{X}\rightarrow \mathbb R$
as the reproducing kernel. Let $(\widetilde{X},\mathcal{F})$ be a measurable
space with $\mathcal{F}$ being a Borel $\sigma $-algebra on
$\widetilde{X}$. Denote the set of Borel probability measures on
$(\widetilde{X},\mathcal{F})$ by $\mathcal{M}_{1}(\mathcal{F})$. Then the
\emph{mean embedding} of a distribution
$x\in \mathcal{M}_{1}(\mathcal{F})$ to an element $\mu _{x}$ of RKHS
$H$ is given by
%
%e1.1 #&#
\begin{equation}
\nonumber
\mu _{x}=\int _{\widetilde{X}}\widetilde{K}(\cdot ,\xi )dx(\xi ).
\end{equation}
In practical applications, a class of kernels of great interest are the characteristic kernels (e.g. \cite{fang}, \cite{only}), of which the mean embedding maps are injective.  Some well-known universal kernels are shown to be characteristic such as Gaussian RBF kernel, exponential kernel and binomial kernel. Denote the set of the mean embeddings by
$X_{\mu}=\{\mu _{x}:x\in \mathcal{M}_{1}(\mathcal{F})\}\subseteq H$ and the mean embeddings of $\widetilde{D}$ to $X_{\mu}$ by
$D=\{(\mu _{x_{i}},y_{i})\}_{i=1}^{|D|}$. Let $\rho $ be the $\mu $-induced
probability measure on the product space $Z=X_{\mu}\times Y$. The regression
function with respect to the measure $\rho $ is defined by
%
%e1.2 #&#
\begin{equation}
f_{\rho}(\mu _{x})=\int _{Y}yd\rho (y|\mu _{x}), \ \mu _{x}\in X_{\mu},
\label{eq1.1}
\end{equation}
in which $\rho (\cdot |\mu _{x})$ is the conditional probability measure
of $\rho $ induced at $\mu _{x}\in X_{\mu}$. $f_{\rho}$ is just the minimizer
of the least square problem
%
%e1.3 #&#
\begin{equation}
\nonumber
\mathcal{E}(f)=\int _{Z}(f(\mu _{x})-y)^{2}d\rho .
\end{equation}
Generally, the measure $\rho $ is unknown. In the scenario of distribution
regression, the distributions $\{x_{i}\}_{i=1}^{|D|}$ are still unknown,
we are only able to access their information by the random sample
\begin{equation*}
\widehat{D}=\big\{(\{x_{i,s}\}_{s=1}^{d_{i}},y_{i})\big\}_{i=1}^{|D|}
\end{equation*}
with size $d_{i}\in \mathbb N$, $i=1,2,...,|D|$ respectively.  In the
context of distribution regression, in a reproducing
kernel Hilbert space $(\mathcal{H}_{K},\|\cdot \|_{K})$ associated with
a Mercer kernel $K:X_{\mu}\times X_{\mu}\rightarrow \mathbb R$, the conventional regularization approach
follows the functional optimization scheme
%
%e1.4 #&#
\begin{equation}
f_{\widehat{D},\lambda }=\arg \min _{f\in \mathcal{H}_{K}}\bigg\{
\frac{1}{|D|}\sum _{i=1}^{|D|}\big(f(\mu _{\hat{x}_{i}})-y_{i}\big)^{2}+
\lambda \big\|f\big\|_{K}^{2}\bigg\},
\label{alone}
%%LEAP%%%\label{eq1.2}
\end{equation}
in which
$
\hat{x}_{i}=\frac{1}{d_{i}}\sum _{s=1}^{d_{i}}\delta _{x_{i,s}}
$
serves as the empirical distribution determined by the observable quantity
$\widehat{D}=\big\{(\{x_{i,s}\}_{s=1}^{d_{i}},y_{i})\big\}_{i=1}^{|D|}$,
$\mu _{\hat{x}_{i}}=\frac{1}{d_{i}}\sum _{s=1}^{d_{i}}\widetilde{K}(
\cdot ,x_{i,s})$ is the mean embedding of $\hat{x}_{i}$, and $\lambda >0$ is the regularization
parameter. {\eqref{alone}} is essentially a Tikhonov regularized scheme in
RKHS. It is an extension of the one-stage kernel ridge regression scheme.

In recent years, there has been extensive research focused on the setting
of one-stage-sampling regularized least squares algorithm (\cite{s1},
\cite{s2}, \cite{multiguo}, \cite{s3}, \cite{sbsecond}, \cite{sb3},
\cite{mt5}, \cite{s4}, \cite{s5}, \cite{s6}). In the one-stage setting,
learning rates of regularized least squares schemes have been thoroughly
investigated using kernel method in learning theory (for kernel methods,
refer to e.g., \cite{s1}, \cite{s3}, \cite{s4}). However, there is still a need for
further exploration of regularized least squares algorithms with two sampling
stages. The existing theoretical analysis on learning rates for distribution
regression is only provided in \cite{fang} and \cite{only}. These works
demonstrate that the optimal minimax learning rates for Algorithm {\eqref{alone}} can be obtained through an integral operator approach under
certain mild conditions. It can be observed that, the above studies on
distribution regression are still confined to a single-penalty category,
where only one regularization term appears in Algorithm {\eqref{alone}} with
a corresponding single parameter $\lambda $. On the other hand, the multi-penalty
regularization, as a stable and robust method in one stage-sampling kernel
ridge regression setting (e.g.~\cite{multi1}, \cite{mt1},
\cite{multiguo}, \cite{mt2}, \cite{mt3}, \cite{mt4}, \cite{mt5},
\cite{mt6}), has been demonstrated to possess many advantages over single-penalty
regularization. In contrast to single penalty regularization, it can incorporate
any prior information into additional penalties. It thus can simultaneously
include various features in regularized solutions, such as boundedness,
monotonicity, and smoothness. With these merits, the multi-penalty regularization
scheme has been widely used in a variety of inspiring applications like
image reconstruction, earth gravity potential reconstruction, option pricing
models, data detection, and sparsity analysis. In the context of learning
theory, the learning rates of multi-penalty Tikhonov regularization have
been primarily studied in \cite{multi1} and \cite{multiguo}. However, the
exploration of the multi-penalty regularization scheme in the literature
on distribution regression remains incomplete and requires further investigation.
Inspired by the aforementioned advantages of the multi-penalty regularization
scheme, we aim to explore its performance in a distribution regression
setting to capture more potential features of distribution regression.

In this paper, we investigate a more general framework in two-stages distribution
regression by considering the novel multi-penalty regularization scheme:
%
%e1.5 #&#
\begin{equation}
f_{\widehat{D},\lambda _{1},\lambda _{2}}=\arg \min _{f\in
	\mathcal{H}_{K}}\bigg\{\frac{1}{|D|}\sum _{i=1}^{|D|}\big(f(\mu _{
	\hat{x}_{i}})-y_{i}\big)^{2}+\lambda _{1}\big\|f\big\|_{K}^{2}+
\lambda _{2}\big\|V_{D}f\big\|_{K}^{2}\bigg\}.
\label{algorithm}
%%LEAP%%%\label{eq1.3}
\end{equation}
In {\eqref{algorithm}},
$V_{D}:\mathcal{H}_{K}\rightarrow \mathcal{H}_{K}$ is a bounded linear
operator which may depend on the data set $\widetilde{D}$ and its related first-stage
mean embedding set $D$. The goal of distribution regression in this paper
is to learn the regression function by utilizing Algorithm {\eqref{algorithm}} based on
the given training samples
$\widehat{D}=\big\{(\{x_{i,j}\}_{j=1}^{d_{i}},y_{i})\big\}_{i=1}^{|D|}$
with $x_{i,1},x_{i,2},...,x_{i,d_{i}}\sim x_{i}$ (i.i.d.). This paper aims
to investigate the learning rates of multi-penalty distribution regression {\eqref{algorithm}}. By deriving optimal learning rates for the proposed
multi-penalty distribution regression scheme via an integral operator approach,
we improve the results in the literature. On the other hand, in the real
world, with the development of data mining, large-scale data are always
collected in various application domains, including financial engineering,
medicine, business analysis, personal social network, sensor network, and
monitoring. In these applications, sensitive data, such as personal data,
are always trained in machine learning for different requirements. Hence,
it is essential to protect data privacy. In recent years, distributed learning
has been shown to be a powerful strategy for tackling privacy-preserving
problems (e.g. \cite{s2}, \cite{multiguo}, \cite{s3}, \cite{sbsecond}, \cite{sb3}, \cite{shi2019}). On the other hand, the unprecedentedly large data size and complexity
of distribution samples would always raise the difficulties of computing
in the distribution regression approach. Large-scale data would add unpredictable
storage burdens and memory capacity for a single machine. Meanwhile, it
would take a huge amount of time for a single machine to process scientific
computing on distribution regression. Motivated by these challenges, and
building upon the proposed multi-penalty distribution regression algorithm,
we introduce a novel distributed learning method to address these difficulties
in the realm of distribution regression. The distributed learning method
with multi-penalty distribution regression in this paper is based on a
divide-and-conquer approach. Specifically, for the given data set
$\widetilde{D}$, our distributed learning algorithm starts with partitioning
the mean-embedding data set $D$ (related to $\widetilde{D}$) and its associated
$\widehat{D}$ into $m$ disjoint sub-datasets $\{D_{j}\}_{j=1}^{m}$ and
$\{\widehat{D}_{j}\}_{j=1}^{m}$ with corresponding disjoint union
\begin{eqnarray*}
D=\bigcup _{i=1}^{m}D_{j}, \ \widehat{D}=\bigcup _{i=1}^{m}
\widehat{D}_{j}.
\end{eqnarray*}
Subsequently, we assign the corresponding local mean-embedding data set
$D_{j}$ and its associated second-stage local sample set
$\widehat{D}_{j}$ to one machine (processor) to produce a local estimator
$f_{\widehat{D}_{j},\lambda _{1},\lambda _{2}}$ in RKHS by the multi-penalty
distribution regression scheme {\eqref{algorithm}}. Once these local estimators
are communicated to a central processor, the global estimator
$\overline{f_{\widehat{D},\lambda _{1},\lambda _{2}}}$ is synthesized by
the central processor through the following averaging process
%
%e1.6 #&#
\begin{equation}
\overline{f_{\widehat{D},\lambda _{1},\lambda _{2}}}=\sum _{j=1}^{m}
\frac{|D_{j}|}{|D|}f_{\widehat{D}_{j},\lambda _{1},\lambda _{2}}
\label{disal}
%%LEAP%%%\label{eq1.4}
\end{equation}
using the local estimators
$\left \{f_{\widehat{D}_{j},\lambda _{1},\lambda _{2}}\right \}_{j=1}^{m}$.
For Algorithm {\eqref{disal}}, the learning theory analysis for
$\overline{f_{\widehat{D},\lambda _{1},\lambda _{2}}}$ is conducted through
an integral operator approach.

We summarize some main contributions of the work. In this work, a novel
multi-penalty regularization scheme is proposed for distribution regression.
By incorporating additional penalties, Algorithm {\eqref{algorithm}} becomes
more flexible to capture features of regressor when regressing from distribution
samples. By integral operator approach, we study the learning rates of
Algorithm {\eqref{algorithm}} under different restrictive conditions on regression
function $f_{\rho}$. Optimal learning rates are derived when
$f_{\rho}\in \mathcal{H}_{K}$. In kernel-based learning theory, when the
target function does not belong to the underlying kernel space, the corresponding
problem is often referred to as a hard learning problem (\cite{hard}).
In this paper, we also provide rigorous analysis and derive learning rates
for the hard learning case $f_{\rho}\notin \mathcal{H}_{K}$. To the best
of our knowledge, this paper is the first work to study distribution regression
in a hard learning scenario. The learning rates in this paper improve the
existing achievable results. Finally, based on Algorithm {\eqref{algorithm}}, a novel multi-penalty distribution regression-based
distributed learning algorithm {\eqref{disal}} is proposed. Optimal rates
are achieved for this distributed learning algorithm. We will compare the
main results in this work with those in the literature in Section~\ref{sec3}.

%s2 #&#
\section{Notations, assumptions and main results}
\label{sec2}

We assume that, throughout the paper, there exists a constant $M>0$ such
that $|y|\leq M$ almost surely. $\widetilde{K}$ and $K$ are bounded Mercer
kernel (symmetric, continuous, positive semidefinite) with bounds
$B_{\widetilde{K}}$ and $\kappa $:
%
%e2.1 #&#
\begin{equation}
B_{\widetilde{K}}=\sup _{v\in \widetilde{X}}\widetilde{K}(v,v)<
\infty , \ \kappa =\sup _{\mu _{u}\in X_{\mu}}\sqrt{K(\mu _{u},\mu _{u})}<
\infty .
\label{bddker}
\end{equation}
Suppose that $\alpha \in (0,1]$ and $L>0$. Denote
$\mathcal{L}(Y,\mathcal{H}_{K})$ as the Banach space of the bounded linear
operators from space $Y$ to $\mathcal{H}_{K}$. Denote
$K_{\mu _{x}}=K(\mu _{x},\cdot )$, $\mu _{x}\in X_{\mu}$. We treat
$K_{\mu _{x}}$ as an element of $\mathcal{L}(Y,\mathcal{H}_{K})$ by defining
the linear mapping
%
%e2.2 #&#
\begin{equation}
\nonumber
K_{\mu _{x}}(y)=yK_{\mu _{x}}, \  y\in Y.
\end{equation}
We assume that the mapping
$K_{(\cdot )}:X_{\mu}\rightarrow \mathcal{L}(Y,\mathcal{H}_{K})$ is
$(\alpha ,L)$-H\"older continuous in the following sense
%
%e2.3 #&#
\begin{equation}
\big\|K_{\mu _{x}}-K_{\mu _{y}}\big\|_{\mathcal{L}(Y,\mathcal{H}_{K})}
\leq L\big\|\mu _{x}-\mu _{y}\big\|_{H}^{\alpha }, \ \forall (\mu _{x},
\mu _{y})\in X_{\mu}\times X_{\mu}.
\label{eq2.1}
\end{equation}
Due to the fact in \cite{only}, the set of mean embeddings $X_{\mu}$ is
a separable compact set of the continuous function space defined on $\widetilde{X}$. Denote
$\rho _{X_{\mu}}$ as the marginal distribution of $\rho $ on
$X_{\mu}$. Let $L_{\rho _{X_{\mu}}}^{2}$ be the Hilbert space of square-integrable
functions defined on $X_{\mu}$. For $f\in L_{\rho _{X_{\mu}}}^{2}$, denote
the $L_{\rho _{X_{\mu}}}^{2}$-norm of $f$ by
%
%e2.4 #&#
\begin{equation}
\nonumber
\big\|f\big\|_{\rho}=\big\|f\big\|_{L_{\rho _{X_{\mu}}}^{2}}=
\big\langle f,f\big\rangle _{\rho _{X_{\mu}}}^{1/2}=\bigg(\int _{X_{
		\mu}}\big|f(\mu _{x})\big|^{2}d\rho _{X_{\mu}}(\mu _{x})\bigg)^{1/2}.
\end{equation}
Define the integral operator $L_{K}$ on $L_{\rho _{X_{\mu}}}^{2}$ associated
with the Mercer kernel
$K: X_{\mu}\times X_{\mu}\rightarrow \mathbb R$ as
%
%e2.5 #&#
\begin{equation}
L_{K}(f)=\int _{X_{\mu}}K_{\mu _{x}}f(\mu _{x})d\rho _{X_{\mu}}, \ f
\in L_{\rho _{X_{\mu}}}^{2}.
\label{eq2.2}
\end{equation}
Because the set of mean embeddings $X_{\mu}$ is compact and $K$ is a Mercer
kernel, we know $L_{K}$ is a positive compact operator on
$L_{\rho _{X_{\mu}}}^{2}$. Then for any $r>0$, its $r$-th power
$L_{K}^{r}$ is well-defined according to spectral theorem in functional
calculus. Throughout the paper, we assume the following
\emph{regularity condition} for the regression function $f_{\rho}$:
%
%e2.6 #&#
\begin{equation}
f_{\rho}=L_{K}^{r}(g_{\rho}) \ \text{for some} \ g_{\rho}\in L_{\rho _{X_{
			\mu}}}^{2}, \  r>0.
\label{regularity}
%%LEAP%%%\label{eq2.3}
\end{equation}
The above assumption means that the regression function lies in the range
of operator $L_{K}^{r}$, the special case $r=1/2$ corresponds to
$f_{\rho}\in \mathcal{H}_{K}$. We use the \emph{effective dimension} $\mathcal{N}(\lambda _{1})$ to measure
the capacity of $\mathcal{H}_{K}$ with respect to the measure
$\rho _{X_{\mu}}$, which is defined to be the trace of the operator
$(\lambda _{1}I+L_{K})^{-1}L_{K}$, that is
%
%e2.7 #&#
\begin{equation}
\mathcal{N}(\lambda _{1})=\text{Tr}((\lambda _{1}I+L_{K})^{-1}L_{K}),
\ \lambda _{1}>0.
\label{2.5}
%%LEAP%%%\label{eq2.4}
\end{equation}
For the effective dimension $\mathcal{N}(\lambda _{1})$, we need the following
capacity assumption, which focuses on the $\beta $-rate of
$\mathcal{N}(\lambda _{1})$: there exists a constant
$\mathcal{C}_{0}$ such that for any $\lambda _{1}>0$,
%
%e2.8 #&#
\begin{equation}
\mathcal{N}(\lambda _{1})\leq \mathcal{C}_{0}\lambda _{1}^{-\beta},
\  \text{for some} \ 0<\beta \leq 1.
\label{cap}
\end{equation}

In this paper, we assume that the sample set
$D=\{(\mu _{x_{i}},y_{i})\}_{i=1}^{|D|}$ is drawn independently according
to the probability measure $\rho $. Sample
$\{x_{i,s}\}_{s=1}^{d_{i}}$ is drawn independently according to probability
distribution $x_{i}$ for $i=1,2,...,|D|$. For the operator $V_{D}$, throughout the paper, we assume that there is
a constant $c_{V}>0$ independent of the data set such that
%
%e2.9 #&#
\begin{equation}
\big\|V_{D}^{T}V_{D}\big\|\leq c_{V} \ \ \ \ \ \
\text{a.s.},\
\label{L}
%%LEAP%%%\label{eq2.6}
\end{equation}
in which $V_{D}^{T}$ denotes the adjoint operator of $V_{D}$ and
$\big\|\cdot \big\|$ is the operator norm. Since
$\big \|V_{D}\big \|^{2}= \big\|V_{D}^{T}V_{D}\big\|$, therefore an equivalent
assumption 
\begin{equation}
\nonumber
\big \|V_{D}\big \|\leq \sqrt{c_{V}}, \ \ \ \ \ \
\text{a.s.},
\end{equation}
can also be used
in this paper. In the following, we denote the quantity
$\mathcal{B}_{|D|,\lambda _{1}}$ and
$\mathcal{B}_{|D|,\lambda _{1}}'$ as
%
%e2.10 #&#
\begin{equation}
\mathcal{B}_{|D|,\lambda _{1}}=\frac{2\kappa }{\sqrt{|D|}}\bigg(
\frac{\kappa }{\sqrt{|D|\lambda _{1}}}+\sqrt{\mathcal{N}(\lambda _{1})}
\bigg),
\label{b}
%%LEAP%%%\label{eq2.7}
\end{equation}
%
%e2.11 #&#
\begin{equation}
\mathcal{B}_{|D|,\lambda _{1}}'=\frac{1}{|D|\sqrt{\lambda _{1}}}+
\frac{\sqrt{\mathcal{N}(\lambda _{1})}}{\sqrt{|D|}}.
\label{bp}
%%LEAP%%%\label{eq2.8}
\end{equation}
It can be observed that $\mathcal{B}_{|D|,\lambda _{1}}$ and
$\mathcal{B}_{|D|,\lambda _{1}}'$ only differ in $\kappa $-scaling sense.
$\mathcal{B}_{|D|,\lambda _{1}}$ and
$\mathcal{B}_{|D|,\lambda _{1}}'$ will be used in error upper bound representation
and operator estimates in subsequent analysis.

Consider the estimator $f_{\widehat{D},\lambda _{1},\lambda _{2}}$  defined in Algorithm {\eqref{algorithm}}. We present our main results on error estimate between
$f_{\widehat{D},\lambda _{1},\lambda _{2}}$ and $f_{\rho}$ as detailed
in the following theorems. These results offer estimates on the expected
difference between $f_{\widehat{D},\lambda _{1},\lambda _{2}}$ and
$f_{\rho}$ in $L_{\rho _{X_{\mu}}}^{2}$-norm with the expectation taken
for both $D$ and $\widehat{D}$. The first theorem presents a general error
estimate without imposing decaying restrictions on the effective dimension
$\mathcal{N}(\lambda _{1})$.
%
%t1 #&#
\begin{thm}%
\label{thm1}
Assume that the regularity condition {\eqref{regularity}} holds with
$1/2\leq r\leq 1$, $|y|\leq M$ holds almost surely, the operator boundedness
condition {\eqref{L}} holds and the mapping
$K_{(\cdot )}:X_{\mu}\rightarrow \mathcal{L}(Y,\mathcal{H}_{K})$ is
$(\alpha ,L)$-H\"older continuous with $\alpha \in (0,1]$ and $L>0$. If the regularization parameters
$\lambda _{1}, \lambda _{2}\in (0,1)$ satisfy the relation
$2c_{V}\lambda _{2}=\lambda _{1}^{2r}$ and
$\widetilde{d}$ satisfies
$\frac{1}{\widetilde{d}^{\frac{\alpha }{2}}}=\frac{1}{|D|}\sum _{i=1}^{|D|}
\frac{1}{d_{i}^{\alpha /2}}$, then we have
\begin{eqnarray}
	\nonumber
	&&E\bigg[\big\|f_{\widehat{D},\lambda _{1},\lambda _{2}}-f_{\rho}
	\big\|_{\rho}\bigg]
	\\
	\nonumber
	&&\leq 2\big(2\sqrt{2}(2+\sqrt{\pi})^{\frac{1}{2}}\kappa L
	\frac{2^{\frac{\alpha +2}{2}}B_{\widetilde{K}}^{\frac{\alpha }{2}}}{\lambda _{1}\widetilde{d}^{\frac{\alpha }{2}}}+2
	\big)\bigg(
	\frac{2\mathcal{B}_{|D|,\lambda _{1}}}{\sqrt{\lambda _{1}}}+1\bigg)^{2}
	\frac{2^{\frac{\alpha }{2}}B_{\widetilde{K}}^{\frac{\alpha }{2}}}{\lambda _{1}^{\frac{1}{2}}\widetilde{d}^{\frac{\alpha }{2}}}
	\Bigg[(2+\sqrt{\pi})^{\frac{1}{2}}LM(2\Gamma (3)+\log ^{2}2)
	\\
	\nonumber
	&&\ +2(2\Gamma (5)+\log ^{4}2)^{\frac{1}{2}}\Bigg(2\sqrt{6}(2\Gamma (5)+
	\log ^{4}2)^{\frac{1}{2}}\frac{M}{\kappa }
	\frac{\mathcal{B}_{|D|,\lambda _{1}}}{\sqrt{\lambda _{1}}}\bigg(
	\frac{2\mathcal{B}_{|D|,\lambda _{1}}}{\sqrt{\lambda _{1}}}+1\bigg)+
	\sqrt{3}(\lambda _{1}+\lambda _{2}c_{V})\times
	\\
	\nonumber
	&&\ \ \lambda _{1}^{r-\frac{3}{2}}2^{r-\frac{1}{2}}\|g_{\rho}\|_{\rho}(2
	\Gamma (4r-1)+\log ^{4r-2}2)^{\frac{1}{2}}\bigg(
	\frac{2\mathcal{B}_{|D|,\lambda _{1}}}{\sqrt{\lambda _{1}}}+1\bigg)^{2r-1}+
	\sqrt{3}\kappa ^{r-\frac{1}{2}}\|g_{\rho}\|_{\rho}\Bigg)\Bigg]
	\\
	\nonumber
	&&\ \ +4(2\Gamma (3)+\log ^{2}2)^{\frac{1}{2}}(2\Gamma (5)+\log ^{4}2)^{
		\frac{1}{2}}\bigg(
	\frac{2\mathcal{B}_{|D|,\lambda _{1}}}{\sqrt{\lambda _{1}}}+1\bigg)^{2}
	\frac{M}{\kappa }\mathcal{B}_{|D|,\lambda _{1}}
	\\
	\nonumber
	&&\ \ +(2\Gamma (2r+1)+\log ^{2r}2)2^{r}(\lambda _{1}^{r}+\lambda _{1}^{r-1}
	\lambda _{2}c_{V})\|g_{\rho}\|_{\rho}\bigg(
	\frac{2\mathcal{B}_{|D|,\lambda _{1}}}{\sqrt{\lambda _{1}}}+1\bigg)^{2r}.
\end{eqnarray}
\end{thm}
The next result establishes the explicit learning rates of multi-penalty
distribution regression after we further assume the decaying capacity condition
on effective dimension
$\mathcal{N}(\lambda _{1})\leq \mathcal{C}_{0}\lambda _{1}^{-\beta}$,
$\forall \lambda _{1}>0$, for some $\beta \in (0,1]$. The following result on minimax
optimal learning rates for multi-penalty distribution regression algorithm {\eqref{algorithm}} holds.
%
%t2 #&#
\begin{thm}%
\label{thm2}
Assume that the regularity condition {\eqref{regularity}} holds with
$1/2\leq r\leq 1$, $|y|\leq M$ holds almost surely, the capacity condition {\eqref{cap}} holds with index $\beta \in (0,1]$, the operator boundedness
condition {\eqref{L}} holds and the mapping
$K_{(\cdot )}:X_{\mu}\rightarrow \mathcal{L}(Y,\mathcal{H}_{K})$ is
$(\alpha ,L)$-H\"older continuous with $\alpha \in (0,1]$ and $L>0$. If
we choose the regularization parameters
$\lambda _{1}=|D|^{-\frac{1}{2r+\beta }}$,
$\lambda _{2}=\frac{1}{2c_{V}}|D|^{-\frac{2r}{2r+\beta }}$ and second-stage
sample size
$d_{1}=d_{2}=\cdots =d_{|D|}=|D|^{\frac{1+2r}{\alpha (2r+\beta )}}$, then
we have
%
%e2.12 #&#
\begin{equation}
	E\bigg[\big\|f_{\widehat{D},\lambda _{1},\lambda _{2}}-f_{\rho}\big\|_{
		\rho}\bigg]=\mathcal{O}\big(|D|^{-\frac{r}{2r+\beta }}\big).
	\label{eq2.9}
\end{equation}
\end{thm}
The above results handle standard setting $r\in [1/2,1]$, which corresponds
to $f_{\rho}\in \mathcal{H}_{K}$. The following two results are concerned
with error bounds and learning rates in the hard learning scenario
$f_{\rho}\notin \mathcal{H}_{K}$. The investigation of this scenario does
not appear in the literature on distribution regression. The following
error bound holds without capacity assumption
on effective dimension $\mathcal{N}(\lambda _{1})$.
%
%t3 #&#
\begin{thm}%
\label{thm3}
Assume that the regularity condition {\eqref{regularity}} holds with
$0< r< 1/2$, $|y|\leq M$ holds almost surely, the operator boundedness
condition {\eqref{L}} holds and the mapping
$K_{(\cdot )}:X_{\mu}\rightarrow \mathcal{L}(Y,\mathcal{H}_{K})$ is
$(\alpha ,L)$-H\"older continuous with $\alpha \in (0,1]$ and $L>0$. If the regularization parameters
$\lambda _{1}, \lambda _{2}\in (0,1)$ satisfy the relation
$2c_{V}\lambda _{2}=\lambda _{1}$ and
$\widetilde{d}$ satisfies
$\frac{1}{\widetilde{d}^{\frac{\alpha }{2}}}=\frac{1}{|D|}\sum _{i=1}^{|D|}
\frac{1}{d_{i}^{\alpha /2}}$, then we have
%
%e2.13 #&#
\begin{eqnarray}
	\nonumber
	&&E\bigg[\big\|f_{\widehat{D},\lambda _{1},\lambda _{2}}-f_{\rho}
	\big\|_{\rho}\bigg]
	\\
	\nonumber
	&&\leq 2\bigg(2\sqrt{2}(2+\sqrt{\pi})^{\frac{1}{2}}\kappa L
	\frac{2^{\frac{\alpha +2}{2}}B_{\widetilde{K}}^{\frac{\alpha }{2}}}{\lambda _{1}\widetilde{d}^{\frac{\alpha }{2}}}+2
	\bigg)(2+\sqrt{\pi})^{\frac{1}{2}}LM
	\frac{2^{\frac{\alpha }{2}}B_{\widetilde{K}}^{\frac{\alpha }{2}}}{\lambda _{1}^{\frac{1}{2}}\widetilde{d}^{\frac{\alpha }{2}}}(2
	\Gamma (3)+\log ^{2}2)\bigg(
	\frac{2\mathcal{B}_{|D|,\lambda _{1}}}{\sqrt{\lambda _{1}}}+1\bigg)^{2}
	\\
	\nonumber
	&&+\bigg(2\sqrt{2}(2+\sqrt{\pi})^{\frac{1}{2}}\kappa L
	\frac{2^{\frac{\alpha +2}{2}}B_{\widetilde{K}}^{\frac{\alpha }{2}}}{\lambda _{1}\widetilde{d}^{\frac{\alpha }{2}}}+2
	\bigg)\bigg(
	\frac{2\mathcal{B}_{|D|,\lambda _{1}}}{\sqrt{\lambda _{1}}}+1\bigg)^{2}(2
	\Gamma (5)+\log ^{4}2)^{\frac{1}{2}}\kappa L(2+\sqrt{\pi})^{
		\frac{1}{2}}2^{\frac{\alpha +2}{2}}B_{\widetilde{K}}^{
		\frac{\alpha }{2}}\times
	\\
	\nonumber
	&&\Bigg(2\sqrt{3}(2\Gamma (9)+\log ^{8}2)^{\frac{1}{4}}(2\Gamma (5)+
	\log ^{4}2)^{\frac{1}{4}}\Big(
	\frac{2\mathcal{B}_{|D|,\lambda _{1}}}{\sqrt{\lambda _{1}}}+1\Big)^{2}
	\Big[2M(\kappa +1)
	\frac{1}{\lambda _{1}^{\frac{1}{2}}\widetilde{d}^{\frac{\alpha }{2}}}
	\frac{1}{\sqrt{\lambda _{1}}}\mathcal{B}_{|D|,\lambda _{1}}'
	\\
	\nonumber
	&&+2(\kappa ^{2}+\kappa )\big \|g_{\rho}\big \|_{\rho}
	\frac{\lambda _{1}^{r-1}}{\lambda _{1}^{\frac{1}{2}}\widetilde{d}^{\frac{\alpha }{2}}}
	\mathcal{B}_{|D|,\lambda _{1}}'\Big]+2\sqrt{3}\big \|g_{\rho}\big \|_{
		\rho}
	\frac{\lambda _{1}^{r-\frac{1}{2}}}{\lambda _{1}^{\frac{1}{2}}\widetilde{d}^{\frac{\alpha }{2}}}
	\Bigg)
	\\
	\nonumber
	&&+4(2\Gamma (5)+\log ^{4}2)^{\frac{1}{2}}(2\Gamma (3)+\log ^{2}2)^{
		\frac{1}{2}}\bigg(
	\frac{2\mathcal{B}_{|D|,\lambda _{1}}}{\sqrt{\lambda _{1}}}+1\bigg)^{2}
	\Big[M(\kappa +1)\mathcal{B}_{|D|,\lambda _{1}}'+\|g_{\rho}\|_{\rho}(\kappa ^{2}+\kappa )\lambda _{1}^{r-
		\frac{1}{2}}\mathcal{B}_{|D|,\lambda _{1}}'\Big]
	\\
	&& +2\|g_{\rho}\|_{\rho}c_{V}
	\lambda _{2}\lambda _{1}^{r-1}+\lambda _{1}^{r}\|g_{\rho}\|_{\rho}.
	\label{eq2.10}
\end{eqnarray}
\end{thm}

When the capacity condition for $\mathcal{N}(\lambda _{1})$ holds, for
$r\in (0,1/2)$, by assigning new $\lambda _{1}$, $\lambda _{2}$,
$d_{1}, d_{2},..., d_{|D|}$, we derive following learning rates for the
case $f_{\rho}\notin \mathcal{H}_{K}$.
%
%t4 #&#
\begin{thm}%
\label{thm4}
Assume that the regularity condition {\eqref{regularity}} holds with
$0< r<1/2$ and $|y|\leq M$ almost surely. Assume the capacity condition {\eqref{cap}} holds with index $\beta \in (0,1]$, the operator boundedness
condition {\eqref{L}} holds and the mapping
$K_{(\cdot )}:X_{\mu}\rightarrow \mathcal{L}(Y,\mathcal{H}_{K})$ is
$(\alpha ,L)$-H\"older continuous with $\alpha \in (0,1]$ and $L>0$. If we
choose the regularization parameters
$\lambda _{1}=|D|^{-\frac{1}{1+\beta }}$,
$\lambda _{2}=\frac{1}{2c_{V}}|D|^{-\frac{1}{1+\beta }}$ and choose
$d_{1}=d_{2}=\cdots =d_{|D|}=|D|^{\frac{2}{\alpha (1+\beta )}}$, then we
have
%
%e2.14 #&#
\begin{equation}
	E\bigg[\big\|f_{\widehat{D},\lambda _{1},\lambda _{2}}-f_{\rho}\big\|_{
		\rho}\bigg]=\mathcal{O}\big(|D|^{-\frac{r}{1+\beta }}\big).
	\label{eq2.11}
\end{equation}
\end{thm}

For the multi-penalty distribution regression-based distributed learning
algorithm {\eqref{disal}}, under a mild restriction on the machine number
$m$, we derive the following result on optimal learning rate of the estimator
$\overline{f_{\widehat{D},\lambda _{1},\lambda _{2}}}$ generated from our
new distributed learning scheme of distribution regression in {\eqref{disal}}.
%
%t5 #&#
\begin{thm}%
%%LEAP%%%\label{thm5}
\label{disrate}
Assume that the regularity condition {\eqref{regularity}} holds with
$1/2\leq r\leq 1$, $|y|\leq M$ almost surely, the capacity condition {\eqref{cap}} holds and the operator boundedness
condition {\eqref{L}} holds with a constant $c_{V}>0$ independent of data sets. The mapping
$K_{(\cdot )}:X_{\mu}\rightarrow \mathcal{L}(Y,\mathcal{H}_{K})$ is assumed
to be $(\alpha ,L)$-H\"older continuous with $\alpha \in (0,1]$ and
$L>0$. 
If the sample size of the local machines satisfies $|D_{j}|=|D|/m$ for
$j=1,2,...,m$, the penalties satisfy
$\lambda _{1}=|D|^{-\frac{1}{2r+\beta }}$,
$\lambda _{2}=\frac{1}{2c_{V}}|D|^{-\frac{2r}{2r+\beta }}$, the second-stage
sample sizes are taken as $d=|D|^{\frac{1+2r}{\alpha (2r+\beta )}}$ and total number
$m$ of local machines satisfies
%
%e2.15 #&#
\begin{equation}
	\nonumber
	m\leq |D|^{\frac{2r-1}{2r+\beta }},
\end{equation}
then we have
%
%e2.16 #&#
\begin{equation}
	E\Big[\big\|\overline{f_{\widehat{D},\lambda _{1},\lambda _{2}}}-f_{
		\rho}\big\|_{\rho}\Big]=\mathcal{O}(|D|^{-\frac{r}{2r+\beta }}).
	\label{eq2.12}
\end{equation}
\end{thm}

The theoretical result in {Theorem~\ref{disrate}} indicates that our new
proposed distributed estimator
$\overline{f_{\widehat{D},\lambda _{1},\lambda _{2}}}$ can achieve optimal
learning rates of $\mathcal{O}(|D|^{-\frac{r}{2r+\beta }})$. Hence our
theoretical result provides a satisfactory guarantee for the nice learning
ability of $\overline{f_{\widehat{D},\lambda _{1},\lambda _{2}}}$. Moreover,
{Theorem~\ref{disrate}} provides the possibility that, when handling massive
distribution samples in practical computation, we can divide the large-scale
distribution data set $D$ into several subsets to reduce the huge computational
burden. Hence these facts indicate the great potential of our distributed
estimator $\overline{f_{\widehat{D},\lambda _{1},\lambda _{2}}}$ in practical
computation field on handling massive distribution data in the big data
era.\looseness=-1

%s3 #&#
\section{Related work and discussions}
\label{sec3}

We make some comparisons between this work and those in the existing literature.
In the past decade, studies on classical regularized least squares regression
are on the way toward maturity. Various approaches are utilized to analyze
the minimax learning rates of regularized least squares regression. Recently,
studies have just turned to the learning rates of distribution regression.
Theoretical investigation on distribution regression is very limited. To
the best of our knowledge, the only existing works on learning theory analysis
of learning rates of distribution regression scheme {\eqref{alone}} are contained
in \cite{fang} and \cite{only}. In {Theorem~\ref{thm1}} and {Theorem~\ref{thm2}}, we show that an optimal learning rate of
%
%e3.1 #&#
\begin{equation}
\nonumber
E\bigg[\big\|f_{\widehat{D},\lambda _{1},\lambda _{2}}-f_{\rho}\big\|_{
	\rho}\bigg]=\mathcal{O}\big(|D|^{-\frac{r}{2r+\beta }}\big)
\end{equation}
for multi-penalty distribution regression is obtained for
$\frac{1}{2}\leq r\leq 1$. The results improve the work of
\cite{only} and \cite{fang} to a more general setting. Specifically, the
optimal rates cover the case of \cite{only} with $r=1/2$, in contrast to
the suboptimal rate of
$E[\|f_{\widehat{D},\lambda }-f_{\rho}\|^{2}]=\mathcal{O}(|D|^{-
\frac{2}{5}})$. Also, the rates coincide with the optimal rate in
\cite{fang}, and this work improves the result in \cite{fang} to a more
general setting by considering the additional penalty with parameter
$\lambda _{2}$ and a bounded operator $V_{D}$ in its multi-penalty regularization
framework. It can also be observed that the regression scheme of
\cite{fang} can be treated as a degenerate case of the scheme {\eqref{algorithm}} with $\lambda _{2}=0$. As we do not impose coercive conditions
on operator $V_{D}$ except for its boundedness, Algorithm {\eqref{algorithm}} presents potential flexibility in handling distribution
regression. In contrast to previous works in the literature of distribution
regression, due to the introduction of an additional new penalty induced
by the operator $V_{D}$, difficulties arise when we need to handle both
first-stage operator estimates and second-stage estimates since the operator
$V_{D}$ has already participated in the operator analysis in each part
where the operator decomposition process needs to be carried out. To face the analysis challenges
arising from the additional penalty induced by the operator $V_{D}$, we
have rigorously introduced the quantities
$\mathcal{A}_{D,\lambda _{1},\lambda _{2},V_{D}}$ and
$\Omega _{D,\lambda _{1},\lambda _{2},V_{D}}$ below in order to address the influence
of $V_{D}$ and perform operator decomposition successfully in the current
setting.

On the other hand, all existing results in distribution regression literature
consider the case when the regression function
$f_{\rho}\in \mathcal{H}_{K}$. In kernel-based learning theory, when the
target function does not live in the underlying kernel space, the corresponding
problem is often referred to as a hard learning problem (see, e.g.
\cite{hard}) which is very hot in recent studies in the literature. In
kernel ridge regression, the hard learning problems have been intensively
studied very recently (e.g. \cite{lrrc2020}, \cite{hard},
\cite{shi2019}, \cite{sf2020}). One of the reasons for investigating hard
learning scenario is that, the assumption that $f_{\rho}$ lies precisely
in $\mathcal{H}_{K}$ is quite restrictive in many practical applications.
However, for $f_{\rho}$ satisfying the regularity condition
$r\in (0,1/2)$, namely $f_{\rho}\notin \mathcal{H}_{K}$, which belongs
to the hard learning scenario, the learning theory has not been established
for distribution regression and the convergence analysis on min-max learning
rates has not been carried out. Hence it is interesting to ask whether
the multi-penalty distribution regression estimator
$f_{\widehat{D},\lambda _{1},\lambda _{2}}$ still maintains a nice learning
performance in the hard learning setting. In {Theorem~\ref{thm3}} and {Theorem~\ref{thm4}} of this work, we have answered the question by carrying out
learning rate analysis and improved the analyzable regularity range from
[1/2,1] for the standard setting $f_{\rho}\in \mathcal{H}_{K}$ (\cite{fang},
\cite{only}) to $r\in (0,1/2)$ for hard learning setting
$f_{\rho}\notin \mathcal{H}_{K}$. The theoretical results of {Theorem~\ref{thm3}} and {Theorem~\ref{thm4}} fill the gap of the study of hard learning
problem in distribution regression.

Let us describe some core differences in theoretical approaches between
the current work and previous work \cite{fang}. As a starting point of
convergence analysis, the previous work \cite{fang}, in fact, employed
an error decomposition in terms of
\begin{equation*}
f_{\widehat{D}, \lambda }-f_{\rho}=(f_{\widehat{D}, \lambda }-f_{D,
	\lambda } )+\left (f_{D, \lambda }-f_{\rho}\right )
\end{equation*}
in which $f_{D,\lambda }$ is the first-stage counterpart of
$f_{\widehat{D},\lambda }$. For the hard learning scenario
$f_{\rho}\notin \mathcal{H}_{K}$, one of the difficulties is that such
a type of error decomposition approach fails in the tough settings when
$f_{\rho}$ does not lie in the RKHS $\mathcal{H}_{K}$ since we can not
directly handle $f_{D, \lambda }-f_{\rho}$ via operator representation
and the corresponding underlying bounds for estimating
$f_{\widehat{D}, \lambda }-f_{D, \lambda }$ also lose. In this work, to
handle such a tough setting, we in fact, essentially carry out a further
new two-stage error decomposition
\begin{equation*}
(f_{\widehat{D},\lambda _{1},\lambda _{2}}-f_{D,\lambda _{1},\lambda _{2}})+(f_{D,
	\lambda _{1},\lambda _{2}}-f_{\lambda _{1}})+(f_{\lambda _{1}}-f_{
	\rho})
\end{equation*}
for the estimator $f_{\widehat{D},\lambda _{1},\lambda _{2}}$ which is
generated from our more general multi-penalty regularized regression scheme {\eqref{algorithm}}, one of the technical novelties of this type two-stage
error decomposition is to use the crucial data-free minimizer
$f_{\lambda _{1}}$ (associated with the first regularization parameter
$\lambda _{1}$) defined by
$f_{\lambda _{1}}=\arg \min _{f\in \mathcal{H}_{K}}\{\|f-f_{\rho}\|_{L_{
	\rho _{X_{\mu}}}^{2}}^{2}+\lambda _{1}\|f\|_{K}^{2}\}$ as a key bridge
to realize core estimates in both the first-stage representation and the
second-stage representation (also see discussions in {Remark~\ref{decomposition_discuss}} and {Remark~\ref{secondstage_discuss}}). Fortunately,
we also observe that the introduction of $f_{\lambda _{1}}$ also performs
very well when handling some main terms hidden in the estimates of the
second-stage
\begin{equation*}
f_{\widehat{D},\lambda _{1},\lambda _{2}}-f_{D,\lambda _{1},\lambda _{2}}
\end{equation*}
(for example, see {Proposition~\ref{rfd}} and the proof of {Theorem~\ref{thm4}}) where $f_{D,\lambda _{1},\lambda _{2}}$ is defined in {\eqref{fuzhuf}}. These nice performances just right result in a satisfactory
mini-max rate of $\mathcal{O}\big(|D|^{-\frac{r}{1+\beta }}\big)$ ($r
\in (0,1/2)$) for
$f_{\widehat{D},\lambda _{1},\lambda _{2}}-f_{D,\lambda _{1},\lambda _{2}}$
after using our selection rule for regularization parameters
$\lambda _{1}$, $\lambda _{2}$ (see analysis in the proof of {Theorem~\ref{thm4}}). The detailed analysis has been carried out in Section~\ref{nonstandardcase} (Also see {Remark~\ref{decomposition_discuss}} and
{Remark~\ref{secondstage_discuss}}). In addition, due to the multi-penalty
nature of our algorithm, we need further take the influence of the operator
$V_{D}$ that induces the additional penalty into consideration in each
operator decomposition process. In contrast to previous works, the participation
of $V_{D}$ raises the difficulties of the operator decomposition processes
and makes the operator norm estimates in both the first and second stage
more intricate. To overcome the difficulties arising from the operator
$V_{D}$, as mentioned above, we have rigorously introduced the new quantities
$\mathcal{A}_{D,\lambda _{1},\lambda _{2},V_{D}}$ and
$\Omega _{D,\lambda _{1},\lambda _{2},V_{D}}$ and successfully realize
the operator decomposition process in our setting.

Another contribution of this paper is that, based on a divide-and-conquer
approach, we present a novel two-stage multi-penalty distribution-regression-based
distributed learning algorithm. It can be witnessed that, among the existing
works in the literature on distribution regression, there is no powerful
method for handling large-scale or massive distribution data problems.
The theoretical result is also lacking and not established. The existing
models of distribution regression mainly focus on the single-machine model.
When the data scale of the distribution samples is extremely large, their
drawback becomes obvious since the single machine can not perform efficiently
when handling massive distribution data. However, an appropriate distributed
learning scheme for two-stage distribution regression has not been proposed.
One of the difficulties in deriving results of learning rates is that a
well-defined two-stage distributed distribution regression estimator is
still lacking. Another challenge is that a suitable two-stage error decomposition
for the estimator has not yet been established either before. In learning
theory, establishing an appropriate error decomposition related to generalization
error is often a crucial step to obtaining the desired learning rates of
the estimator. The current work fills the gap by proposing the distributed
learning estimator
$\overline{f_{\widehat{D},\lambda _{1},\lambda _{2}}}$ (defined in {\eqref{disal}}) for handling massive distribution data and providing a novel
two-stage decomposition for
\begin{equation*}
\overline{f_{\widehat{D},\lambda _{1},\lambda _{2}}}- f_{\widehat{D},
	\lambda _{1},\lambda _{2}}
\end{equation*}
via operator representation (See {Proposition~\ref{two-stage_distributed_decomposition}}). The optimal rates are obtained
for this method in {Theorem~\ref{disrate}}. The mini-max rates coincide with
the rates we establish in non-distributed settings. Therefore the new regression
method presents the advantages of reducing the computational burden on
computing time, storage bottleneck, and memory requirements over a single
machine when processing the distribution regression algorithm. Distributed
learning algorithms have been extensively studied recently (e.g.,
\cite{s2,s3,sbsecond,sb3}). From these existing methods, it can be observed
that the distributed learning method of handling distribution or functional
data is still unexplored. By presenting the new distributed learning algorithm
with multi-penalty distribution regression and proving its optimal rates,
this paper provides an effective distributed learning method for handling
distribution data. Based on the above discussions and the satisfactory
theoretical result in {Theorem~\ref{disrate}}, the new proposed multi-penalty
distribution-regression-based algorithm possesses great potential for handling
large-scale distribution data in modern computational science and statistical
science.

%s4 #&#
\section{Analysis and estimates on learning rates}
\label{sec4}

In this section, we provide the analysis and proofs for the main results
of {Theorems~\ref{thm1}, \ref{thm2}, \ref{thm3}, \ref{thm4}}. In order to estimate the expected
learning rates of the algorithm, the subsequent estimates for our multi-penalty
distribution regression estimator
$f_{\widehat{D},\lambda _{1},\lambda _{2}}$ are based on the following
basic two-stage error decomposition
%
%e4.1 #&#
\begin{equation}
\|f_{\widehat{D},\lambda _{1},\lambda _{2}}-f_{\rho}\|_{\rho}\leq \|f_{
	\widehat{D},\lambda _{1},\lambda _{2}}-f_{D,\lambda _{1},\lambda _{2}}+f_{D,
	\lambda _{1},\lambda _{2}}-f_{\rho }\|_{\rho }\leq \|f_{\widehat{D},
	\lambda _{1},\lambda _{2}}-f_{D,\lambda _{1},\lambda _{2}}\|_{\rho }+
\|f_{D,\lambda _{1},\lambda _{2}}-f_{\rho }\|_{\rho },
\label{eq4.1}
\end{equation}
and a further two-stage error decomposition
%
%e4.2 #&#
\begin{equation}
\|f_{\widehat{D},\lambda _{1},\lambda _{2}}-f_{\rho}\|_{\rho}\leq \|f_{
	\widehat{D},\lambda _{1},\lambda _{2}}-f_{D,\lambda _{1},\lambda _{2}}
\|_{\rho}+\|f_{D,\lambda _{1},\lambda _{2}}-f_{\lambda _{1}}\|_{\rho}+
\|f_{\lambda _{1}}-f_{\rho}\|_{\rho}
\label{eq4.2}
\end{equation}
where
\begin{equation*}
f_{\lambda _{1}}=\arg \min _{f\in \mathcal{H}_{K}}\bigg\{\|f-f_{\rho}\|_{L_{
		\rho _{X_{\mu}}}^{2}}^{2}+\lambda _{1}\|f\|_{K}^{2}\bigg\}.
\end{equation*}
In the above decomposition, $f_{D,\lambda _{1},\lambda _{2}}$ is the first-stage
counterpart of $f_{\widehat{D},\lambda _{1},\lambda _{2}}$ which is the
minimizer of the following classical multi-penalty regression scheme associated
with the first-stage sample $D=\{(\mu _{x_{i}},y_{i})\}_{i=1}^{|D|}$,
%
%e4.3 #&#
\begin{equation}
f_{D,\lambda _{1},\lambda _{2}}=\arg \min _{f\in \mathcal{H}_{K}}
\bigg\{\frac{1}{|D|}\sum _{i=1}^{|D|}\big(f(\mu _{x_{i}})-y_{i}\big)^{2}+
\lambda _{1}\big\|f\big\|_{K}^{2}+\lambda _{2}\big\|V_{D} f\big\|_{K}^{2}
\bigg\}.
\label{fuzhuf}
%%LEAP%%%\label{eq4.3}
\end{equation}
This quantity serves as an important bridge in subsequent analysis.

To make a preliminary preparation, we define the sampling operator
$S_{D}:\mathcal{H}_{K}\rightarrow \mathbb R^{|D|}$ associated with the
first-stage sample as
%
%e4.4 #&#
\begin{equation}
\nonumber
S_{D}f=(f(\mu _{x_{1}}),f(\mu _{x_{2}}),...,f(\mu _{x_{|D|}}))^{T},
\ f\in \mathcal{H}_{K}.
\end{equation}
The adjoint operator $S_{D}^T:\mathbb R^{|D|}\rightarrow \mathcal{H}_{K}$ is given by
%
%e4.5 #&#
\begin{equation}
\nonumber
S_{D}^{T}\mathbf{c}_{D}=\frac{1}{|D|}\sum _{i=1}^{|D|}c_{i}K_{\mu _{x_{i}}},
\ \mathbf{c}_{D}=(c_{1},c_{2},...,c_{|D|})^T\in \mathbb R^{|D|}.
\end{equation}
Then we can define the first-stage empirical operator $L_{K,D}$ of
$L_{K}$ as
%
%e4.6 #&#
\begin{equation}
\nonumber
L_{K,D}(f)=S_{D}^{T}S_{D}(f)=\frac{1}{|D|}\sum _{i=1}^{|D|}f(\mu _{x_{i}})K_{
	\mu _{x_{i}}}=\frac{1}{|D|}\sum _{i=1}^{|D|}\langle K_{\mu _{x_{i}}},f
\rangle _{K}K_{\mu _{x_{i}}}, \ f\in \mathcal{H}_{K}.
\end{equation}
We also define the sampling operator $\hat{S}_{D}:\mathcal{H}_{K}\rightarrow \mathbb R^{|D|}$ associated with the
second-stage sample as
%
%e4.7 #&#
\begin{equation}
\nonumber
\hat{S}_{D}f=(f(\mu _{\hat{x}_{1}}),f(\mu _{\hat{x}_{2}}),...,f(\mu _{
	\hat{x}_{|D|}}))^{T}, \ f\in \mathcal{H}_{K}.
\end{equation}
Its adjoint operator $\hat{S}_{D}^{T}:\mathbb R^{|D|}\rightarrow \mathcal{H}_{K}$ is given by
%
%e4.8 #&#
\begin{equation}
\nonumber
\hat{S}_{D}^{T}\mathbf{c}_{D}=\frac{1}{|D|}\sum _{i=1}^{|D|}c_{i}K_{
	\mu _{\hat{x}_{i}}}, \ \mathbf{c}_{D}=(c_{1},c_{2},...,c_{|D|})^T\in
\mathbb R^{|D|}.
\end{equation}
Then the empirical operator of $L_{K,D}$ can be defined accordingly by
using the second-stage sample $\widehat{D}$ {as follows}
%
%e4.9 #&#
\begin{equation}
L_{K,\widehat{D}}(f)=\hat{S}_{D}^{T}\hat{S}_{D}(f)=\frac{1}{|D|}\sum _{i=1}^{|D|}f(
\mu _{\hat{x}_{i}})K_{\mu _{\hat{x}_{i}}}=\frac{1}{|D|}\sum _{i=1}^{|D|}
\langle K_{\mu _{\hat{x}_{i}}},f\rangle _{K}K_{\mu _{\hat{x}_{i}}},
\ f\in \mathcal{H}_{K}.
\label{lkd}
%%LEAP%%%\label{eq4.4}
\end{equation}
With these notations, after taking Fr\'echet derivative on the right hand sides of
Algorithm {\eqref{algorithm}} and its first-stage associate {\eqref{fuzhuf}}, it is easy to see the following representation holds,
\begin{equation*}
f_{\widehat{D},\lambda _{1},\lambda _{2}}=(L_{K,\widehat{D}}+\lambda _{1}I+
\lambda _{2}V_{D}^{T}V_{D})^{-1}\hat{S}_{D}^{T}y_{D},
\end{equation*}
\begin{equation*}
f_{D,\lambda _{1},\lambda _{2}}=(L_{K,D}+\lambda _{1}I+\lambda _{2}V_{D}^{T}V_{D})^{-1}S_{D}^{T}y_{D},
\end{equation*}
with $y_D=(y_1,y_2,...,y_{|D|})^T\in\mathbb R^{|D|}$. In the following, we use $E_{\mathbf{z}^{|D|}}[\cdot ]$ to denote the expectation
w.r.t. $\mathbf{z}^{|D|}=\{z_{i}=(\mu _{x_{i}},y_{i})\}_{i=1}^{|D|}$. Use
$E_{\mathbf{x}^{\mathbf{d},|D|}|\mathbf{z}^{|D|}}$ to denote the conditional
expectation w.r.t. sample
$\big\{\{x_{i,s}\}_{s=1}^{d_{i}}\big\}_{i=1}^{|D|}$ conditioned on
$\{z_{1},z_{2},...,z_{|D|}\}$. Namely
%
%e4.10 #&#
\begin{equation}
\nonumber
E_{\mathbf{z}^{|D|}}[\cdot ]=E_{\{(\mu _{x_{i}},y_{i})\}_{i=1}^{|D|}}[
\cdot ], \ E_{\mathbf{x}^{\mathbf{d},|D|}|\mathbf{z}^{|D|}}[\cdot ]=E_{
	\{\{x_{i,s}\}_{s=1}^{d_{i}}\}_{i=1}^{|D|}\big|\{z_{i}\}_{i=1}^{|D|}}[
\cdot ].
\end{equation}
Following (13) in \cite{twostage} and Theorem 15 in \cite{unify}, we know
that for any $i=1,2,...,|D|$,
%
%e4.11 #&#
\begin{equation}
\nonumber
\big \|\mu _{\hat{x}_{i}}-\mu _{x_{i}}\big \|_{H}\leq
\frac{(1+\sqrt{\delta })\sqrt{2B_{\widetilde{K}}}}{\sqrt{d_{i}}}
\end{equation}
with probability $1-e^{-\delta}$. Set
$\xi =\big \|\mu _{\hat{x}_{i}}-\mu _{x_{i}}\big \|_{H}^{2\alpha }$ with
$\alpha \in (0,1]$. The boundedness condition on kernel
$\widetilde{K}$ implies
$\xi \in [0,2^{2\alpha }B_{\widetilde{K}}^{\alpha }]$ almost surely. Then
it follows that
$E[\xi ]=\int _{0}^{\infty}\text{Prob}_{\{x_{i,s}\}_{s=1}^{d_{i}}|x_{i}}(
\xi >t)dt=\int _{0}^{2^{2\alpha }B_{\widetilde{K}}^{\alpha }}
\text{Prob}_{\{x_{i,s}\}_{s=1}^{d_{i}}|x_{i}}(\big \|\mu _{\hat{x}_{i}}-
\mu _{x_{i}}\big \|_{H}>t^{\frac{1}{2\alpha }})dt$. Then a variable change
of $t=(2B_{\widetilde{K}}/ d_{i})(1+\sqrt{\delta})^{2\alpha }$ implies
that
%
%e4.12 #&#
\begin{equation}
\nonumber
E_{\{x_{i,s}\}_{s=1}^{d_{i}}|x_{i}}\Big[\big \|\mu _{\hat{x}_{i}}-
\mu _{x_{i}}\big \|_{H}^{2\alpha }\Big]\leq (2+\sqrt{\pi})
\frac{2^{\alpha }B_{\widetilde{K}}^{\alpha }}{d_{i}^{\alpha }}.
\end{equation}
Then we know that
%
%e4.13 #&#
\begin{eqnarray}
\nonumber
&&\Big\{E_{\mathbf{x}^{\mathbf{d},|D|}|\mathbf{z}^{|D|}}\Big[\big \|
\hat{S}_{D}^{T}y_{D}-S_{D}y_{D}\big \|_{K}^{2}\Big]\Big\}^{
	\frac{1}{2}}=\Big\{E_{\mathbf{x}^{\mathbf{d},|D|}|\mathbf{z}^{|D|}}
\Big[\Big\| \frac{1}{|D|}\sum _{i=1}^{|D|}y_{i}(K_{\mu _{\hat{x}_{i}}}-K_{
	\mu _{x_{i}}})\Big\|_{K}^{2}\Big]\Big\}^{\frac{1}{2}}
\\
\nonumber
&&\leq \Big\{E_{\mathbf{x}^{\mathbf{d},|D|}|\mathbf{z}^{|D|}}\Big[
\frac{1}{|D|} \sum _{i=1}^{|D|}\big|y_{i}\big|\big\|K_{\mu _{\hat{x}_{i}}}-K_{
	\mu _{x_{i}}}\big\|_{\mathcal{L}(Y,\mathcal{H}_{K})}^{2}\Big]\Big\}^{
	\frac{1}{2}}\leq M\Big\{E_{\mathbf{x}^{\mathbf{d},|D|}|\mathbf{z}^{|D|}}
\Big[\frac{1}{|D|} \sum _{i=1}^{|D|}L^{2}\big \|\mu _{\hat{x}_{i}}-
\mu _{x_{i}}\big \|_{H}^{2\alpha }\Big]\Big\}^{\frac{1}{2}}
\\
&&\leq ML\frac{1}{|D|}\sum _{i=1}^{|D|} \Big\{E_{\mathbf{x}^{
		\mathbf{d},|D|}|\mathbf{z}^{|D|}} \Big[\big \|\mu _{\hat{x}_{i}}-\mu _{x_{i}}
\big \|_{H}^{2\alpha }\Big]\Big\}^{\frac{1}{2}}\leq (2+\sqrt{\pi})^{
	\frac{1}{2}}ML\frac{1}{|D|}\sum _{i=1}^{|D|}
\frac{2^{\frac{\alpha }{2}}B_{\widetilde{K}}^{\frac{\alpha }{2}}}{d_{i}^{\frac{\alpha }{2}}},
\label{s1sd}
%%LEAP%%%\label{eq4.5}
\end{eqnarray}
in which the first inequality follows from the convex inequality of norm
$\|\cdot \|_{K}$ and the fact that
$\|yK_{\mu _{x}}\|_{K}\leq \|y\|_{Y}\|K_{\mu _{x}}\|_{\mathcal{L}(Y,
\mathcal{H}_{K})}$ for any $y\in Y$, $\mu _{x}\in X_{\mu}$. Specifically,
$\|\cdot \|_{Y}=|\cdot |$ when $Y=\mathbb R$ in our setting. The second
inequality follows from the H\"older assumption of mapping
$K_{(\cdot )}$. Recalling the structure of $L_{K,\widehat{D}}$ defined
in {\eqref{lkd}} and boundedness condition of kernel $K$ in {\eqref{bddker}}, using H\"older condition again to $K_{(\cdot )}$, we know
%
%e4.14 #&#
\begin{equation}
\Big\{E_{\mathbf{x}^{\mathbf{d},|D|}|\mathbf{z}^{|D|}}\Big[\big \|L_{K,
	\widehat{D}}-L_{K,D}\big \|^{2}\Big]\Big\}^{\frac{1}{2}}\leq \kappa L(2+
\sqrt{\pi})^{\frac{1}{2}}\frac{1}{|D|}\sum _{i=1}^{|D|}
\frac{2^{\frac{\alpha +2}{2}}B_{\widetilde{K}}^{\frac{\alpha }{2}}}{d_{i}^{\frac{\alpha }{2}}}.
\label{s1ld}
%%LEAP%%%\label{eq4.6}
\end{equation}

In subsequent analysis, we need the following lemma (\cite{multiguo}) of one-stage learning theory to handle the terms involving the information of additional
penalty.

%l1 #&#
\begin{lem}%
%%LEAP%%%\label{lem1}
\label{guolem}
If the regularization parameters
$\lambda _{1},\lambda _{2}\in (0,1)$ satisfy
$2c_{V}\lambda _{2}=\lambda _{1}^{\max \{2r,1\}}$ for $r\in (0,1]$, then
the following norm bound holds almost surely:
%
%e4.15 #&#
\begin{equation}
	\Big\|(\lambda _{1}I+L_{K})(\lambda _{1}I+L_{K}+\lambda _{2}V_{D}^{T}V_{D})^{-1}
	\Big\|\leq 2.
	\label{eq4.7}
\end{equation}
\end{lem}
In following proofs, we denote
%
%e4.16 #&#
\begin{equation}
\nonumber
\mathcal{A}_{D,\lambda _{1},\lambda _{2},V_{D}}=\big\|(\lambda _{1}I+L_{K}+
\lambda _{2}V_{D}^{T}V_{D})(\lambda _{1}I+L_{K,D}+\lambda _{2}V_{D}^{T}V_{D})^{-1}
\big\|.
\end{equation}
%

%s4.1 #&#
\subsection{Analysis of the error bounds and rates when $r\in [1/2,1]$}
\label{sec4.1}

Now we come to the proof of {Theorem~\ref{thm1}} and {Theorem~\ref{thm2}}, we begin with the
following error estimate for
$\big\|f_{D,\lambda _{1},\lambda _{2}}-f_{\rho}\big\|_{\rho}$.
%
%p1 #&#
\begin{pro}%
%%LEAP%%%\label{prop1}
\label{ro1}
Assume that $|y|\leq M$ and {\eqref{L}} hold almost surely. Let the regularity
condition {\eqref{regularity}} hold with some index $r$ satisfying
$1/2\leq r\leq 1$. If $\lambda _{1},\lambda _{2}\in (0,1)$ satisfy
$2c_{V}\lambda _{2}=\lambda _{1}^{2r}$. Then there holds almost surely
\begin{eqnarray}
	\nonumber
	&&\big\|f_{D,\lambda _{1},\lambda _{2}}-f_{\rho}\big\|_{\rho}\leq 2
	\mathcal{A}_{D,\lambda _{1},\lambda _{2},V_{D}}\big\|(\lambda _{1}I+L_{K})^{-
		\frac{1}{2}}(S_{D}^{T}y_{D}-L_{K,D}f_{\rho})\big\|_{K}
	\\
	\nonumber
	&&\quad \quad \quad \quad \quad \quad \quad \quad \ +(\lambda _{1}^{r}+
	\lambda _{1}^{r-1}\lambda _{2}c_{V})2^{r}\mathcal{A}_{D,\lambda _{1},
		\lambda _{2},V_{D}}^{r}\big\|g_{\rho}\big\|_{\rho}.
\end{eqnarray}
\end{pro}

\begin{proof}
Using the fact that for any $h\in L_{\rho _{X_{\mu}}}^{2}$,
$\| h\|_{\rho}=\|L_{K}^{1/2}h\|_{K}$, we split
$\big \|f_{D,\lambda _{1},\lambda _{2}}-f_{\rho}\big \|_{\rho}$ into three
parts as
\begin{eqnarray}
	\nonumber
	&&\big \|f_{D,\lambda _{1},\lambda _{2}}-f_{\rho}\big \|_{\rho}
	\\
	\nonumber
	&&=\bigg \|L_{K}^{1/2}\bigg\{(\lambda _{1}I+L_{K,D}+\lambda _{2}V_{D}^{T}V_{D})^{-1}S_{D}^{T}y_{D}-(
	\lambda _{1}I+L_{K,D}+\lambda _{2}V_{D}^{T}V_{D})^{-1}(\lambda _{1}I+L_{K,D}+
	\lambda _{2}V_{D}^{T}V_{D})f_{\rho}\bigg\}\bigg \|_{K}
	\\
	\nonumber
	&&=\bigg \|L_{K}^{1/2}\bigg\{(\lambda _{1}I+L_{K,D}+\lambda _{2}V_{D}^{T}V_{D})^{-1}(S_{D}^{T}y_{D}-L_{K,D}f_{
		\rho}-\lambda _{1}f_{\rho}-\lambda _{2}V_{D}^{T}V_{D}f_{\rho})\bigg\}
	\bigg \|_{K}
	\\
	\nonumber
	&&\leq \mathcal{T}_{1}+\mathcal{T}_{2}+\mathcal{T}_{3},
\end{eqnarray}
where
\begin{eqnarray*}
	&&\mathcal{T}_{1}=\big \|L_{K}^{1/2}(\lambda _{1}I+L_{K,D}+\lambda _{2}V_{D}^{T}V_{D})^{-1}(S_{D}^{T}y_{D}-L_{K,D}f_{
		\rho})\big \|_{K},
	\\
	&&\mathcal{T}_{2}=\lambda _{1}\big \|L_{K}^{1/2}(\lambda _{1}I+L_{K,D}+
	\lambda _{2}V_{D}^{T}V_{D})^{-1}L_{K}^{r}g_{\rho}\big \|_{K},
	\\
	&&\mathcal{T}_{3}= \big \|L_{K}^{1/2}(\lambda _{1}I+L_{K,D}+\lambda _{2}V_{D}^{T}V_{D})^{-1}
	\lambda _{2}V_{D}^{T}V_{D}f_{\rho}\big \|_{K}.
\end{eqnarray*}
We estimate $\mathcal{T}_{1}$, $\mathcal{T}_{2}$, and
$\mathcal{T}_{3}$ respectively, as follows, we have almost surely,
\begin{eqnarray}
	\nonumber
	&&\mathcal{T}_{1}=\big \|L_{K}^{1/2}(\lambda _{1}I+L_{K,D}+\lambda _{2}V_{D}^{T}V_{D})^{-1}(S_{D}^{T}y_{D}-L_{K,D}f_{
		\rho})\big \|_{K}
	\\
	\nonumber
	&&\leq \big \|L_{K}^{\frac{1}{2}}(L_{K}+\lambda _{1}I)^{-\frac{1}{2}}
	\big \|\big \|(L_{K}+\lambda _{1}I)^{\frac{1}{2}}(\lambda _{1}I+L_{K}+
	\lambda _{2}V_{D}^{T}V_{D})^{-\frac{1}{2}}\big \|\big \|(\lambda _{1}I+L_{K}+
	\lambda _{2}V_{D}^{T}V_{D})^{\frac{1}{2}}
	\\
	\nonumber
	&&(\lambda _{1}I+L_{K,D}+\lambda _{2}V_{D}^{T}V_{D})^{-\frac{1}{2}}
	\big \|\big \|(\lambda _{1}I+L_{K,D}+\lambda _{2}V_{D}^{T}V_{D})^{-
		\frac{1}{2}}(\lambda _{1}I+L_{K}+\lambda _{2}V_{D}^{T}V_{D})^{
		\frac{1}{2}}\big \|
	\\
	\nonumber
	&&\big \|(\lambda _{1}I+L_{K}+\lambda _{2}V_{D}^{T}V_{D})^{-
		\frac{1}{2}}(\lambda _{1}I+L_{K})^{\frac{1}{2}}\big \|\big \|(
	\lambda _{1}I+L_{K})^{-\frac{1}{2}}(S_{D}^{T}y_{D}-L_{K,D}f_{\rho})
	\big \|_{K}
	\\
	\nonumber
	&&\leq 2\mathcal{A}_{D,\lambda _{1},\lambda _{2},V_{D}}\big \|(
	\lambda _{1}I+L_{K})^{-\frac{1}{2}}(S_{D}^{T}y_{D}-L_{K,D}f_{\rho})
	\big \|_{K},
\end{eqnarray}
in which we have used
$\|L_{K}^{1/2}(L_{K}+\lambda _{1}I)^{-1/2}\|\leq 1$, {Lemma~\ref{guolem}}, and the fact that, for any positive self-adjoint operators
$T_{1}$ and $T_{2}$ on a Hilbert space and $s\in [0,1]$,
$\|T_{1}^{s}T_{2}^{s}\|\leq \|T_{1}T_{2}\|^{s}$. For
$\mathcal{T}_{2}$, we have
\begin{eqnarray}
	\nonumber
	&&\mathcal{T}_{2}=\lambda _{1}\big \|L_{K}^{1/2}(\lambda _{1}I+L_{K,D}+
	\lambda _{2}V_{D}^{T}V_{D})^{-1}L_{K}^{r}g_{\rho}\big \|_{K}
	\\
	\nonumber
	&&\leq \lambda _{1}\big \|L_{K}^{\frac{1}{2}}(L_{K}+\lambda _{1}I)^{-
		\frac{1}{2}}\big \|\big \|(L_{K}+\lambda _{1}I)^{\frac{1}{2}}(
	\lambda _{1}I+L_{K}+\lambda _{2}V_{D}^{T}V_{D})^{-\frac{1}{2}}\big \|
	\big \|(\lambda _{1}I+L_{K}+\lambda _{2}V_{D}^{T}V_{D})^{\frac{1}{2}}
	\\
	\nonumber
	&&\ \ (\lambda _{1}I+L_{K,D}+\lambda _{2}V_{D}^{T}V_{D})^{-
		\frac{1}{2}}\big \|\big \|(\lambda _{1}I+L_{K,D}+\lambda _{2}V_{D}^{T}V_{D})^{-
		\frac{1}{2}}(\lambda _{1}I+L_{K,D}+\lambda _{2}V_{D}^{T}V_{D})^{r-
		\frac{1}{2}}\big \|
	\\
	\nonumber
	&&\ \ \big \|(\lambda _{1}I+L_{K,D}+\lambda _{2}V_{D}^{T}V_{D})^{-r+
		\frac{1}{2}}(\lambda _{1}I+L_{K}+\lambda _{2}V_{D}^{T}V_{D})^{r-
		\frac{1}{2}}\big \|
	\\
	\nonumber
	&&\ \ \big \|(\lambda _{1}I+L_{K}+\lambda _{2}V_{D}^{T}V_{D})^{-r+
		\frac{1}{2}}(\lambda _{1}I+L_{K})^{r-\frac{1}{2}}\big \|
	\\
	\nonumber
	&&\ \ \big \|(\lambda _{1}I+L_{K})^{-r+\frac{1}{2}}L_{K}^{r-
		\frac{1}{2}}\big \|\big \|L_{K}^{1/2}g_{\rho}\big \|_{K}
	\\
	\nonumber
	&&\leq \lambda _{1}^{r}\mathcal{A}_{D,\lambda _{1},\lambda _{2},V_{D}}^{r}
	\cdot 2^{r}\big \|g_{\rho}\big \|_{\rho},
\end{eqnarray}
where we have used the fact that
\begin{eqnarray*}
	&&\big \|(\lambda _{1}I+L_{K,D}+\lambda _{2}V_{D}^{T}V_{D})^{-
		\frac{1}{2}}(\lambda _{1}I+L_{K,D}+\lambda _{2}V_{D}^{T}V_{D})^{r-
		\frac{1}{2}}\big \|
	\\
	&&=\big \|(\lambda _{1}I+L_{K,D}+\lambda _{2}V_{D}^{T}V_{D})^{r-1}
	\big \|\leq (\lambda _{1}+\lambda _{D}^{2}+\lambda _{2}\lambda _{V_{D}}^{2})^{r-1}
	\leq \lambda _{1}^{r-1}
\end{eqnarray*}
in which
\begin{equation*}
	\lambda _{D}=\inf _{f\in \mathcal{H}_{K}}
	\frac{\|S_{D}f\|_{l^{2}}}{\|f\|_{K}},\ \lambda _{V_{D}}=\inf _{f\in
		\mathcal{H}_{K}}\frac{\|V_{D}f\|_{K}}{\|f\|_{K}}.
\end{equation*}
This fact can be treated as a direct result of Proposition 3.1 in
\cite{multi1}. Note that $V_{D}$ is a bounded operator on
$\mathcal{H}_{K}$, then following a similar procedure with the estimate on
$\mathcal{T}_{2}$, we obtain
\begin{eqnarray}
	\nonumber
	&& \mathcal{T}_{3}\leq \lambda _{1}^{r-1}\lambda _{2}c_{V}2^{r}
	\mathcal{A}_{D,\lambda _{1},\lambda _{2},V_{D}}^{r}\big \|g_{\rho}
	\big \|_{\rho}.
\end{eqnarray}
Combining the three estimates, the desired bound follows.
\end{proof}
The following lemma is used to prove the main results. It is based on techniques
of Neumann expansion \cite{multiguo} and the second-order decomposition
\cite{sbsecond} for invertible operators on Banach space, namely, for any
invertible operators $A$ and $B$,
%
%e4.17 #&#
%e4.18 #&#
\begin{eqnarray}
&& A^{-1}-B^{-1}=B^{-1}(B-A)A^{-1}(B-A)B^{-1}+B^{-1}(B-A)B^{-1}\\
\label{sd1}
%%LEAP%%%\label{eq4.8}
&& \ \ \ \ \ \ \ \ \ \ \ \ \ \ \ \ =B^{-1}(B-A)B^{-1}(B-A)A^{-1}+B^{-1}(B-A)B^{-1}.
\label{sd2}
%%LEAP%%%\label{eq4.9}
\end{eqnarray}

%l2 #&#
\begin{lem}%
%%LEAP%%%\label{lem2}
\label{le2}
Let the sample set $D$ be drawn independently according to the probability
measure $\rho $. If $\lambda _{1},\lambda _{2}\in (0,1)$ satisfy
$2c_{V}\lambda _{2}=\lambda _{1}^{\max \{2r,1\}}$ for $r\in (0,1]$. Then
we have,
%
%e4.19 #&#
%e4.20 #&#
\begin{eqnarray}
	&&E_{\mathbf{z}^{|D|}}\bigg[\mathcal{A}_{D,\lambda _{1},\lambda _{2},V_{D}}^{s}
	\bigg]\leq \big(2\Gamma (2s+1)+\log ^{2s}2\big)\bigg(
	\frac{2\mathcal{B}_{|D|,\lambda _{1}}}{\sqrt{\lambda _{1}}}+1\bigg)^{2s},
	\ s\geq 0;\\
	\label{l21}
	%%LEAP%%%\label{eq4.10}
		&&E_{\mathbf{z}^{|D|}}\bigg[\big \|(\lambda _{1}I+L_{K})^{-
		\frac{1}{2}}(S_{D}^{T}y_{D}-L_{K,D}f_{\rho})\big \|_{K}^{s}\bigg]
	\leq \big(2\Gamma (s+1)+\log ^{s}2\big)\bigg(\frac{2M}{\kappa }
	\mathcal{B}_{|D|,\lambda _{1}}\bigg)^{s},\ s\geq 1.
	\label{l22}
	%%LEAP%%%\label{eq4.11}
\end{eqnarray}
\end{lem}
\begin{proof}
From Proposition 4.3 in \cite{multiguo}, we know, for any
$\delta \in (0,1)$, there holds
%
%e4.21 #&#
\begin{equation}
	\text{Prob}\bigg\{\mathcal{A}_{D,\lambda _{1},\lambda _{2},V_{D}}\leq
	\bigg(
	\frac{2\mathcal{B}_{|D|,\lambda _{1}}\log \frac{2}{\delta }}{\sqrt{\lambda _{1}}}+1
	\bigg)^{2}\bigg\}\geq 1-\delta .
	\label{eq4.12}
\end{equation}
Then it is easy to see, for any $\delta \in (0,2)$,
%
%e4.22 #&#
\begin{equation}
	\text{Prob}\bigg\{\mathcal{A}_{D,\lambda _{1},\lambda _{2},V_{D}}\leq
	\bigg(
	\frac{2\mathcal{B}_{|D|,\lambda _{1}}\log \frac{4}{\delta }}{\sqrt{\lambda _{1}}}+1
	\bigg)^{2}\bigg\}\geq 1-\frac{\delta }{2}.
	\label{eq4.13}
\end{equation}
Denote the random variable
$\xi =\mathcal{A}_{D,\lambda _{1},\lambda _{2},V_{D}}^{s}$, perform the
variable substitution
$t=\bigg(
\frac{2\mathcal{B}_{|D|,\lambda _{1}}\log \frac{2}{\delta }}{\sqrt{\lambda _{1}}}+1
\bigg)^{2s}\log ^{2s}\frac{4}{\delta }$ with $s\geq 0$, we know, for
$t>\bigg(\frac{2\mathcal{B}_{|D|,\lambda _{1}}}{\sqrt{\lambda _{1}}}+1
\bigg)^{2s}\log ^{2s}2$,
%
%e4.23 #&#
\begin{equation}
	\text{Prob}\bigg(\xi >t\bigg)=\text{Prob}\bigg(\xi ^{\frac{1}{s}}>t^{
		\frac{1}{s}}\bigg)\leq \frac{\delta }{2}=2\exp \bigg\{-
	\frac{t^{\frac{1}{2s}}}{\frac{2\mathcal{B}_{|D|,\lambda _{1}}}{\sqrt{\lambda _{1}}}+1}
	\bigg\}.
	\label{eq4.14}
\end{equation}
Then by using the formula
$E_{\mathbf{z}^{|D|}}[\xi ]=\int _{0}^{\infty}\text{Prob}(\xi >t)dt$, it
follows that
\begin{eqnarray}
	\nonumber
	&&E_{\mathbf{z}^{|D|}}[\xi ]\leq \bigg(
	\frac{2\mathcal{B}_{|D|,\lambda _{1}}}{\sqrt{\lambda _{1}}}+1\bigg)^{2s}
	\log ^{2s}2+\int _{\big(
		\frac{2\mathcal{B}_{|D|,\lambda _{1}}}{\sqrt{\lambda _{1}}}+1\big)^{2s}
		\log ^{2s}2}^{\infty}2\exp \bigg\{-
	\frac{t^{\frac{1}{2s}}}{\frac{2\mathcal{B}_{|D|,\lambda _{1}}}{\sqrt{\lambda _{1}}}+1}
	\bigg\}dt
	\\
	\nonumber
	&&\ \ \ \ \ \ \leq \bigg(
	\frac{2\mathcal{B}_{|D|,\lambda _{1}}}{\sqrt{\lambda _{1}}}+1\bigg)^{2s}
	\log ^{2s}2+4s\bigg(
	\frac{2\mathcal{B}_{|D|,\lambda _{1}}}{\sqrt{\lambda _{1}}}+1\bigg)^{2s}
	\int _{\log 2}^{\infty}e^{-x}x^{2s-1}dx
	\\
	\nonumber
	&&\ \ \ \ \ \ \leq \big(2\Gamma (2s+1)+\log ^{2s}2\big)\bigg(
	\frac{2\mathcal{B}_{|D|,\lambda _{1}}}{\sqrt{\lambda _{1}}}+1\bigg)^{2s}.
\end{eqnarray}
For the second inequality, note that \cite{s1} implies that, for any
$0<\delta <1$, with probability of at least $1-\delta $,
%
%e4.24 #&#
\begin{equation}
	\big \|(\lambda _{1}I+L_{K})^{-\frac{1}{2}}(S_{D}^{T}y_{D}-L_{K,D}f_{
		\rho})\big \|_{K}\leq \frac{2M}{\kappa }\mathcal{B}_{|D|,\lambda _{1}}
	\log \frac{2}{\delta }.
	\label{eq4.15}
\end{equation}
Then the same procedure of proving the first inequality implies the result.
\end{proof}

%p2 #&#
\begin{pro}
\label{prop2}
Assume that $|y|\leq M$ and {\eqref{L}} hold almost surely. Assume the regularity
condition {\eqref{regularity}} holds with some index $1/2\leq r\leq 1$. If
$\lambda _{1},\lambda _{2}\in (0,1)$ satisfy
$2c_{V}\lambda _{2}=\lambda _{1}^{2r}$. Then we have,
%
%e4.25 #&#
\begin{eqnarray}
	\nonumber
	&&E_{\mathbf{z}^{|D|}}\bigg[\big \|f_{D,\lambda _{1},\lambda _{2}}-f_{
		\rho}\big \|_{\rho}\bigg]\leq 4(2\Gamma (3)+\log ^{2}2)^{\frac{1}{2}}(2
	\Gamma (5)+\log ^{4}2)^{\frac{1}{2}}\bigg(
	\frac{2\mathcal{B}_{|D|,\lambda _{1}}}{\sqrt{\lambda _{1}}}+1\bigg)^{2}
	\frac{M}{\kappa }\mathcal{B}_{|D|,\lambda _{1}}
	\\
	&&\ \ \ \ \ \ \ \ \ \ \ \ \ \ \ \ \ +(2\Gamma (2r+1)+\log ^{2r}2)2^{r}(
	\lambda _{1}^{r}+\lambda _{1}^{r-1}\lambda _{2}c_{V})\|g_{\rho}\|_{
		\rho}\bigg(
	\frac{2\mathcal{B}_{|D|,\lambda _{1}}}{\sqrt{\lambda _{1}}}+1\bigg)^{2r}.
	\label{1main}
	%%LEAP%%%\label{eq4.16}
\end{eqnarray}
\end{pro}
\begin{proof}
Starting from {Proposition~\ref{ro1}}, taking expectations on both sides
and using Schwarz inequality, we have
\begin{eqnarray}
	\nonumber
	&&E_{\mathbf{z}^{|D|}}\bigg[\big \|f_{D,\lambda _{1},\lambda _{2}}-f_{
		\rho}\big\|_{\rho}\bigg]\leq 2E_{\mathbf{z}^{|D|}}\bigg[\mathcal{A}_{D,
		\lambda _{1},\lambda _{2},V_{D}}\big\|(\lambda _{1}I+L_{K})^{-
		\frac{1}{2}}(S_{D}^{T}y_{D}-L_{K,D}f_{\rho})\big\|_{K}\bigg]
	\\
	\nonumber
	&&\ \ \ \ \ \quad \quad \quad \quad \quad \quad \quad \quad \ +(
	\lambda _{1}^{r}+\lambda _{1}^{r-1}\lambda _{2}c_{V})2^{r}\big\|g_{
		\rho}\big\|_{\rho}E_{\mathbf{z}^{|D|}}\bigg[\mathcal{A}_{D,\lambda _{1},
		\lambda _{2},V_{D}}^{r}\bigg].
	\\
	\nonumber
	&&\quad \quad \quad \quad \quad \quad \quad \quad \ \leq 2\bigg\{E_{
		\mathbf{z}^{|D|}}\bigg[\big\|(\lambda _{1}I+L_{K})^{-\frac{1}{2}}(S_{D}^{T}y_{D}-L_{K,D}f_{
		\rho})\big\|_{K}^{2}\bigg]\bigg\}^{1/2}\bigg\{E_{\mathbf{z}^{|D|}}
	\bigg[\mathcal{A}_{D,\lambda _{1},\lambda _{2},V_{D}}^{2}\bigg]
	\bigg\}^{1/2}
	\\
	\nonumber
	&&\ \ \ \ \ \quad \quad \quad \quad \quad \quad \quad \ +(\lambda _{1}^{r}+
	\lambda _{1}^{r-1}\lambda _{2}c_{V})2^{r}\big\|g_{\rho}\big\|_{\rho}E_{
		\mathbf{z}^{|D|}}\bigg[\mathcal{A}_{D,\lambda _{1},\lambda _{2},V_{D}}^{r}
	\bigg].
	\\
	\nonumber
	&&\ \ \ \  \quad \quad \quad \quad \quad \quad \quad \leq 2(2\Gamma (3)+
	\log ^{2}2)^{1/2}\frac{2M}{\kappa }\mathcal{B}_{|D|,\lambda _{1}}(2
	\Gamma (5)+\log ^{4}2)^{1/2}\bigg(
	\frac{2\mathcal{B}_{|D|,\lambda _{1}}}{\sqrt{\lambda _{1}}}+1\bigg)^{2}
	\\
	\nonumber
	&&\quad \quad \quad \quad \quad \quad \quad \quad \quad +(\lambda _{1}^{r}+
	\lambda _{1}^{r-1}\lambda _{2}c_{V})2^{r}\big \|g_{\rho}\big \|_{\rho}(2
	\Gamma (2r+1)+\log ^{2r}2)\bigg(
	\frac{2\mathcal{B}_{|D|,\lambda _{1}}}{\sqrt{\lambda _{1}}}+1\bigg)^{2r},
\end{eqnarray}
in which the last inequality follows from {\eqref{l22}} with index
$s=2$ and {\eqref{l21}} with index $s=r$ in {Lemma~\ref{le2}}.
\end{proof}

Denote
%
%e4.26 #&#
\begin{equation}
\Omega _{D,\lambda _{1},\lambda _{2},V_{D}}=\big \|L_{K}^{1/2}(L_{K,
	\widehat{D}}+\lambda _{1}I+\lambda _{2}V_{D}^{T}V_{D})^{-1}\big \|.
\label{eq4.17}
\end{equation}
Now we start to bound
$\|f_{\widehat{D},\lambda _{1},\lambda _{2}}-f_{D,\lambda _{1},
\lambda _{2}}\|_{\rho}$ part in error decomposition of
$\|f_{\widehat{D},\lambda _{1},\lambda _{2}}-f_{\rho}\|_{\rho}$.
%
%p3 #&#
\begin{pro}%
%%LEAP%%%\label{prop3}
\label{pro3}
There holds almost surely
%
%e4.27 #&#
\begin{equation}
	\big \|f_{\widehat{D},\lambda _{1},\lambda _{2}}-f_{D,\lambda _{1},
		\lambda _{2}}\big \|_{\rho}\leq \Omega _{D,\lambda _{1},\lambda _{2},V_{D}}
	\bigg(\big \|\hat{S}_{D}^{T}y_{D}-S_{D}^{T}y_{D}\big \|_{K}+\big \|L_{K,D}-L_{K,
		\widehat{D}}\big \|\big \|f_{D,\lambda _{1},\lambda _{2}}\big \|_{K}
	\bigg).
	\label{mbdd}
	%%LEAP%%%\label{eq4.18}
\end{equation}
\end{pro}

\begin{proof}
Recalling the representations of
$f_{\widehat{D},\lambda _{1},\lambda _{2}}$ in {\eqref{algorithm}} and
$f_{D,\lambda _{1},\lambda _{2}}$ in {\eqref{fuzhuf}}, we have
\begin{eqnarray}
	\nonumber
	&&f_{\widehat{D},\lambda _{1},\lambda _{2}}-f_{D,\lambda _{1},
		\lambda _{2}}
	\\
	\nonumber
	&&= (L_{K,\widehat{D}}+\lambda _{1}I+\lambda _{2}V_{D}^{T}V_{D})^{-1}
	\hat{S}_{D}^{T}y_{D}-(L_{K,D}+\lambda _{1}I+\lambda _{2}V_{D}^{T}V_{D})^{-1}S_{D}^{T}y_{D}
	\\
	\nonumber
	&&=(L_{K,\widehat{D}}+\lambda _{1}I+\lambda _{2}V_{D}^{T}V_{D})^{-1}(
	\hat{S}_{D}^{T}y_{D}-S_{D}^{T}y_{D})
	\\
	\nonumber
	&&\ \ \ +\bigg[(L_{K,\widehat{D}}+\lambda _{1}I+\lambda _{2}V_{D}^{T}V_{D})^{-1}-(L_{K,D}+
	\lambda _{1}I+\lambda _{2}V_{D}^{T}V_{D})^{-1}\bigg]S_{D}^{T}y_{D}
	\\
	\nonumber
	&&=(L_{K,\widehat{D}}+\lambda _{1}I+\lambda _{2}V_{D}^{T}V_{D})^{-1}(
	\hat{S}_{D}^{T}y_{D}-S_{D}^{T}y_{D})
	\\
	\nonumber
	&&\ \ \ +(L_{K,\widehat{D}}+\lambda _{1}I+\lambda _{2}V_{D}^{T}V_{D})^{-1}(L_{K,D}-L_{K,
		\widehat{D}})f_{D,\lambda _{1},\lambda _{2}},
\end{eqnarray}
in which we have used the fact that, for any invertible operator $A$ and
$B$, $A^{-1}-B^{-1}=A^{-1}(B-A)B^{-1}$, with
$A=L_{K,\widehat{D}}+\lambda _{1}I+\lambda _{2}V_{D}^{T}V_{D}$ and
$B=L_{K,D}+\lambda _{1}I+\lambda _{2}V_{D}^{T}V_{D}$. Take
$L_{\rho _{X_{\mu}}}^{2}$-norm on both sides of the above equality and
use the triangle formula, it follows that
\begin{eqnarray}
	\nonumber
	&&\big \|f_{\widehat{D},\lambda _{1},\lambda _{2}}-f_{D,\lambda _{1},
		\lambda _{2}}\big \|_{\rho}\leq \big \|L_{K}^{1/2}(L_{K,\widehat{D}}+
	\lambda _{1}I+\lambda _{2}V_{D}^{T}V_{D})^{-1}(\hat{S}_{D}^{T}y_{D}-S_{D}^{T}y_{D})
	\big \|_{K}
	\\
	\nonumber
	&&\ \ \ \ \ \ \ \ \ \ \ \ \ \ \ \ \ \ \ \ \ \ \ \ \ \ \ \ \ \ \ \ \  +
	\big \|L_{K}^{1/2}(L_{K,\widehat{D}}+\lambda _{1}I+\lambda _{2}V_{D}^{T}V_{D})^{-1}(L_{K,D}-L_{K,
		\widehat{D}})f_{D,\lambda _{1},\lambda _{2}}\big \|_{K}
	\\
	\nonumber
	&&\quad \quad \quad \quad \quad \quad \quad \quad \quad \quad \leq
	\big \|L_{K}^{1/2}(L_{K,\widehat{D}}+\lambda _{1}I+\lambda _{2}V_{D}^{T}V_{D})^{-1}
	\big \|\big \|\hat{S}_{D}^{T}y_{D}-S_{D}^{T}y_{D}\big \|_{K}
	\\
	\nonumber
	&&\quad \quad \quad \quad \quad \quad \quad \quad \quad \quad \quad +
	\big \|L_{K}^{1/2}(L_{K,\widehat{D}}+\lambda _{1}I+\lambda _{2}V_{D}^{T}V_{D})^{-1}
	\big \|\big \|L_{K,D}-L_{K,\widehat{D}}\big \|\big \|f_{D,\lambda _{1},
		\lambda _{2}}\big \|_{K}
	\\
	\nonumber
	&&\ \ \ \quad \quad \quad \quad \quad \quad \quad \quad \quad =
	\Omega _{D,\lambda _{1},\lambda _{2},V_{D}}\bigg(\big \|\hat{S}_{D}^{T}y_{D}-S_{D}^{T}y_{D}
	\big \|_{K}+\big \|L_{K,D}-L_{K,\widehat{D}}\big \|\big \|f_{D,
		\lambda _{1},\lambda _{2}}\big \|_{K}\bigg).\qedhere
\end{eqnarray}
\end{proof}
The following result provides an estimate for
$\Omega _{D,\lambda _{1},\lambda _{2},V_{D}}$.
%
%l3 #&#
\begin{lem}%
%%LEAP%%%\label{lem3}
\label{omega}
Let $D$ be a sample drawn independently according to measure $\rho $, and
$\{x_{i,s}\}_{s=1}^{d_{i}}$ be a sample independently drawn according to
distribution $x_{i}$, $i=1,2,...,|D|$. Let the regularity condition {\eqref{regularity}} hold with some index $r\in (0,1]$ and the mapping
$K_{(\cdot )}:X_{\mu}\rightarrow \mathcal{L}(Y,\mathcal{H}_{K})$ be
$(\alpha ,L)$-H\"older continuous with $\alpha \in (0,1]$ and $L>0$. If
$\lambda _{1},\lambda _{2}\in (0,1)$ satisfy
$2c_{V}\lambda _{2}=\lambda _{1}^{\max \{2r,1\}}$, then we have
\begin{eqnarray}
	\nonumber
	&&\bigg\{E_{\mathbf{x}^{\mathbf{d},|D|}|\mathbf{z}^{|D|}}\big[\Omega _{D,
		\lambda _{1},\lambda _{2},V_{D}}^{2}\big]\bigg\}^{\frac{1}{2}}\leq
	\bigg(\frac{2\sqrt{2}}{\lambda _{1}^{3/2}}(2+\sqrt{\pi})^{\frac{1}{2}}\kappa L
	\frac{1}{|D|}\sum _{i=1}^{|D|}
	\frac{2^{\frac{\alpha +2}{2}}B_{\widetilde{K}}^{\frac{\alpha }{2}}}{d_{i}^{\frac{\alpha }{2}}}
	\bigg)\mathcal{A}_{D,\lambda _{1},\lambda _{2},V_{D}}
	\\
	\nonumber
	&&\quad \quad \quad \quad \quad \quad \quad \quad \quad \quad \quad
	\quad \quad \quad +\frac{2}{\lambda _{1}^{\frac{1}{2}}}\mathcal{A}_{D,
		\lambda _{1},\lambda _{2},V_{D}}^{1/2}.
\end{eqnarray}
\end{lem}
\begin{proof}
Divide $\Omega _{D,\lambda _{1},\lambda _{2},V_{D}}$ into two parts and
use triangle inequality of operator norm as follows,
\begin{eqnarray}
	\nonumber
	&&\Omega _{D,\lambda _{1},\lambda _{2},V_{D}}\leq \bigg \|L_{K}^{1/2}
	\bigg\{(L_{K,\widehat{D}}+\lambda _{1}I+\lambda _{2}V_{D}^{T}V_{D})^{-1}-(L_{K,D}+
	\lambda _{1}I+\lambda _{2}V_{D}^{T}V_{D})^{-1}\bigg\}\bigg \|
	\\
	\nonumber
	&&\ \ \ \ \ \ \ \ \ \ \ \ \ \ \ \ +\bigg \|L_{K}^{1/2}(L_{K,D}+
	\lambda _{1}I+\lambda _{2}V_{D}^{T}V_{D})^{-1}\bigg \|.
\end{eqnarray}
The first term is estimated as follows
\begin{eqnarray}
	\nonumber
	&&\bigg \|L_{K}^{1/2}\bigg\{(L_{K,\widehat{D}}+\lambda _{1}I+\lambda _{2}V_{D}^{T}V_{D})^{-1}-(L_{K,D}+
	\lambda _{1}I+\lambda _{2}V_{D}^{T}V_{D})^{-1}\bigg\}\bigg \|
	\\
	\nonumber
	&&=\bigg \|L_{K}^{1/2}(L_{K,D}+\lambda _{1}I+\lambda _{2}V_{D}^{T}V_{D})^{-1}(L_{K,D}-L_{K,
		\widehat{D}})(L_{K,\widehat{D}}+\lambda _{1}I+\lambda _{2}V_{D}^{T}V_{D})^{-1}
	\bigg \|
	\\
	\nonumber
	&&\leq \big \|L_{K}^{1/2}(L_{K}+\lambda _{1}I)^{-\frac{1}{2}}\big \|
	\big \|(\lambda _{1}I+L_{K})^{\frac{1}{2}}(L_{K}+\lambda _{1}I+
	\lambda _{2}V_{D}^{T}V_{D})^{-\frac{1}{2}}\big \|\big \|(L_{K}+
	\lambda _{1}I+\lambda _{2}V_{D}^{T}V_{D})^{\frac{1}{2}}
	\\
	\nonumber
	&&\ \ (L_{K,D}+\lambda _{1}I+\lambda _{2}V_{D}^{T}V_{D})^{-
		\frac{1}{2}}\big \|\big \|(L_{K,D}+\lambda _{1}I+\lambda _{2}V_{D}^{T}V_{D})^{-
		\frac{1}{2}}(L_{K}+\lambda _{1}I+\lambda _{2}V_{D}^{T}V_{D})^{
		\frac{1}{2}}\big \|
	\\
	\nonumber
	&&\ \ \big \|(L_{K}+\lambda _{1}I+\lambda _{2}V_{D}^{T}V_{D})^{-
		\frac{1}{2}}(L_{K}+\lambda _{1}I)^{\frac{1}{2}}\big \|\big \|(L_{K}+
	\lambda _{1}I)^{-\frac{1}{2}}\big \|\big \|L_{K,D}-L_{K,\widehat{D}}
	\big \|
	\\
	\nonumber
	&&\ \ \big \|(L_{K,\widehat{D}}+\lambda _{1}I+\lambda _{2}V_{D}^{T}V_{D})^{-1}
	\big \|
	\\
	\nonumber
	&&\leq \frac{2}{\lambda _{1}^{3/2}}\mathcal{A}_{D,\lambda _{1},
		\lambda _{2},V_{D}}\big \|L_{K,D}-L_{K,\widehat{D}}\big \|,
\end{eqnarray}
where we have used
\begin{equation*}
	\big \|L_{K}^{1/2}(L_{K}+\lambda _{1}I)^{-1/2}\big \|\leq 1,
\end{equation*}
\begin{equation*}
	\big \|(L_{K}+\lambda _{1}I)^{-1/2}\big \|\leq
	\frac{1}{\sqrt{\lambda _{1}}},
\end{equation*}
\begin{equation*}
	\big \|(L_{K,\widehat{D}}+\lambda _{1}I+\lambda _{2}V_{D}^{T}V_{D})^{-1}
	\big \|\leq \frac{1}{\lambda _{1}},
\end{equation*}
and the fact that for two self-adjoint operators $T_{1}$, $T_{2}$, there
holds $\|T_{1}T_{2}\|=\|T_{2}T_{1}\|$. The second term is estimated {as follows,}
\begin{eqnarray}
	\nonumber
	&&\bigg \|L_{K}^{1/2}(L_{K,D}+\lambda _{1}I+\lambda _{2}V_{D}^{T}V_{D})^{-1}
	\bigg \|
	\\
	\nonumber
	&&\leq \big \|L_{K}^{1/2}(\lambda _{1}I+L_{K})^{-\frac{1}{2}}\big \|
	\big \|(\lambda _{1}I+L_{K})^{\frac{1}{2}}(L_{K}+\lambda _{1}I+
	\lambda _{2}V_{D}^{T}V_{D})^{-\frac{1}{2}}\big \|\big \|(L_{K}+
	\lambda _{1}I+\lambda _{2}V_{D}^{T}V_{D})^{\frac{1}{2}}
	\\
	\nonumber
	&&\ \ \ (L_{K,D}+\lambda _{1}I+\lambda _{2}V_{D}^{T}V_{D})^{-
		\frac{1}{2}}\big \|\big \|(L_{K,D}+\lambda _{1}I+\lambda _{2}V_{D}^{T}V_{D})^{-
		\frac{1}{2}}\big \|
	\\
	\nonumber
	&& \leq \frac{\sqrt{2}}{\sqrt{\lambda _{1}}}\mathcal{A}_{D,\lambda _{1},
		\lambda _{2},V_{D}}^{1/2},
\end{eqnarray}
in which we have used the fact
\begin{equation*}
	\big \|(L_{K,D}+\lambda _{1}I+\lambda _{2}V_{D}^{T}V_{D})^{-
		\frac{1}{2}}\big \|\leq \frac{1}{\sqrt{\lambda _{1}}}
\end{equation*}
which follows from the analysis above. Then we know from the above analysis that
%
%e4.28 #&#
\begin{equation}
	\nonumber
	\Omega _{D,\lambda _{1},\lambda _{2},V_{D}}\leq
	\frac{2}{\lambda _{1}^{3/2}}\mathcal{A}_{D,\lambda _{1},\lambda _{2},V_{D}}
	\big \|L_{K,D}-L_{K,\widehat{D}}\big \|+
	\frac{\sqrt{2}}{\sqrt{\lambda _{1}}}\mathcal{A}_{D,\lambda _{1},
		\lambda _{2},V_{D}}^{1/2}.
\end{equation}
This directly implies
%
%e4.29 #&#
\begin{equation}
	\nonumber
	\Omega _{D,\lambda _{1},\lambda _{2},V_{D}}^{2}\leq
	\frac{8}{\lambda _{1}^{3}}\mathcal{A}_{D,\lambda _{1},\lambda _{2},V_{D}}^{2}
	\big \|L_{K,D}-L_{K,\widehat{D}}\big \|^{2}+\frac{4}{\lambda _{1}}
	\mathcal{A}_{D,\lambda _{1},\lambda _{2},V_{D}}.
\end{equation}
Taking expectations on both sides and using {\eqref{s1ld}}, we obtain that
\begin{eqnarray}
	\nonumber
	&&\bigg\{E_{\mathbf{x}^{\mathbf{d},|D|}|\mathbf{z}^{|D|}}\big[\Omega _{D,
		\lambda _{1},\lambda _{2},V_{D}}^{2}\big]\bigg\}\leq
	\frac{8}{\lambda _{1}^{3}}\mathcal{A}_{D,\lambda _{1},\lambda _{2},V_{D}}^{2}E_{
		\mathbf{x}^{d,|D|}|\mathbf{z}^{|D|}}\bigg[\big \|L_{K,D}-L_{K,
		\widehat{D}}\big \|^{2}\bigg]+\frac{4}{\lambda _{1}}\mathcal{A}_{D,
		\lambda _{1},\lambda _{2},V_{D}}
	\\
	\nonumber
	&&\quad \quad \quad \quad \quad \quad \quad \quad \quad \quad \quad
	\ \ \leq \bigg(\frac{8}{\lambda _{1}^{3}}\kappa ^{2}(2+\sqrt{\pi})L^{2}
	\frac{1}{|D|^{2}}\sum _{i=1}^{|D|}
	\frac{2^{\alpha +2}B_{\widetilde{K}}^{\alpha }}{d_{i}^{\alpha }}
	\bigg)\mathcal{A}_{D,\lambda _{1},\lambda _{2},V_{D}}^{2}
	\\
	\nonumber
	&&\quad \quad \quad \quad \quad \quad \quad \quad \quad \quad \quad
	\ \ \quad +\frac{4}{\lambda _{1}}\mathcal{A}_{D,\lambda _{1},\lambda _{2},V_{D}}.
\end{eqnarray}
Finally, we have
%
%e4.30 #&#
\begin{equation}
	\nonumber
	\bigg\{E_{\mathbf{x}^{\mathbf{d},|D|}|\mathbf{z}^{|D|}}\big[\Omega _{D,
		\lambda _{1},\lambda _{2},V_{D}}^{2}\big]\bigg\}^{\frac{1}{2}}\leq
	\bigg(\frac{2\sqrt{2}}{\lambda _{1}^{3/2}}(2+\sqrt{\pi})^{\frac{1}{2}}\kappa L
	\frac{1}{|D|}\sum _{i=1}^{|D|}
	\frac{2^{\frac{\alpha +2}{2}}B_{\widetilde{K}}^{\frac{\alpha }{2}}}{d_{i}^{\frac{\alpha }{2}}}
	\bigg)\mathcal{A}_{D,\lambda _{1},\lambda _{2},V_{D}}+
	\frac{2}{\lambda _{1}^{\frac{1}{2}}}\mathcal{A}_{D,\lambda _{1},
		\lambda _{2},V_{D}}^{1/2}.\qedhere
\end{equation}
\end{proof}
Before coming to the proof of {Theorem~\ref{thm1}}, the following expected bound for
$\big \|f_{D,\lambda _{1},\lambda _{2}}\big \|_{K}$ in {\eqref{mbdd}} is
needed.
%
%p4 #&#
\begin{pro}
\label{prop4}
Suppose $|y|\leq M$ almost surely. Let the regularity condition {\eqref{regularity}} hold with some index $1/2\leq r\leq 1$. If
$\lambda _{1},\lambda _{2}\in (0,1)$ satisfy
$2c_{V}\lambda _{2}=\lambda _{1}^{2r}$, then we have
\begin{eqnarray}
	\nonumber
	&&\bigg\{E_{\mathbf{z}^{|D|}}\bigg[\big \|f_{D,\lambda _{1},\lambda _{2}}
	\big \|_{K}^{2}\bigg]\bigg\}^{1/2}\leq 2\sqrt{6}(2\Gamma (5)+\log ^{4}2)^{1/2}
	\frac{M}{\kappa }
	\frac{\mathcal{B}_{|D|,\lambda _{1}}}{\sqrt{\lambda _{1}}}\bigg(
	\frac{2\mathcal{B}_{|D|,\lambda _{1}}}{\sqrt{\lambda _{1}}}+1\bigg)
	\\
	\nonumber
	&&\ \ \ \ +\sqrt{3}(\lambda _{1}+\lambda _{2}c_{V})\lambda _{1}^{r-
		\frac{3}{2}}2^{r-\frac{1}{2}}\big \|g_{\rho}\big \|_{\rho}(2\Gamma (4r-1)+
	\log ^{4r-2}2)^{1/2}\bigg(
	\frac{2\mathcal{B}_{|D|,\lambda _{1}}}{\sqrt{\lambda _{1}}}+1\bigg)^{2r-1}+
	\sqrt{3}\kappa ^{r-\frac{1}{2}}\big \|g_{\rho}\big \|_{\rho}
\end{eqnarray}
where $\mathcal{B}_{|D|,\lambda _{1}}$ is defined as in {\eqref{b}}.
\end{pro}
\begin{proof}
Dividing $f_{D,\lambda _{1},\lambda _{2}}$ into two parts and using the triangle
formula for RKHS norm, we have
\begin{eqnarray}
	\nonumber
	&&\big \|f_{D,\lambda _{1},\lambda _{2}}\big \|_{K}\leq \big \|f_{D,
		\lambda _{1},\lambda _{2}}-f_{\rho}\big \|_{K}+\big \|f_{\rho}\big \|_{K}
	\\
	\nonumber
	&&=\big \|(L_{K,D}+\lambda _{1}I+\lambda _{2}V_{D}^{T}V_{D})^{-1}(S_{D}^{T}y_{D}-L_{K,D}f_{
		\rho}-\lambda _{1}f_{\rho}-\lambda _{2}V_{D}^{T}V_{D}f_{\rho})\big \|_{K}+
	\big \|f_{\rho}\big \|_{K}
	\\
	\nonumber
	&&\leq \big \|(L_{K,D}+\lambda _{1}I+\lambda _{2}V_{D}^{T}V_{D})^{-1}(S_{D}^{T}y_{D}-L_{K,D}f_{
		\rho})\big \|_{K}
	\\
	\nonumber
	&&\ \ \  +(\lambda _{1}+\lambda _{2}c_{V})\big \|(L_{K,D}+\lambda _{1}I+
	\lambda _{2}V_{D}^{T}V_{D})^{-1}L_{K}^{r}g_{\rho}\big \|_{K}+\big \|f_{
		\rho}\big \|_{K}.
\end{eqnarray}
The first term is estimated as follows,
\begin{eqnarray}
	\nonumber
	&&\big \|(L_{K,D}+\lambda _{1}I+\lambda _{2}V_{D}^{T}V_{D})^{-1}(S_{D}^{T}y_{D}-L_{K,D}f_{
		\rho})\big \|_{K}
	\\
	\nonumber
	&&=\big \|(L_{K,D}+\lambda _{1}I+\lambda _{2}V_{D}^{T}V_{D})^{-
		\frac{1}{2}}\big \|\big \|(L_{K,D}+\lambda _{1}I+\lambda _{2}V_{D}^{T}V_{D})^{-
		\frac{1}{2}}(L_{K}+\lambda _{1}I+\lambda _{2}V_{D}^{T}V_{D})^{
		\frac{1}{2}}\big \|
	\\
	\nonumber
	&&\ \ \ \big \|(L_{K}+\lambda _{1}I+\lambda _{2}V_{D}^{T}V_{D})^{-
		\frac{1}{2}}(L_{K}+\lambda _{1}I)^{\frac{1}{2}}\big \|\big \|(L_{K}+
	\lambda _{1}I)^{-\frac{1}{2}}(S_{D}^{T}y_{D}-L_{K,D}f_{\rho})\big \|_{K}
	\\
	\nonumber
	&&\leq \frac{\sqrt{2}}{\sqrt{\lambda _{1}}}\mathcal{A}_{D,\lambda _{1},
		\lambda _{2},V_{D}}^{1/2}\big \|(L_{K}+\lambda _{1}I)^{-\frac{1}{2}}(S_{D}^{T}y_{D}-L_{K,D}f_{
		\rho})\big \|_{K}.
\end{eqnarray}
Recall that $1/2\leq r\leq 1$, then the second term is estimated {as follows,}
\begin{eqnarray}
	\nonumber
	&&(\lambda _{1}+\lambda _{2}c_{V})\big \|(L_{K,D}+\lambda _{1}I+
	\lambda _{2}V_{D}^{T}V_{D})^{-1}L_{K}^{r}g_{\rho}\big \|_{K}
	\\
	\nonumber
	&&\leq (\lambda _{1}+\lambda _{2}c_{V})\big \|(L_{K,D}+\lambda _{1}I+
	\lambda _{2}V_{D}^{T}V_{D})^{r-\frac{3}{2}}\big \|\big \|(L_{K,D}+
	\lambda _{1}I+\lambda _{2}V_{D}^{T}V_{D})^{-r+\frac{1}{2}}(L_{K}+
	\lambda _{1}I+\lambda _{2}V_{D}^{T}V_{D})^{r-\frac{1}{2}}\big \|
	\\
	\nonumber
	&&\ \ \ \big \|(L_{K}+\lambda _{1}I+\lambda _{2}V_{D}^{T}V_{D})^{-r+
		\frac{1}{2}}(L_{K}+\lambda _{1}I)^{r-\frac{1}{2}}\big \|\big \|(L_{K}+
	\lambda _{1}I)^{-r+\frac{1}{2}}L_{K}^{r-\frac{1}{2}}\big \|\big \|L_{K}^{1/2}g_{
		\rho}\big \|_{K}
	\\
	\nonumber
	&&\leq (\lambda _{1}+\lambda _{2}c_{V})\lambda _{1}^{r-\frac{3}{2}}2^{r-
		\frac{1}{2}}\big \|g_{\rho}\big \|_{\rho}\mathcal{A}_{D,\lambda _{1},
		\lambda _{2},V_{D}}^{r-\frac{1}{2}}.
\end{eqnarray}
Hence
\begin{eqnarray}
	\nonumber
	&&\big \|f_{D,\lambda _{1},\lambda _{2}}\big \|_{K}\leq
	\frac{\sqrt{2}}{\sqrt{\lambda _{1}}}\big \|(L_{K}+\lambda _{1}I)^{-
		\frac{1}{2}}(S_{D}^{T}y_{D}-L_{K,D}f_{\rho})\big \|_{K}\mathcal{A}_{D,
		\lambda _{1},\lambda _{2},V_{D}}^{1/2}
	\\
	\nonumber
	&&\ \ \ \ \ \ \ \ \ \ \ \ \ \ \ \ \ \ \ \ +(\lambda _{1}+\lambda _{2}c_{V})
	\lambda _{1}^{r-\frac{3}{2}}2^{r-\frac{1}{2}}\big \|g_{\rho}\big \|_{
		\rho}\mathcal{A}_{D,\lambda _{1},\lambda _{2},V_{D}}^{r-\frac{1}{2}}+
	\kappa ^{r-\frac{1}{2}}\big \|g_{\rho}\big \|_{\rho}.
\end{eqnarray}
After taking expectations, Jensen's inequality implies
\begin{eqnarray}
	\nonumber
	&&E_{\mathbf{z}^{|D|}}\bigg[\big \|f_{D,\lambda _{1},\lambda _{2}}
	\big \|_{K}^{2}\bigg]\leq \frac{6}{\lambda _{1}}E_{\mathbf{z}^{|D|}}
	\bigg[\big \|(L_{K}+\lambda _{1}I)^{-\frac{1}{2}}(S_{D}^{T}y_{D}-L_{K,D}f_{
		\rho})\big \|_{K}^{2}\mathcal{A}_{D,\lambda _{1},\lambda _{2},V_{D}}
	\bigg]
	\\
	\nonumber
	&&\quad \quad \quad \quad \quad \quad \quad \quad \quad \quad + 3(
	\lambda _{1}+\lambda _{2}c_{V})^{2}\lambda _{1}^{2r-3}2^{2r-1}\big \|g_{
		\rho}\big \|_{\rho}^{2}E_{\mathbf{z}^{|D|}}\bigg[\mathcal{A}_{D,
		\lambda _{1},\lambda _{2},V_{D}}^{2r-1}\bigg]+3\kappa ^{2r-1}\big \|g_{
		\rho}\big \|_{\rho}^{2}
	\\
	\nonumber
	&&\quad \quad \quad \quad \quad \quad \leq \frac{6}{\lambda _{1}}
	\Big\{E_{\mathbf{z}^{|D|}}\Big[\big \|(L_{K}+\lambda _{1}I)^{-
		\frac{1}{2}}(S_{D}^{T}y_{D}-L_{K,D}f_{\rho})\big \|_{K}^{4}\Big]
	\Big\}^{\frac{1}{2}}\Big\{E_{\mathbf{z}^{|D|}}\Big[\mathcal{A}_{D,\lambda _{1},\lambda _{2},V_{D}}^{2}
	\Big]\Big\}^{\frac{1}{2}}
	\\
	\nonumber
	&&\quad \quad \quad \quad \quad \quad \quad \quad +3(\lambda _{1}+
	\lambda _{2}c_{V})^{2}\lambda _{1}^{2r-3}2^{2r-1}\big \|g_{\rho}
	\big \|_{\rho}^{2}E_{\mathbf{z}^{|D|}}\bigg[\mathcal{A}_{D,\lambda _{1},
		\lambda _{2},V_{D}}^{2r-1}\bigg]+3\kappa ^{2r-1}\big \|g_{\rho}
	\big \|_{\rho}^{2}
	\\
	\nonumber
	&&\quad \quad \quad \quad \quad \quad \leq 24(2\Gamma (5)+\log ^{4}2)
	\frac{M^{2}}{\kappa ^{2}}
	\frac{\mathcal{B}_{|D|,\lambda _{1}}^{2}}{\lambda _{1}}\bigg(
	\frac{2\mathcal{B}_{|D|,\lambda _{1}}}{\sqrt{\lambda _{1}}}+1\bigg)^{2}
	\\
	\nonumber
	&&+3(\lambda _{1}+\lambda _{2}c_{V})^{2}\lambda _{1}^{2r-3}2^{2r-1}
	\big \|g_{\rho}\big \|_{\rho}^{2}(2\Gamma (4r-1)+\log ^{4r-2}2)\bigg(
	\frac{2\mathcal{B}_{|D|,\lambda _{1}}}{\sqrt{\lambda _{1}}}+1\bigg)^{4r-2}+3
	\kappa ^{2r-1}\big \|g_{\rho}\big \|_{\rho}^{2},
\end{eqnarray}
in which the second inequality follows from the Schwarz inequality, and
the third inequality follows from {Lemma~\ref{le2}} with $s=2$ and
$2r-1$ in {\eqref{l21}} and $s=4$ in {\eqref{l22}}. The desired result is obtained
after noting the fact
$\sqrt{a_{1}+a_{2}+a_{3}}\leq \sqrt{a_{1}}+\sqrt{a_{2}}+\sqrt{a_{3}}$ for
$a_{1},a_{2},a_{3}\geq 0$.
\end{proof}

\begin{proof}[Proof of {Theorem~\ref{thm1}}]
Return to {Proposition~\ref{pro3}}, after taking expectations on both sides, and
using Schwarz inequality, inequalities {\eqref{s1sd}},
{\eqref{s1ld}} and {Lemma~\ref{omega}}, we have
%
%e4.31 #&#
\begin{eqnarray}
	\nonumber
	&&E\bigg[\Omega _{D,\lambda _{1},\lambda _{2},V_{D}}\big \|\hat{S}_{D}^{T}y_{D}-S_{D}^{T}y_{D}
	\big \|\bigg]
	\\
	\nonumber
	&&\leq E_{\mathbf{z}^{|D|}}\bigg[\bigg\{E_{\mathbf{x}^{\mathbf{d},|D|}|
		\mathbf{z}^{|D|}}\big[\Omega _{D,\lambda _{1},\lambda _{2},V_{D}}^{2}
	\big]\bigg\}^{\frac{1}{2}}\bigg\{E_{\mathbf{x}^{\mathbf{d},|D|}|
		\mathbf{z}^{|D|}}\big[\big \|\hat{S}_{D}^{T}y_{D}-S_{D}^{T}y_{D}
	\big \|^{2}\big]\bigg\}^{\frac{1}{2}}\bigg]
	\\
	\nonumber
	&&\leq \bigg(\frac{2\sqrt{2}}{\lambda _{1}^{3/2}}(2+\sqrt{\pi})^{
		\frac{1}{2}}\kappa L
	\frac{2^{\frac{\alpha +2}{2}}B_{\widetilde{K}}^{\frac{\alpha }{2}}}{\widetilde{d}^{\frac{\alpha }{2}}}+
	\frac{2}{\lambda _{1}^{\frac{1}{2}}}\bigg)(2+\sqrt{\pi})^{\frac{1}{2}}LM
	\frac{2^{\frac{\alpha }{2}}B_{\widetilde{K}}^{\frac{\alpha }{2}}}{\widetilde{d}^{\frac{\alpha }{2}}}(2
	\Gamma (3)+\log ^{2}2)\bigg[\bigg(
	\frac{2\mathcal{B}_{|D|,\lambda _{1}}}{\sqrt{\lambda _{1}}}+1\bigg)^{2}
	\\
	\nonumber
	&&\quad +(2\Gamma (2)+\log 2)\bigg(
	\frac{2\mathcal{B}_{|D|,\lambda _{1}}}{\sqrt{\lambda _{1}}}+1\bigg)
	\bigg]
	\\
	&&\leq 2\bigg(\frac{2\sqrt{2}}{\lambda _{1}^{3/2}}(2+\sqrt{\pi})^{
		\frac{1}{2}}\kappa L
	\frac{2^{\frac{\alpha +2}{2}}B_{\widetilde{K}}^{\frac{\alpha }{2}}}{\widetilde{d}^{\frac{\alpha }{2}}}+
	\frac{2}{\lambda _{1}^{\frac{1}{2}}}\bigg)(2+\sqrt{\pi})^{\frac{1}{2}}LM
	\frac{2^{\frac{\alpha }{2}}B_{\widetilde{K}}^{\frac{\alpha }{2}}}{\widetilde{d}^{\frac{\alpha }{2}}}(2
	\Gamma (3)+\log ^{2}2)\bigg(
	\frac{2\mathcal{B}_{|D|,\lambda _{1}}}{\sqrt{\lambda _{1}}}+1\bigg)^{2}
	\label{bdd1}
	%%LEAP%%%\label{eq4.19}
\end{eqnarray}
and
%
%e4.32 #&#
\begin{eqnarray}
	\nonumber
	&&E\bigg[\Omega _{D,\lambda _{1},\lambda _{2},V_{D}}\big \|L_{K,D}-L_{K,
		\widehat{D}}\big \|\big \|f_{D,\lambda _{1},\lambda _{2}}\big \|_{K}
	\bigg]
	\\
	\nonumber
	&&\leq E_{\mathbf{z}^{|D|}}\bigg[E_{\mathbf{x}^{\mathbf{d},|D|}|
		\mathbf{z}^{|D|}}\Big[\Omega _{D,\lambda _{1},\lambda _{2},V_{D}}
	\big \|L_{K,D}-L_{K,\widehat{D}}\big \|\Big]\cdot \big \|f_{D,
		\lambda _{1},\lambda _{2}}\big \|_{K}\bigg]
	\\
	\nonumber
	&&\leq E_{\mathbf{z}^{|D|}}\bigg[\bigg\{E_{\mathbf{x}^{\mathbf{d},|D|}|
		\mathbf{z}^{|D|}}\Big[\Omega _{D,\lambda _{1},\lambda _{2},V_{D}}^{2}
	\Big]\bigg\}^{\frac{1}{2}}\bigg\{E_{\mathbf{x}^{\mathbf{d},|D|}|
		\mathbf{z}^{|D|}}\Big[\big \|L_{K,D}-L_{K,\widehat{D}}\big \|^{2}
	\Big]\bigg\}^{\frac{1}{2}}\big \|f_{D,\lambda _{1},\lambda _{2}}
	\big \|_{K}\bigg]
	\\
	\nonumber
	&&\leq \bigg(\frac{2\sqrt{2}}{\lambda _{1}^{3/2}}(2+\sqrt{\pi})^{
		\frac{1}{2}}\kappa L
	\frac{2^{\frac{\alpha +2}{2}}B_{\widetilde{K}}^{\frac{\alpha }{2}}}{\widetilde{d}^{\frac{\alpha }{2}}}+
	\frac{2}{\lambda _{1}^{\frac{1}{2}}}\bigg)\kappa L(2+\sqrt{\pi})^{
		\frac{1}{2}}
	\frac{2^{\frac{\alpha +2}{2}}B_{\widetilde{K}}^{\frac{\alpha }{2}}}{\widetilde{d}^{\frac{\alpha }{2}}}E_{
		\mathbf{z}^{|D|}}\Big[\big \|f_{D,\lambda _{1},\lambda _{2}}\big \|_{K}(
	\mathcal{A}_{D,\lambda _{1},\lambda _{2},V_{D}}+\mathcal{A}_{D,
		\lambda _{1},\lambda _{2},V_{D}}^{1/2})\Big]
	\\
	\nonumber
	&&\leq \bigg(\frac{2\sqrt{2}}{\lambda _{1}^{3/2}}(2+\sqrt{\pi})^{
		\frac{1}{2}}\kappa L
	\frac{2^{\frac{\alpha +2}{2}}B_{\widetilde{K}}^{\frac{\alpha }{2}}}{\widetilde{d}^{\frac{\alpha }{2}}}+
	\frac{2}{\lambda _{1}^{\frac{1}{2}}}\bigg)\kappa L(2+\sqrt{\pi})^{
		\frac{1}{2}}
	\frac{2^{\frac{\alpha +2}{2}}B_{\widetilde{K}}^{\frac{\alpha }{2}}}{\widetilde{d}^{\frac{\alpha }{2}}}
	\Big\{E_{\mathbf{z}^{|D|}}\Big[\big \|f_{D,\lambda _{1},\lambda _{2}}
	\big \|_{K}^{2}\Big]\Big\}^{\frac{1}{2}}\times
	\\
	\nonumber
	&&\bigg(\Big\{E_{\mathbf{z}^{|D|}}\Big[\mathcal{A}_{D,\lambda _{1},
		\lambda _{2},V_{D}}^{2}\Big]\Big\}^{\frac{1}{2}}+\Big\{E_{\mathbf{z}^{|D|}}
	\Big[\mathcal{A}_{D,\lambda _{1},\lambda _{2},V_{D}}\Big]\Big\}^{
		\frac{1}{2}}\bigg)
	\\
	\nonumber
	&&\leq 2\bigg(\frac{2\sqrt{2}}{\lambda _{1}^{3/2}}(2+\sqrt{\pi})^{
		\frac{1}{2}}\kappa L
	\frac{2^{\frac{\alpha +2}{2}}B_{\widetilde{K}}^{\frac{\alpha }{2}}}{\widetilde{d}^{\frac{\alpha }{2}}}+
	\frac{2}{\lambda _{1}^{\frac{1}{2}}}\bigg)\kappa L(2+\sqrt{\pi})^{
		\frac{1}{2}}
	\frac{2^{\frac{\alpha +2}{2}}B_{\widetilde{K}}^{\frac{\alpha }{2}}}{\widetilde{d}^{\frac{\alpha }{2}}}
	\Bigg[2\sqrt{6}(2\Gamma (5)+\log ^{4}2)^{\frac{1}{2}}
	\frac{M}{\kappa }\times
	\\
	\nonumber
	&&\ \ \frac{\mathcal{B}_{|D|,\lambda _{1}}}{\sqrt{\lambda _{1}}}
	\bigg(\frac{2\mathcal{B}_{|D|,\lambda _{1}}}{\sqrt{\lambda _{1}}}+1
	\bigg)+\sqrt{3}(\lambda _{1}+\lambda _{2}c_{V})\lambda _{1}^{r-
		\frac{3}{2}}2^{r-\frac{1}{2}}\big \|g_{\rho}\big \|_{\rho}(2\Gamma (4r-1)+
	\log ^{4r-2}2)^{\frac{1}{2}}\bigg(
	\frac{2\mathcal{B}_{|D|,\lambda _{1}}}{\sqrt{\lambda _{1}}}+1\bigg)^{2r-1}
	\\
	&&\ \ +\sqrt{3}\kappa ^{r-\frac{1}{2}}\big \|g_{\rho}\big \|_{\rho}
	\Bigg](2\Gamma (5)+\log ^{4}2)^{\frac{1}{2}}\bigg(
	\frac{2\mathcal{B}_{|D|,\lambda _{1}}}{\sqrt{\lambda _{1}}}+1\bigg)^{2}.
	\label{bdd2}
	%%LEAP%%%\label{eq4.20}
\end{eqnarray}
Combining {\eqref{bdd1}} and {\eqref{bdd2}} with {Proposition~\ref{pro3}} and
taking expectations on both sides of {\eqref{mbdd}}, we have
%
%e4.33 #&#
\begin{eqnarray}
	\nonumber
	&&E\Big[\big \|f_{\widehat{D},\lambda _{1},\lambda _{2}}-f_{D,
		\lambda _{1},\lambda _{2}}\big \|_{\rho}\Big]\leq 2\bigg(2\sqrt{2}(2+
	\sqrt{2})^{\frac{1}{2}}\kappa L
	\frac{2^{\frac{\alpha +2}{2}}B_{\widetilde{K}}^{\frac{\alpha }{2}}}{\lambda _{1}\widetilde{d}^{\frac{\alpha }{2}}}+2
	\bigg)\bigg(
	\frac{2\mathcal{B}_{|D|,\lambda _{1}}}{\sqrt{\lambda _{1}}}+1\bigg)^{2}
	\frac{2^{\frac{\alpha }{2}}B_{\widetilde{K}}^{\frac{\alpha }{2}}}{\lambda _{1}^{\frac{1}{2}}\widetilde{d}^{\frac{\alpha }{2}}}
	\times
	\\
	\nonumber
	&&\Bigg[(2+\sqrt{\pi})^{\frac{1}{2}}LM(2\Gamma (3)+\log ^{2}2)+2(2
	\Gamma (5)+\log ^{4}2)^{\frac{1}{2}}\Bigg(2\sqrt{6}(2\Gamma (5)+\log ^{4}2)^{
		\frac{1}{2}}\frac{M}{\kappa }
	\frac{\mathcal{B}_{|D|,\lambda _{1}}}{\sqrt{\lambda _{1}}}\bigg(
	\frac{2\mathcal{B}_{|D|,\lambda _{1}}}{\sqrt{\lambda _{1}}}+1\bigg)
	\\
	&&+\sqrt{3}(\lambda _{1}+\lambda _{2}c_{V})\lambda _{1}^{r-
		\frac{3}{2}}2^{r-\frac{1}{2}}\big \|g_{\rho}\big \|_{\rho}(2\Gamma (4r-1)+
	\log ^{4r-2}2)^{\frac{1}{2}}\bigg(
	\frac{2\mathcal{B}_{|D|,\lambda _{1}}}{\sqrt{\lambda _{1}}}+1\bigg)^{2r-1}+
	\sqrt{3}\kappa ^{r-\frac{1}{2}}\big \|g_{\rho}\big \|_{\rho}\Bigg)
	\Bigg].
	\label{yin}
	%%LEAP%%%\label{eq4.21}
\end{eqnarray}
Also, we have already obtained {\eqref{1main}}, noting that
$\big \|f_{\widehat{D},\lambda _{1},\lambda _{2}}-f_{\rho}\big \|_{
	\rho}\leq \big \|f_{\widehat{D},\lambda _{1},\lambda _{2}}-f_{D,
	\lambda _{1},\lambda _{2}}\big \|_{\rho}+\big \|f_{D,\lambda _{1},
	\lambda _{2}}-f_{\rho}\big \|_{\rho}$, we complete the proof.
\end{proof}

\begin{proof}[Proof of {Theorem~\ref{thm2}}]
Considering the selection of
\begin{equation*}
	\lambda _{1}=|D|^{-\frac{1}{2r+\beta }}, \lambda _{2}=
	\frac{1}{2c_{V}}|D|^{-\frac{2r}{2r+\beta }}, d_{1}=d_{2}=\cdots =d_{|D|}=|D|^{
		\frac{1+2r}{\alpha (2r+\beta )}},
\end{equation*}
and the capacity condition
$\mathcal{N}(\lambda _{1})\leq \mathcal{C}_{0}\lambda _{1}^{-\beta}$,
$\beta \in (0,1]$, we can obtain the following bounds hold,
%
%e4.34 #&#
\begin{equation}
	\frac{1}{|D|}\sum _{i=1}^{|D|}
	\frac{1}{\lambda _{1}^{\frac{1}{2}}d_{i}^{\frac{\alpha }{2}}}=|D|^{-
		\frac{\alpha }{2}\frac{1+2r}{\alpha (2r+\beta )}}|D|^{
		\frac{1}{2(2r+\beta )}}=|D|^{-\frac{r}{2r+\beta }},
	\label{e1}
	%%LEAP%%%\label{eq4.22}
\end{equation}
%
%e4.35 #&#
\begin{equation}
	\frac{1}{|D|}\sum _{i=1}^{|D|}
	\frac{1}{\lambda _{1}d_{i}^{\frac{\alpha }{2}}}=|D|^{
		\frac{1}{2r+\beta }}|D|^{-\frac{1+2r}{2(2r+\beta )}}=|D|^{
		\frac{2-1-2r}{2(2r+\beta )}}=|D|^{\frac{1-2r}{2(2r+\beta )}} \leq 1\
	\ (r\in [1/2,1]),
	\label{e2}
	%%LEAP%%%\label{eq4.23}
\end{equation}
%
%e4.36 #&#
\begin{equation}
	\nonumber
	\mathcal{B}_{|D|,\lambda _{1}}=\frac{2\kappa }{\sqrt{|D|}}\bigg(
	\frac{\kappa }{\sqrt{|D|\lambda _{1}}}+\sqrt{\mathcal{N}(\lambda _{1})}
	\bigg)\leq 2\kappa (\kappa +\sqrt{\mathcal{C}_{0}})|D|^{-
		\frac{r}{2r+\beta }},
\end{equation}
%
%e4.37 #&#
\begin{equation}
	\nonumber
	\frac{\mathcal{B}_{|D|,\lambda _{1}}}{\sqrt{\lambda _{1}}}\leq 2
	\kappa (\kappa +\sqrt{\mathcal{C}_{0}}),
\end{equation}
%
%e4.38 #&#
\begin{equation}
	\nonumber
	\frac{2\mathcal{B}_{|D|,\lambda _{1}}}{\sqrt{\lambda _{1}}}+1\leq 4
	\kappa (\kappa +\sqrt{\mathcal{C}_{0}})+1,
\end{equation}
%
%e4.39 #&#
\begin{equation}
	\nonumber
	(\lambda _{1}+\lambda _{2}c_{V})\lambda _{1}^{r-\frac{3}{2}}\leq (
	\lambda _{1}+\frac{\lambda _{1}^{2r}}{2})\lambda _{1}^{r-\frac{3}{2}}
	\leq \big(\frac{3}{2}\big)\lambda _{1}^{r-\frac{1}{2}}\leq
	\frac{3}{2},
\end{equation}
%
%e4.40 #&#
\begin{equation}
	\nonumber
	\lambda _{1}^{r}+\lambda _{1}^{r-1}\lambda _{2}c_{V}=\lambda _{1}^{r}+
	\frac{1}{2}\lambda _{1}^{3r-1}\leq \frac{3}{2}\lambda _{1}^{r}=
	\frac{3}{2}|D|^{-\frac{r}{2r+\beta }}.
\end{equation}
Substitute these bounds into {Theorem~\ref{thm1}}, we arrive at the following
bound
\begin{eqnarray}
	\nonumber
	&&E\bigg[\big\|f_{\widehat{D},\lambda _{1},\lambda _{2}}-f_{\rho}
	\big\|_{\rho}\bigg]
	\\
	\nonumber
	&&\leq \Big(2\sqrt{2}(2+\sqrt{\pi})^{\frac{1}{2}}\kappa L2^{
		\frac{\alpha +2}{2}}B_{\widetilde{K}}^{\frac{\alpha }{2}}+2\Big)(2+
	\sqrt{\pi})^{\frac{1}{2}}LM2^{\frac{\alpha }{2}}B_{\widetilde{K}}^{
		\frac{\alpha }{2}}(2\Gamma (3)+\log ^{2}2)\times \Big(4\kappa (
	\kappa +\sqrt{\mathcal{C}_{0}})+1\Big)^{2}|D|^{-\frac{r}{2r+\beta }}
	\\
	\nonumber
	&&+\Big(2\sqrt{2}(2+\sqrt{\pi})^{\frac{1}{2}}\kappa L2^{\frac{\alpha +2}{2}}B_{
		\widetilde{K}}^{\frac{\alpha }{2}}+2\Big)(2+\sqrt{\pi})^{\frac{1}{2}}
	\kappa L2^{\frac{\alpha +2}{2}}B_{\widetilde{K}}^{\frac{\alpha }{2}}
	\bigg[2\sqrt{6}(2\Gamma (5)+\log ^{4}2)^{\frac{1}{2}}M(2\kappa (
	\kappa +\sqrt{\mathcal{C}_{0}}))\times
	\\
	\nonumber
	&&\Big(4\kappa (\kappa +\sqrt{\mathcal{C}_{0}})+1\Big)+
	\frac{3\sqrt{3}}{2}2^{r-\frac{1}{2}}\|g_{\rho}\|_{\rho}(2\Gamma (4r-1)+
	\log ^{4r-1}2)^{\frac{1}{2}}\Big(4\kappa (\kappa +\sqrt{\mathcal{C}_{0}})+1
	\Big)^{2r-1}+\sqrt{3}\kappa ^{r-\frac{1}{2}}\|g_{\rho}\|_{\rho}\bigg]
	\times
	\\
	\nonumber
	&&(2\Gamma (5)+\log ^{4}2)^{\frac{1}{2}}\Big(4\kappa (\kappa +\sqrt{
		\mathcal{C}_{0}})+1\Big)^{2}|D|^{-\frac{r}{2r+\beta }}
	\\
	\nonumber
	&&+4(2\Gamma (3)+\log ^{2}2)^{\frac{1}{2}}(2\Gamma (5)+\log ^{4}2)
	\Big(4\kappa (\kappa +\sqrt{\mathcal{C}_{0}})+1\Big)^{2}
	\frac{M}{\kappa }2\kappa (\kappa +\sqrt{\mathcal{C}_{0}})|D|^{-
		\frac{r}{2r+\beta }}
	\\
	\nonumber
	&&+(2\Gamma (2r+1)+\log ^{2r}2)2^{r}\|g_{\rho}\|_{\rho}\Big(4\kappa (
	\kappa +\sqrt{\mathcal{C}_{0}})+1\Big)^{2r}\frac{3}{2}|D|^{-
		\frac{r}{2r+\beta }},
\end{eqnarray}
which directly implies the desired learning rate and completes the proof.
\end{proof}

%s4.2 #&#
\subsection{Analysis of the error bounds and rates when $r\in (0,1/2)$}
%%LEAP%%%\label{sec4.2}
\label{nonstandardcase}

Now we turn to prove the result in the nonstandard setting
$r\in (0,1/2)$ that corresponds to $f_{\rho}\notin \mathcal{H}_{K}$. In
this section, we need the following norms for subsequent estimates
%
%e4.41 #&#
%e4.42 #&#
\begin{eqnarray}
&&\Pi _{D,\lambda _{1}}=\big \|(\lambda _{1}I+L_{K})^{-\frac{1}{2}}(S_{D}^{T}y_{D}-L_{K}f_{
	\rho})\big \|_{K},\\
\label{pldef}
%%LEAP%%%\label{eq4.24}
&&\Xi _{D,\lambda _{1}}=\big \|(\lambda _{1}I+L_{K})^{-\frac{1}{2}}(L_{K}-L_{K,D})
\big \|.
\label{fldef}
%%LEAP%%%\label{eq4.25}
\end{eqnarray}
For $D$ i.i.d. drawn from $\rho $, and $|y|\leq M$ almost surely and for
any $0<\delta <1$, \cite{s1} and \cite{sbsecond} imply that, with probability
at least $1-\delta $, there holds,
\begin{eqnarray}
\nonumber
&&\Pi _{D,\lambda _{1}}\leq 2M(\kappa +1)\mathcal{B}_{|D|,\lambda _{1}}'
\log \frac{2}{\delta },
\\
\nonumber
&&\Xi _{D,\lambda _{1}}\leq 2(\kappa ^{2}+\kappa )\mathcal{B}_{|D|,
	\lambda _{1}}'\log \frac{2}{\delta }.
\end{eqnarray}
Then, same procedures with the proof of {Lemma~\ref{le2}} imply the following
results.
%
%l4 #&#
\begin{lem}%
%%LEAP%%%\label{lem4}
\label{belem}
Let the sample set $D$ be drawn independently according to probability
measure $\rho $. If $|y|\leq M$ almost surely, then we have,
\begin{eqnarray}
	\nonumber
	&&E_{\mathbf{z}^{|D|}}\Big[\Pi _{D,\lambda _{1}}^{s}\Big]\leq (2
	\Gamma (s+1)+\log ^{s}2)\Big(2M(\kappa +1)\mathcal{B}_{|D|,\lambda _{1}}'
	\Big)^{s},s\geq 1;
	\\
	\nonumber
	&&E_{\mathbf{z}^{|D|}}\Big[\Xi _{D,\lambda _{1}}^{s}\Big]\leq (2
	\Gamma (s+1)+\log ^{s}2)\Big(2(\kappa ^{2}+\kappa )\mathcal{B}_{|D|,
		\lambda _{1}}'\Big)^{s},s\geq 1,
\end{eqnarray}
in which $\mathcal{B}_{|D|,\lambda _{1}}'$ is defined as in {\eqref{bp}}.
\end{lem}

Now we prepare to prove the main results when $f_{\rho}$ {does not} lie in
$\mathcal{H}_{K}$. We use the following function
%
%e4.43 #&#
\begin{equation}
f_{\lambda _{1}}=\arg \min _{f\in \mathcal{H}_{K}}\{\|f-f_{\rho}\|_{L_{
		\rho _{X_{\mu}}}}^{2}+\lambda _{1}\|f\|_{K}^{2}\}
\label{eq4.26}
\end{equation}
which lies in $\mathcal{H}_{K}$ as a bridge to perform further operator
analysis. $f_{\lambda _{1}}$ in fact has the operator representation
$f_{\lambda _{1}}=(L_{K}+\lambda _{1}I)^{-1}L_{K}f_{\rho}$. We derive the
following estimate.
%
%p5 #&#
\begin{pro}%
%%LEAP%%%\label{prop5}
\label{rx}
Assume that $|y|\leq M$ and {\eqref{L}} hold almost surely. Suppose the parameters
$\lambda _{1}, \lambda _{2}\in (0,1)$ satisfy
$2c_{V}\lambda _{2}=\lambda _{1}$. Then there holds almost surely
%
%e4.44 #&#
\begin{equation}
	\nonumber
	\max \bigg\{\big \|f_{D,\lambda _{1},\lambda _{2}}-f_{\lambda _{1}}
	\big \|_{\rho},\sqrt{\lambda _{1}}\big \|f_{D,\lambda _{1},\lambda _{2}}-f_{
		\lambda _{1}}\big \|_{K}\bigg\}\leq 2\mathcal{A}_{D,\lambda _{1},
		\lambda _{2},V_{D}}\Pi _{D,\lambda _{1}}+4\Xi _{D,\lambda _{1}}
	\mathcal{A}_{D,\lambda _{1},\lambda _{2},V_{D}}\big \|f_{\lambda _{1}}
	\big \|_{K}+\frac{2\lambda _{2}c_{V}}{\sqrt{\lambda _{1}}}\big \|f_{
		\lambda _{1}}\big \|_{K}.
\end{equation}
\end{pro}
\begin{proof}
Start from the following decomposition
\begin{eqnarray}
	\nonumber
	&&f_{D,\lambda _{1},\lambda _{2}}-f_{\lambda _{1}}=(L_{K,D}+\lambda _{1}I+
	\lambda _{2}V_{D}^{T}V_{D})^{-1}S_{D}^{T}y_{D}-(L_{K}+\lambda _{1}I)^{-1}L_{K}f_{
		\rho}
	\\
	\nonumber
	&&\quad \quad \quad \quad \quad \quad \ =(L_{K,D}+\lambda _{1}I+
	\lambda _{2}V_{D}^{T}V_{D})^{-1}(S_{D}^{T}y_{D}-L_{K}f_{\rho})
	\\
	\nonumber
	&&\quad \quad \quad \quad \quad \quad \ \ \ +\Big[(L_{K,D}+\lambda _{1}I+
	\lambda _{2}V_{D}^{T}V_{D})^{-1}-(L_{K}+\lambda _{1}I)^{-1}\Big]L_{K}f_{
		\rho}
	\\
	\nonumber
	&&\quad \quad \quad \quad \quad \quad \ =(L_{K,D}+\lambda _{1}I+
	\lambda _{2}V_{D}^{T}V_{D})^{-1}(S_{D}^{T}y_{D}-L_{K}f_{\rho})
	\\
	\nonumber
	&&\quad \quad \quad \quad \quad \quad \ \ \ +\Big[(L_{K,D}+\lambda _{1}I+
	\lambda _{2}V_{D}^{T}V_{D})^{-1}-(L_{K}+\lambda _{1}I+\lambda _{2}V_{D}^{T}V_{D})^{-1}
	\Big]L_{K}f_{\rho}
	\\
	\nonumber
	&&\quad \quad \quad \quad \quad \quad \ \ \ +\Big[(L_{K}+\lambda _{1}I+
	\lambda _{2}V_{D}^{T}V_{D})^{-1}-(L_{K}+\lambda _{1}I)^{-1}\Big]L_{K}f_{
		\rho}.
\end{eqnarray}
Then use the fact $\|h\|_{\rho}=\|L_{K}^{1/2}h\|_{K}$ for any
$h\in L_{\rho _{X_{\mu}}}^{2}$, we have
\begin{eqnarray}
	\nonumber
	&&\max \big\{\big \|f_{D,\lambda _{1},\lambda _{2}}-f_{\lambda _{1}}
	\big \|_{\rho},\sqrt{\lambda _{1}}\big \|f_{D,\lambda _{1},\lambda _{2}}-f_{
		\lambda _{1}}\big \|_{K}\big\}
	\\
	\nonumber
	&&\leq \Big\|(L_{K}+\lambda _{1}I)^{\frac{1}{2}}(L_{K,D}+\lambda _{1}I+
	\lambda _{2}V_{D}^{T}V_{D})^{-1}(S_{D}^{T}y_{D}-L_{K}f_{\rho})\Big\|_{K}
	\\
	\nonumber
	&&\ \ +\Big\|(L_{K}+\lambda _{1}I)^{\frac{1}{2}}(L_{K}+\lambda _{1}I+
	\lambda _{2}V_{D}^{T}V_{D})^{-1}(L_{K}-L_{K,D})(L_{K,D}+\lambda _{1}I+
	\lambda _{2}V_{D}^{T}V_{D})^{-1}L_{K}f_{\rho}\Big\|_{K}
	\\
	\nonumber
	&&\ \ +\Big\|(L_{K}+\lambda _{1}I)^{\frac{1}{2}}(L_{K}+\lambda _{1}I)^{-1}
	\lambda _{2}V_{D}^{T}V_{D}(L_{K}+\lambda _{1}I+\lambda _{2}V_{D}^{T}V_{D})^{-1}L_{K}f_{
		\rho}\Big\|_{K}
	\\
	\nonumber
	&&:=\widetilde{\mathcal{T}}_{1}+\widetilde{\mathcal{T}}_{2}+
	\widetilde{\mathcal{T}}_{3}.
\end{eqnarray}
We estimate $\widetilde{\mathcal{T}}_{1}$,
$\widetilde{\mathcal{T}}_{2}$, $\widetilde{\mathcal{T}}_{3}$ as follows
Decompose the operators in $\widetilde{\mathcal{T}}_{1}$, we have
\begin{eqnarray}
	\nonumber
	&&\widetilde{\mathcal{T}}_{1}=\Big\|(L_{K}+\lambda _{1}I)^{
		\frac{1}{2}}(L_{K}+\lambda _{1}I+\lambda _{2}V_{D}^{T}V_{D})^{-
		\frac{1}{2}}(L_{K}+\lambda _{1}I+\lambda _{2}V_{D}^{T}V_{D})^{
		\frac{1}{2}}
	\\
	\nonumber
	&&\ \ \ \ \ \ (L_{K,D}+\lambda _{1}I+\lambda _{2}V_{D}^{T}V_{D})^{-
		\frac{1}{2}}(L_{K,D}+\lambda _{1}I+\lambda _{2}V_{D}^{T}V_{D})^{-
		\frac{1}{2}}(L_{K}+\lambda _{1}I+\lambda _{2}V_{D}^{T}V_{D})^{
		\frac{1}{2}}
	\\
	\nonumber
	&&\ \ \ \ \ \ (L_{K}+\lambda _{1}I+\lambda _{2}V_{D}^{T}V_{D})^{-
		\frac{1}{2}}(L_{K}+\lambda _{1}I)^{\frac{1}{2}}(L_{K}+\lambda _{1}I)^{-
		\frac{1}{2}}(S_{D}^Ty_{D}-L_{K}f_{\rho})\Big\|_{K}
	\\
	\nonumber
	&&\ \ \ \leq \Big\|(L_{K}+\lambda _{1}I)^{\frac{1}{2}}(L_{K}+\lambda _{1}I+
	\lambda _{2}V_{D}^{T}V_{D})^{-\frac{1}{2}}\Big\|\Big\|(L_{K}+\lambda _{1}I+
	\lambda _{2}V_{D}^{T}V_{D})^{\frac{1}{2}}
	\\
	\nonumber
	&&\ \ \ \ \ \ (L_{K,D}+\lambda _{1}I+\lambda _{2}V_{D}^{T}V_{D})^{-
		\frac{1}{2}}\Big\|\Big\|(L_{K,D}+\lambda _{1}I+\lambda _{2}V_{D}^{T}V_{D})^{-
		\frac{1}{2}}(L_{K}+\lambda _{1}I+\lambda _{2}V_{D}^{T}V_{D})^{
		\frac{1}{2}}\Big\|
	\\
	\nonumber
	&&\ \ \ \ \ \ \Big\|(L_{K}+\lambda _{1}I+\lambda _{2}V_{D}^{T}V_{D})^{-
		\frac{1}{2}}(L_{K}+\lambda _{1}I)^{\frac{1}{2}}\Big\|\Big\|(L_{K}+
	\lambda _{1}I)^{-\frac{1}{2}}(S_{D}^Ty_{D}-L_{K}f_{\rho})\Big\|_{K}
	\\
	\nonumber
	&&\ \ \ \leq 2^{\frac{1}{2}}\mathcal{A}_{D,\lambda _{1},\lambda _{2},V_{D}}
	\cdot 2^{\frac{1}{2}}\Pi _{D,\lambda _{1}}=2\mathcal{A}_{D,\lambda _{1},
		\lambda _{2},V_{D}}\Pi _{D,\lambda _{1}}.
\end{eqnarray}
$\widetilde{\mathcal{T}}_{2}$ is estimated as follows,
\begin{eqnarray}
	\nonumber
	&&\widetilde{\mathcal{T}}_{2}\leq \Big \|(L_{K}+\lambda _{1}I)^{
		\frac{1}{2}}(L_{K}+\lambda _{1}I+\lambda _{2}V_{D}^{T}V_{D})^{-
		\frac{1}{2}}\Big \|\Big \|(L_{K}+\lambda _{1}I+\lambda _{2}V_{D}^{T}V_{D})^{-
		\frac{1}{2}}(L_{K}+\lambda _{1}I)^{\frac{1}{2}}\Big \|
	\\
	\nonumber
	&&\ \ \ \  \ \Big \|(L_{K}+\lambda _{1}I)^{-\frac{1}{2}}(L_{K}-L_{K,D})
	\Big \|\Big \|(L_{K,D}+\lambda _{1}I+\lambda _{2}V_{D}^{T}V_{D})^{-1}(L_{K}+
	\lambda _{1}I+\lambda _{2}V_{D}^{T}V_{D})\Big \|
	\\
	\nonumber
	&&\ \ \ \ \ \Big \|(L_{K}+\lambda _{1}I+\lambda _{2}V_{D}^{T}V_{D})^{-1}(L_{K}+
	\lambda _{1}I)\Big \|\Big \|(L_{K}+\lambda _{1}I)^{-1}L_{K}f_{\rho}
	\Big \|_{K}
	\\
	\nonumber
	&&\ \ \ \leq 2^{\frac{1}{2}}\cdot 2^{\frac{1}{2}}\cdot 2\Xi _{D,
		\lambda _{1}}\mathcal{A}_{D,\lambda _{1},\lambda _{2},V_{D}}\big \|f_{
		\lambda _{1}}\big \|_{K}=4\Xi _{D,\lambda _{1}}\mathcal{A}_{D,
		\lambda _{1},\lambda _{2},V_{D}}\big \|f_{\lambda _{1}}\big \|_{K}.
\end{eqnarray}
$\widetilde{\mathcal{T}}_{3}$ is estimated {as follows,}
\begin{eqnarray}
	\nonumber
	&&\widetilde{\mathcal{T}}_{3}\leq \Big \|(L_{K}+\lambda _{1}I)^{
		\frac{1}{2}}(L_{K}+\lambda _{1}I)^{-\frac{1}{2}}(L_{K}+\lambda _{1}I)^{-
		\frac{1}{2}}\lambda _{2}V_{D}^{T}V_{D}
	\\
	\nonumber
	&&\ \ \ \ \ \ (L_{K}+\lambda _{1}I+\lambda _{2}V_{D}^{T}V_{D})^{-1}(L_{K}+
	\lambda _{1}I)(L_{K}+\lambda _{1}I)^{-1}L_{K}f_{\rho}\Big \|_{K}
	\\
	\nonumber
	&&\ \ \ \leq \big \|(L_{K}+\lambda _{1}I)^{-\frac{1}{2}}\big \|
	\lambda _{2}c_{V}\big \|(L_{K}+\lambda _{1}I+\lambda _{2}V_{D}^{T}V_{D})^{-1}(L_{K}+
	\lambda _{1}I)\big \|\big \|(L_{K}+\lambda _{1}I)^{-1}L_{K}f_{\rho}
	\big \|_{K}
	\\
	\nonumber
	&&\ \ \ \leq \lambda _{2}c_{V}\frac{1}{\sqrt{\lambda _{1}}}2\big \|f_{
		\lambda _{1}}\big \|_{K}.
\end{eqnarray}
The desired result is obtained after combining the above three estimates
for $\widetilde{\mathcal{T}}_{1}$, $\widetilde{\mathcal{T}}_{2}$,
$\widetilde{\mathcal{T}}_{3}$.
\end{proof}

%r1 #&#
\begin{rmk}%
%%LEAP%%%\label{rem1}
\label{decomposition_discuss}
Under the regularity condition $f_{\rho}=L_{K}^{r}(g_{\rho})$, $r\in (0,1/2)$ for some $g_{\rho}\in L_{\rho _{X_{\mu}}}^{2}$, the main difference with case $r\in [1/2,1]$ is that the regression function $f_{\rho}$ no longer lies in $\mathcal{H}_{K}$ any more. {Proposition~\ref{rx}} is derived to overcome the difficulties. It plays an important role in the following proofs. In fact, when $r\in (0,1/2)$, the previous estimates in {Proposition~\ref{ro1}} for $\left \|f_{D,\lambda _{1},\lambda _{2}}-f_{\rho}\right \|_{\rho}$ will fail for further deriving learning rates of $\|f_{\widehat{D},\lambda _{1},\lambda _{2}}-f_{\rho}\|_{\rho}$. The main gap and difficulty arise in previous second-term $\mathcal{T}_{2}$ and third-term $\mathcal{T}_{3}$ estimates. When $r\in (0,1/2)$, if we continue handling $\left \|f_{D,\lambda _{1},\lambda _{2}}-f_{\rho}\right \|_{\rho}$ and decomposing term $\mathcal{T}_{2}$ as that in the proof of {Proposition~\ref{ro1}}, it can be observed that the operator norms
\begin{equation*}\big \|(\lambda _{1}I+L_{K,D}+\lambda _{2}V_{D}^{T}V_{D})^{-r+\frac{1}{2}}(\lambda _{1}I+L_{K}+\lambda _{2}V_{D}^{T}V_{D})^{r-\frac{1}{2}}\big \|,\end{equation*}
\begin{equation*}\big \|(\lambda _{1}I+L_{K}+\lambda _{2}V_{D}^{T}V_{D})^{-r+\frac{1}{2}}(\lambda _{1}I+L_{K})^{r-\frac{1}{2}}\big \|,\end{equation*}
and
\begin{equation*}
	\big \|(\lambda _{1}I+L_{K})^{-r+\frac{1}{2}}L_{K}^{r-\frac{1}{2}}\big \|
\end{equation*}
do not possess effective upper bounds for further estimates. Hence, they can
not be used to derive learning rates of $\|f_{\widehat{D},\lambda
	_{1},\lambda _{2}}-f_{\rho}\|_{\rho}$ in the hard learning scenario where
$r\in (0,1/2)$. The previous work in \cite{fang} also suffered from similar
difficulty, and hence only learning rates related to standard regularity
index $r\in [1/2,1]$ can be derived. Till now, we have described one of the
main reasons why previous analysis fails in the hard learning scenario. To
overcome the difficulties, in {Proposition~\ref{rx}}, we introduce data-free
representation $f_{\lambda _{1}}$ and make estimates of the RKHS norm and
$L_{\rho _{X_{\mu}}}^{2}$-norm of $f_{D,\lambda _{1},\lambda _{2}}-f_{\lambda
	_{1}}$. Due to the nice representation of $f_{\lambda _{1}}$, the
decomposition for $\widetilde{\mathcal{T}}_{1}$,
$\widetilde{\mathcal{T}}_{2}$ and $\widetilde{\mathcal{T}}_{3}$ successfully
excludes the interference of the regularity index of $r$ in the operator
decomposition process. After using the operator decomposition for
$\widetilde{\mathcal{T}}_{1}$, $\widetilde{\mathcal{T}}_{2}$ and
$\widetilde{\mathcal{T}}_{3}$, the corresponding upper bounds have
essentially changed. We finally derive a mild bound in {Proposition~\ref{rx}},
which can be used to estimate learning rates further. The upper bound in
{Proposition~\ref{rx}} also influences the error bounds for the second-stage
estimate of $f_{\widehat{D},\lambda _{1},\lambda _{2}}-f_{D,\lambda
	_{1},\lambda _{2}}$, this fact can be witnessed throughout the rest of this
section (see {Proposition~\ref{rfd}} and the proof of {Theorem~\ref{thm3}}).
Furthermore, in contrast to previous work, to capture more potential features
of regularized distribution regression scheme, we have proposed a novel
multi-penalty distribution regression scheme and considered an additional
penalty induced by the operator $V_{D}$ that previous works on distribution
regression have not explored yet. The participation of the operator $V_{D}$
also increases the difficulty of the operator analysis process in the
estimates mentioned above compared with previous works \cite{fang},
\cite{only}. Compared with previous analyses in works \cite{fang},
\cite{only}, we have already introduced two new crucial norms $\mathcal{A}_{D,\lambda _{1},\lambda _{2},V_{D}}=\big\|(\lambda
_{1}I+L_{K}+\lambda _{2}V_{D}^{T}V_{D})(\lambda _{1}I+L_{K,D}+\lambda
_{2}V_{D}^{T}V_{D})^{-1}\big\|$ and $\Omega _{D,\lambda _{1},\lambda _{2},V_{D}}=\big \|L_{K}^{1/2}
(L_{K,\widehat{D}}+\lambda _{1}I+\lambda _{2}V_{D}^{T}V_{D})^{-1}\big \|$
to overcome the difficulties arising from the operator $V_{D}$. From the
analysis above, the two quantities successfully realize effective bounds for
operator decompositions in our new multi-penalty setting. Hence they overcame
the current complicated multi-penalty analysis environment and finally
produced satisfactory theoretical results on learning rates.
\end{rmk}
Based on the estimate in {Proposition~\ref{rx}}, the following expected error
bound is obtained.
%
%p6 #&#
\begin{pro}%
%%LEAP%%%\label{prop6}
\label{r2b}
Assume that $|y|\leq M$ and {\eqref{L}} hold almost surely. Let the regularity
condition {\eqref{regularity}} hold with some index $0< r<1/2$. If
$\lambda _{1},\lambda _{2}\in (0,1)$ satisfy
$2c_{V}\lambda _{2}=\lambda _{1}$. Then we have,
%
%e4.45 #&#
\begin{eqnarray}
	\nonumber
	&&E_{\mathbf{z}^{|D|}}\bigg[\big \|f_{D,\lambda _{1},\lambda _{2}}-f_{
		\rho}\big \|_{\rho}\bigg]\leq 4(2\Gamma (5)+\log ^{4}2)^{\frac{1}{2}}(2
	\Gamma (3)+\log ^{2}2)^{\frac{1}{2}}\bigg(
	\frac{2\mathcal{B}_{|D|,\lambda _{1}}}{\sqrt{\lambda _{1}}}+1\bigg)^{2}
	\Big[M(\kappa +1)\mathcal{B}_{|D|,\lambda _{1}}'
	\\
	&&\ \ \ \ \ \ \ \ \ \ \ \ \ \ \ \ \ \ \ \ \ \ \ \ \ \ \ \ \ \ \ +\|g_{
		\rho}\|_{\rho}(\kappa ^{2}+\kappa )\lambda _{1}^{r-\frac{1}{2}}
	\mathcal{B}_{|D|,\lambda _{1}}'\Big]+2\|g_{\rho}\|_{\rho}c_{V}
	\lambda _{2}\lambda _{1}^{r-1}+\lambda _{1}^{r}\|g_{\rho}\|_{\rho}.
	\label{bdd3}
	%%LEAP%%%\label{eq4.27}
\end{eqnarray}
\end{pro}
\begin{proof}
Note that the following decomposition holds,
%
%e4.46 #&#
\begin{equation}
	E_{\mathbf{z}^{|D|}}\Big[\big \|f_{D,\lambda _{1},\lambda _{2}}-f_{
		\rho}\big \|_{\rho}\Big]\leq E_{\mathbf{z}^{|D|}}\Big[\big \|f_{D,
		\lambda _{1},\lambda _{2}}-f_{\lambda _{1}}\big \|_{\rho}\Big]+E_{
		\mathbf{z}^{|D|}}\Big[\big \|f_{\lambda _{1}}-f_{\rho}\big \|_{\rho}
	\Big].
	\label{eq4.28}
\end{equation}
The one stage estimate in \cite{s4} shows that
$\big \|f_{\lambda _{1}}-f_{\rho}\big \|_{\rho}\leq \lambda _{1}^{r}
\big \|g_{\rho}\big \|_{\rho}$. Hence, the second term is bounded {as follows,}
%
%e4.47 #&#
\begin{equation}
	E_{\mathbf{z}^{|D|}}\Big[\big \|f_{\lambda _{1}}-f_{\rho}\big \|_{
		\rho}\Big]\leq \lambda _{1}^{r}\big \|g_{\rho}\big \|_{\rho}.
	\label{eq4.29}
\end{equation}
We estimate the first term by taking expectations on both sides of {Proposition~\ref{rx}}. Since
\begin{equation*}
	\big \|f_{\lambda _{1}}\big \|_{K}=\big \|(\lambda _{1}I+L_{K})^{-1}L_{K}f_{
		\rho}\big \|_{K}\leq \big \|(\lambda _{1}I+L_{K})^{-1}L_{K}^{r+
		\frac{1}{2}}\big \|\big \|L_{K}^{1/2}g_{\rho}\big \|_{K}\leq \lambda _{1}^{r-
		\frac{1}{2}}\big \|g_{\rho}\big \|_{\rho},
\end{equation*}
by using Schwarz inequality and using {Lemma~\ref{belem}} with $s=2$, we
have
\begin{eqnarray*}
	\begin{aligned}
		E_{\mathbf{z}^{|D|}}\Big[\big \|f_{D,\lambda _{1},\lambda _{2}}-f_{
			\lambda _{1}}\big \|_{\rho}\Big]\leq &2E_{\mathbf{z}^{|D|}}\Big[
		\mathcal{A}_{D,\lambda _{1},\lambda _{2},V_{D}}\Pi _{D,\lambda _{1}}
		\Big]+4E_{\mathbf{z}^{|D|}}\Big[\Xi _{D,\lambda _{1}}\mathcal{A}_{D,
			\lambda _{1},\lambda _{2},V_{D}}\Big]\big \|f_{\lambda _{1}}\big \|_{K}+
		\frac{2\lambda _{2}c_{V}}{\sqrt{\lambda _{1}}}\big \|f_{\lambda _{1}}
		\big \|_{K}
		\\
		\leq &2\Big\{E_{\mathbf{z}^{|D|}}\Big[\mathcal{A}_{D,\lambda _{1},
			\lambda _{2},V_{D}}^{2}\Big]\Big\}^{\frac{1}{2}}\Big\{E_{\mathbf{z}^{|D|}}
		\Big[\Pi _{D,\lambda _{1}}^{2}\Big]\Big\}^{\frac{1}{2}}
		\\
		&+4\lambda _{1}^{r-\frac{1}{2}}\big \|g_{\rho}\big \|_{\rho}\Big\{E_{
			\mathbf{z}^{|D|}}\Big[\Xi _{D,\lambda _{1}}^{2}\Big]\Big\}^{
			\frac{1}{2}}\Big\{E_{\mathbf{z}^{|D|}}\Big[\mathcal{A}_{D,\lambda _{1},
			\lambda _{2},V_{D}}^{2}\Big]\Big\}^{\frac{1}{2}}+
		\frac{2\lambda _{2}c_{V}}{\sqrt{\lambda _{1}}}\lambda _{1}^{r-
			\frac{1}{2}}\big \|g_{\rho}\big \|_{\rho}
	\end{aligned}
\end{eqnarray*}
which can be further bounded by
\begin{eqnarray*}
	&& 2(2\Gamma (5)+\log ^{4}2)^{\frac{1}{2}}\bigg(
	\frac{2\mathcal{B}_{|D|,\lambda _{1}}}{\sqrt{\lambda _{1}}}+1\bigg)^{2}(2
	\Gamma (3)+\log ^{2}2)^{\frac{1}{2}}2M(\kappa +1)\mathcal{B}_{|D|,
		\lambda _{1}}'
	\\
	&&+4\lambda _{1}^{r-\frac{1}{2}}\big \|g_{\rho}\big \|_{\rho}(2
	\Gamma (3)+\log ^{2}2)^{\frac{1}{2}}2(\kappa ^{2}+\kappa )\mathcal{B}_{|D|,
		\lambda _{1}}'(2\Gamma (5)+\log ^{4}2)^{\frac{1}{2}}\bigg(
	\frac{2\mathcal{B}_{|D|,\lambda _{1}}}{\sqrt{\lambda _{1}}}+1\bigg)^{2}+2c_{V}
	\lambda _{2}\lambda _{1}^{r-1}\big \|g_{\rho}\big \|_{\rho}.
\end{eqnarray*}
The desired result is obtained by combining these estimates.
\end{proof}
For the case $f_{\rho}\notin \mathcal{H}_{K}$, we need the following new
estimate for
$\Big\{E_{\mathbf{z}^{|D|}}\Big[\big \|f_{D,\lambda _{1},\lambda _{2}}
\big \|_{K}^{2}\Big]\Big\}^{\frac{1}{2}}$. The approach on deriving an
upper bound of
$\Big\{E_{\mathbf{z}^{|D|}}\Big[\big \|f_{D,\lambda _{1},\lambda _{2}}
\big \|_{K}^{2}\Big]\Big\}^{\frac{1}{2}}$ relies on our new estimate in
{Proposition~\ref{rx}}.
%
%p7 #&#
\begin{pro}%
%%LEAP%%%\label{prop7}
\label{rfd}
Suppose $|y|\leq M$ almost surely. Let the regularity condition {\eqref{regularity}} hold with some index $r\in (0,1/2)$. If
$\lambda _{1},\lambda _{2}\in (0,1)$ satisfy
$2\lambda _{2}c_{V}=\lambda _{1}$, then we have
\begin{eqnarray}
	\nonumber
	&&\Big\{E_{\mathbf{z}^{|D|}}\Big[\big \|f_{D,\lambda _{1},\lambda _{2}}
	\big \|_{K}^{2}\Big]\Big\}^{\frac{1}{2}}\leq 2\sqrt{3}(2\Gamma (9)+
	\log ^{8}2)^{\frac{1}{4}}(2\Gamma (5)+\log ^{4}2)^{\frac{1}{4}}\bigg(
	\frac{2\mathcal{B}_{|D|,\lambda _{1}}}{\sqrt{\lambda _{1}}}+1\bigg)^{2}
	\times
	\\
	\nonumber
	&&\times \Big[2M(\kappa +1)\frac{1}{\sqrt{\lambda _{1}}}\mathcal{B}_{|D|,
		\lambda _{1}}'+2(\kappa ^{2}+\kappa )\big \|g_{\rho}\big \|_{\rho}
	\lambda _{1}^{r-1}\mathcal{B}_{|D|,\lambda _{1}}'\Big]+2\sqrt{3}
	\big \|g_{\rho}\big \|_{\rho}\lambda _{1}^{r-\frac{1}{2}}.
\end{eqnarray}
\end{pro}

\begin{proof}
We start with the decomposition
$\big \|f_{D,\lambda _{1},\lambda _{2}}\big \|_{K}\leq \big \|f_{D,
	\lambda _{1},\lambda _{2}}-f_{\lambda _{1}}\big \|_{K}+\big \|f_{
	\lambda _{1}}\big \|_{K}$. {Proposition~\ref{rx}} implies
%
%e4.48 #&#
\begin{equation}
	\nonumber
	\sqrt{\lambda _{1}}\big \|f_{D,\lambda _{1},\lambda _{2}}-f_{\lambda _{1}}
	\big \|_{K}\leq 2\mathcal{A}_{D,\lambda _{1},\lambda _{2},V_{D}}\Pi _{D,
		\lambda _{1}}+4\Xi _{D,\lambda _{1}}\mathcal{A}_{D,\lambda _{1},
		\lambda _{2},V_{D}}\big \|f_{\lambda _{1}}\big \|_{K}+
	\frac{2\lambda _{2}c_{V}}{\sqrt{\lambda _{1}}}\big \|f_{\lambda _{1}}
	\big \|_{K}.
\end{equation}
The above inequality together with the fact
$\big \|f_{\lambda _{1}}\big \|_{K}\leq \lambda _{1}^{r-\frac{1}{2}}
\big \|g_{\rho}\big \|_{\rho}$ imply that
%
%e4.49 #&#
\begin{equation}
	\nonumber
	\big \|f_{D,\lambda _{1},\lambda _{2}}-f_{\lambda _{1}}\big \|_{K}
	\leq 2\frac{1}{\sqrt{\lambda _{1}}}\mathcal{A}_{D,\lambda _{1},
		\lambda _{2},V_{D}}\Pi _{D,\lambda _{1}}+4\lambda _{1}^{-\frac{1}{2}}
	\lambda _{1}^{r-\frac{1}{2}}\big \|g_{\rho}\big \|_{\rho}\Xi _{D,
		\lambda _{1}}\mathcal{A}_{D,\lambda _{1},\lambda _{2},V_{D}}+2\big \|g_{
		\rho}\big \|_{\rho}c_{V}\lambda _{2}\lambda _{1}^{r-1-\frac{1}{2}}.
\end{equation}
Noting the condition $2\lambda _{2}c_{V}=\lambda _{1}$, we obtain
%
%e4.50 #&#
\begin{equation}
	\nonumber
	\big \|f_{D,\lambda _{1},\lambda _{2}}\big \|_{K}\leq
	\frac{2}{\sqrt{\lambda _{1}}}\mathcal{A}_{D,\lambda _{1},\lambda _{2},V_{D}}
	\Pi _{D,\lambda _{1}}+4\lambda _{1}^{r-1}\big \|g_{\rho}\big \|_{\rho}
	\Xi _{D,\lambda _{1}}\mathcal{A}_{D,\lambda _{1},\lambda _{2},V_{D}}+2
	\lambda _{1}^{r-\frac{1}{2}}\big \|g_{\rho}\big \|_{\rho}.
\end{equation}
The basic Jensen's inequality and Schwarz inequality implies that
\begin{eqnarray}
	\nonumber
	&&E_{\mathbf{z}^{|D|}}\Big[\big \|f_{D,\lambda _{1},\lambda _{2}}
	\big \|_{K}^{2}\Big]\leq \frac{12}{\lambda _{1}}\Big\{E_{\mathbf{z}^{|D|}}
	\Big[\mathcal{A}_{D,\lambda _{1},\lambda _{2},V_{D}}^{4}\Big]\Big\}^{
		\frac{1}{2}}\Big\{E_{\mathbf{z}^{|D|}}\Big[\Pi _{D,\lambda _{1}}^{4}
	\Big]\Big\}^{\frac{1}{2}}
	\\
	\nonumber
	&&+48\lambda _{1}^{2r-2}\big \|g_{\rho}\big \|_{\rho}^{2}\Big\{E_{
		\mathbf{z}^{|D|}}\Big[\Xi _{D,\lambda _{1}}^{4}\Big]\Big\}^{
		\frac{1}{2}}\Big\{E_{\mathbf{z}^{|D|}}\Big[\mathcal{A}_{D,\lambda _{1},
		\lambda _{2},V_{D}}^{4}\Big]\Big\}^{\frac{1}{2}}+12\lambda _{1}^{2r-1}
	\big \|g_{\rho}\big \|_{\rho}^{2}.
\end{eqnarray}
Using Jensen's inequality and {Lemma~\ref{belem}} with $s=4$, we have
\begin{eqnarray}
	\nonumber
	&&\Big\{E_{\mathbf{z}^{|D|}}\Big[\big \|f_{D,\lambda _{1},\lambda _{2}}
	\big \|_{K}^{2}\Big]\Big\}^{\frac{1}{2}}\leq
	\frac{2\sqrt{3}}{\sqrt{\lambda _{1}}}\Big\{E_{\mathbf{z}^{|D|}}\Big[
	\mathcal{A}_{D,\lambda _{1},\lambda _{2},V_{D}}^{4}\Big]\Big\}^{
		\frac{1}{4}}\Big\{E_{\mathbf{z}^{|D|}}\Big[\Pi _{D,\lambda _{1}}^{4}
	\Big]\Big\}^{\frac{1}{4}}
	\\
	\nonumber
	&&+2\sqrt{3}\lambda _{1}^{r-1}\big \|g_{\rho}\big \|_{\rho}\Big\{E_{
		\mathbf{z}^{|D|}}\Big[\Xi _{D,\lambda _{1}}^{4}\Big]\Big\}^{
		\frac{1}{4}}\Big\{E_{\mathbf{z}^{|D|}}\Big[\mathcal{A}_{D,\lambda _{1},
		\lambda _{2},V_{D}}^{4}\Big]\Big\}^{\frac{1}{4}}+2\sqrt{3}\lambda _{1}^{r-
		\frac{1}{2}}\big \|g_{\rho}\big \|_{\rho}
	\\
	\nonumber
	&&\leq 2\sqrt{3}(2\Gamma (9)+\log ^{8}2)^{\frac{1}{4}}\bigg(
	\frac{2\mathcal{B}_{|D|,\lambda _{1}}}{\sqrt{\lambda _{1}}}+1\bigg)^{2}(2
	\Gamma (5)+\log ^{4}2)^{\frac{1}{4}}2M(\kappa +1)\mathcal{B}_{|D|,
		\lambda _{1}}'\frac{1}{\sqrt{\lambda _{1}}}
	\\
	\nonumber
	&&4\sqrt{3}\big \|g_{\rho}\big \|_{\rho}(2\Gamma (9)+\log ^{8}2)^{
		\frac{1}{4}}\bigg(
	\frac{2\mathcal{B}_{|D|,\lambda _{1}}}{\sqrt{\lambda _{1}}}+1\bigg)^{2}(2
	\Gamma (5)+\log ^{4}2)^{\frac{1}{4}}2(\kappa ^{2}+\kappa )\mathcal{B}_{|D|,
		\lambda _{1}}'\lambda _{1}^{r-1}+2\sqrt{3}\big \|g_{\rho}\big \|_{
		\rho}\lambda _{1}^{r-\frac{1}{2}}.
\end{eqnarray}
After combining the first two terms, we finish the proof.
\end{proof}
Now we are ready to prove {Theorem~\ref{thm3}}.
\begin{proof}[Proof of {Theorem~\ref{thm3}}]
Under above preparations, we need to re-estimate the following expected
norm,
\begin{eqnarray}
	\nonumber
	&&E\Big[\Omega _{D,\lambda _{1},\lambda _{2},V_{D}}\big \|L_{K,D}-L_{K,
		\widehat{D}}\big \|\big \|f_{D,\lambda _{1},\lambda _{2}}\big \|_{K}
	\Big]
	\\
	\nonumber
	&&\leq \bigg(2\sqrt{2}(2+\sqrt{\pi})^{\frac{1}{2}}\kappa L
	\frac{2^{\frac{\alpha +2}{2}}B_{\widetilde{K}}^{\frac{\alpha }{2}}}{\lambda _{1}\widetilde{d}^{\frac{\alpha }{2}}}+2
	\bigg)\kappa L(2+\sqrt{\pi})^{\frac{1}{2}}
	\frac{2^{\frac{\alpha +2}{2}}B_{\widetilde{K}}^{\frac{\alpha }{2}}}{\lambda _{1}^{\frac{1}{2}}\widetilde{d}^{\frac{\alpha }{2}}}
	\Big\{E_{\mathbf{z}^{|D|}}\Big[\big \|f_{D,\lambda _{1},\lambda _{2}}
	\big \|_{K}^{2}\Big]\Big\}^{\frac{1}{2}}\Big\{E_{\mathbf{z}^{|D|}}
	\Big[\mathcal{A}_{D,\lambda _{1},\lambda _{2},V_{D}}^{2}\Big]\Big\}^{
		\frac{1}{2}}.
\end{eqnarray}
Substitute {Proposition~\ref{rfd}} into above inequality, we know that the above terms can be further bounded by
%
%e4.51 #&#
\begin{eqnarray}
	\nonumber
	&& \bigg(2\sqrt{2}(2+\sqrt{\pi})^{\frac{1}{2}}\kappa L
	\frac{2^{\frac{\alpha +2}{2}}B_{\widetilde{K}}^{\frac{\alpha }{2}}}{\lambda _{1}\widetilde{d}^{\frac{\alpha }{2}}}+2
	\bigg)\bigg(
	\frac{2\mathcal{B}_{|D|,\lambda _{1}}}{\sqrt{\lambda _{1}}}+1\bigg)^{2}(2
	\Gamma (5)+\log ^{4}2)^{\frac{1}{2}}\kappa L(2+\sqrt{\pi})^{
		\frac{1}{2}}2^{\frac{\alpha +2}{2}}B_{\widetilde{K}}^{
		\frac{\alpha }{2}}\times
	\\
	\nonumber
	&&\Bigg(4\sqrt{3}(2\Gamma (9)+\log ^{8}2)^{\frac{1}{4}}(2\Gamma (5)+
	\log ^{4}2)^{\frac{1}{4}}\Big(
	\frac{2\mathcal{B}_{|D|,\lambda _{1}}}{\sqrt{\lambda _{1}}}+1\Big)^{2}
	\Big[2M(\kappa +1)
	\frac{1}{\lambda _{1}^{\frac{1}{2}}\widetilde{d}^{\frac{\alpha }{2}}}
	\frac{1}{\sqrt{\lambda _{1}}}\mathcal{B}_{|D|,\lambda _{1}}'
	\\
	&&+2(\kappa ^{2}+\kappa )\big \|g_{\rho}\big \|_{\rho}
	\frac{\lambda _{1}^{r-1}}{\lambda _{1}^{\frac{1}{2}}\widetilde{d}^{\frac{\alpha }{2}}}
	\mathcal{B}_{|D|,\lambda _{1}}'\Big]+2\sqrt{3}\big \|g_{\rho}\big \|_{
		\rho}
	\frac{\lambda _{1}^{r-\frac{1}{2}}}{\lambda _{1}^{\frac{1}{2}}\widetilde{d}^{\frac{\alpha }{2}}}
	\Bigg).
	\label{nbd}
	%%LEAP%%%\label{eq4.30}
\end{eqnarray}
On the other hand, note that the estimate of {\eqref{bdd1}} on
$E\bigg[\Omega _{D,\lambda _{1},\lambda _{2},V_{D}}\big \|\hat{S}_{D}^{T}y_{D}-S_{D}^{T}y_{D}
\big \|\bigg]$ in case $r\in (0,1/2)$ does not change. Now combine {\eqref{bdd1}}, {\eqref{nbd}} and {\eqref{bdd3}} in {Proposition~\ref{r2b}}, the
desired bound follows.
\end{proof}
%
%r2 #&#
\begin{rmk}%
%%LEAP%%%\label{rem2}
\label{secondstage_discuss}
In contrast to previous works, for handling the
tough case $r\in (0,1/2)$, since we have employed the different decomposition
approach described above, the core terms of the second-stage estimates have
already changed to some new terms
\begin{equation*}\frac{1}{\lambda
		_{1}^{\frac{1}{2}}\widetilde{d}^{\frac{\alpha }{2}}}\frac{1}{\sqrt{\lambda
			_{1}}}\mathcal{B}_{|D|,\lambda _{1}}', \frac{\lambda _{1}^{r-1}}{\lambda
		_{1}^{\frac{1}{2}}\widetilde{d}^{\frac{\alpha }{2}}}\mathcal{B}_{|D|,\lambda
		_{1}}', \frac{\lambda _{1}^{r-\frac{1}{2}}}{\lambda
		_{1}^{\frac{1}{2}}\widetilde{d}^{\frac{\alpha }{2}}}.\end{equation*} In fact,
when $r\in (0,1/2)$, in the above estimate of $\Omega _{D,\lambda
	_{1},\lambda _{2},V_{D}}\big \|L_{K,D}-L_{K,\widehat{D}}\big \|\big
\|f_{D,\lambda _{1},\lambda _{2}}\big \|_{K}$, the derivation of the upper
bound of $\big \|f_{D,\lambda _{1},\lambda _{2}}\big \|_{K}$ deeply relies on
our new estimates in {Proposition~\ref{rx}} and {Proposition~\ref{rfd}}.
\end{rmk}
\begin{proof}[Proof of {Theorem~\ref{thm4}}]
We substitute the main parameters
\begin{equation*}
	\lambda _{1}=|D|^{-\frac{1}{1+\beta }},\  \lambda _{2}=
	\frac{1}{2c_{V}}|D|^{-\frac{1}{1+\beta }}, \ \widetilde{d}=d_{1}=d_{2}=
	\cdots =d_{|D|}=|D|^{\frac{2}{\alpha (1+\beta )}}
\end{equation*}
into the general error bound estimate in {Theorem~\ref{thm3}}. We obtain
%
%e4.52 #&#
\begin{equation}
	\nonumber
	\mathcal{B}_{|D|,\lambda _{1}}=\frac{2\kappa }{\sqrt{|D|}}\bigg(
	\frac{\kappa }{\sqrt{|D|\lambda _{1}}}+\sqrt{\mathcal{N}(\lambda _{1})}
	\bigg)\leq 2\kappa (\kappa +\sqrt{\mathcal{C}_{0}})|D|^{-
		\frac{1}{2(1+\beta )}}
\end{equation}
and
%
%e4.53 #&#
\begin{equation}
	\nonumber
	\frac{2\mathcal{B}_{|D|,\lambda _{1}}}{\sqrt{\lambda _{1}}}\leq 4
	\kappa (\kappa +\sqrt{\mathcal{C}_{0}}).
\end{equation}
Now we estimate the key terms
\begin{equation*}
	\frac{1}{\lambda _{1}\widetilde{d}^{\frac{\alpha }{2}}},
	\frac{1}{\lambda _{1}^{\frac{1}{2}}\widetilde{d}^{\frac{\alpha }{2}}}
	\frac{1}{\sqrt{\lambda _{1}}}\mathcal{B}_{|D|,\lambda _{1}}',
	\frac{\lambda _{1}^{r-1}}{\lambda _{1}^{\frac{1}{2}}\widetilde{d}^{\frac{\alpha }{2}}}
	\mathcal{B}_{|D|,\lambda _{1}}',
	\frac{\lambda _{1}^{r-\frac{1}{2}}}{\lambda _{1}^{\frac{1}{2}}\widetilde{d}^{\frac{\alpha }{2}}}.
\end{equation*}
We start with the following three estimates,
%
%e4.54 #&#
\begin{equation}
	\nonumber
	\frac{1}{\lambda _{1}\widetilde{d}^{\frac{\alpha }{2}}}=|D|^{
		\frac{1}{1+\beta }}\cdot |D|^{-\frac{\alpha }{2}
		\frac{2}{\alpha (1+\beta )}}=|D|^{\frac{1}{1+\beta }-
		\frac{1}{1+\beta }}=1.
\end{equation}
%
%e4.55 #&#
\begin{equation}
	\nonumber
	\frac{1}{\lambda _{1}^{\frac{1}{2}}\widetilde{d}^{\frac{\alpha }{2}}}=|D|^{
		\frac{1}{2(1+\beta )}}\cdot |D|^{-\frac{\alpha }{2}
		\frac{2}{\alpha (1+\beta )}}=|D|^{-\frac{1}{2(1+\beta )}},
\end{equation}
%
%e4.56 #&#
\begin{equation}
	\nonumber
	\mathcal{B}_{|D|,\lambda _{1}}'=\frac{1}{|D|\sqrt{\lambda _{1}}}+
	\frac{\sqrt{\mathcal{N}(\lambda _{1})}}{\sqrt{|D|}}\leq (1+\sqrt{
		\mathcal{C}_{0}})|D|^{-\frac{1}{2(1+\beta )}}.
\end{equation}
The above estimate for $\mathcal{B}_{|D|,\lambda _{1}}'$ and the fact
$r\in (0,1/2)$ in fact also imply
%
%e4.57 #&#
\begin{equation}
	\nonumber
	\mathcal{B}_{|D|,\lambda _{1}}'\leq (1+\sqrt{\mathcal{C}_{0}})|D|^{-
		\frac{1}{2(1+\beta )}}<(1+\sqrt{\mathcal{C}_{0}})|D|^{-
		\frac{r}{1+\beta }}.
\end{equation}
Therefore, we obtain
\begin{eqnarray}
	\nonumber
	&&
	\frac{1}{\lambda _{1}^{\frac{1}{2}}\widetilde{d}^{\frac{\alpha }{2}}}
	\frac{1}{\sqrt{\lambda _{1}}}\mathcal{B}_{|D|,\lambda _{1}}'\leq |D|^{-
		\frac{1}{2(1+\beta )}}|D|^{\frac{1}{2(1+\beta )}}(1+\sqrt{\mathcal{C}_{0}})|D|^{-
		\frac{1}{2(1+\beta )}}
	\\
	\nonumber
	&&\ \ \ \ \ \ \ \ \ \ \ \ \ \ \ \ \ \ \ \ \ \ \ \ \  =(1+\sqrt{
		\mathcal{C}_{0}})|D|^{-\frac{1}{2(1+\beta )}}<(1+\sqrt{\mathcal{C}_{0}})|D|^{-
		\frac{r}{1+\beta }},
\end{eqnarray}
\begin{eqnarray}
	\nonumber
	&&
	\frac{\lambda _{1}^{r-1}}{\lambda _{1}^{\frac{1}{2}}\widetilde{d}^{\frac{\alpha }{2}}}
	\mathcal{B}_{|D|,\lambda _{1}}'\leq |D|^{-\frac{r-1}{1+\beta }}|D|^{-
		\frac{1}{2(1+\beta )}}(1+\sqrt{\mathcal{C}_{0}})|D|^{-
		\frac{1}{2(1+\beta )}}=(1+\sqrt{\mathcal{C}_{0}})|D|^{-
		\frac{r}{1+\beta }},
\end{eqnarray}
and
%
%e4.58 #&#
\begin{equation}
	\nonumber
	\frac{\lambda _{1}^{r-\frac{1}{2}}}{\lambda _{1}^{\frac{1}{2}}\widetilde{d}^{\frac{\alpha }{2}}}=|D|^{-
		\frac{r-\frac{1}{2}}{1+\beta }}|D|^{-\frac{1}{2(1+\beta )}}=|D|^{-
		\frac{r}{1+\beta }}.
\end{equation}
We turn to estimate
$\lambda _{1}^{r-\frac{1}{2}}\mathcal{B}_{|D|,\lambda _{1}}'$,
$c_{V}\lambda _{2}\lambda _{1}^{r-1}$ and $\lambda _{1}^{r}$. Above basic
estimate for $\mathcal{B}_{|D|,\lambda _{1}}'$ implies
%
%e4.59 #&#
\begin{equation}
	\nonumber
	\lambda _{1}^{r-\frac{1}{2}}\mathcal{B}_{|D|,\lambda _{1}}'\leq |D|^{-
		\frac{r-\frac{1}{2}}{1+\beta }}\cdot (1+\sqrt{\mathcal{C}_{0}})|D|^{-
		\frac{1}{2(1+\beta )}}\leq (1+\sqrt{\mathcal{C}_{0}})|D|^{-
		\frac{r}{1+\beta }}.
\end{equation}
Also, note that the above parameters design for $\lambda _{1}$ and
$\lambda _{2}$ imply
%
%e4.60 #&#
\begin{equation}
	\nonumber
	c_{V}\lambda _{2}\lambda _{1}^{r-1}=\frac{1}{2}|D|^{-
		\frac{r}{1+\beta }}, \ \lambda _{1}^{r}=|D|^{-\frac{r}{1+\beta }}.
\end{equation}
After combining all the above estimates with {Theorem~\ref{thm3}}, the desired
rates are obtained.
\end{proof}
%

%s5 #&#
\section{Distributed learning with multi-penalty distribution regression and its learning rates}
\label{sec5}

This section is devoted to proving the result on distributed learning with
multi-penalty distribution regression. The proof is based on an observation
of the relation between our former estimates in this paper and the one-stage
distributed learning theory. Firstly, we make some preparations. Recall
the definition of $f_{\lambda _{1}}$, we have
\begin{equation*}
\lambda _{1}f_{\lambda _{1}}=L_{K}(f_{\rho}-f_{\lambda _{1}}).
\end{equation*}
Then for any mean embedding set $D$ associated with the given data set
$\widetilde{D}$, we have
\begin{eqnarray}
\nonumber
&&f_{D,\lambda _{1},\lambda _{2}}-f_{\lambda _{1}}
\\
\nonumber
&&=(L_{K,D}+\lambda _{1}I+\lambda _{2}V_{D}^{T}V_{D})^{-1}S_{D}^Ty_{D}-f_{
	\lambda _{1}}
\\
\nonumber
&&=(L_{K,D}+\lambda _{1}I+\lambda _{2}V_{D}^{T}V_{D})^{-1}\big[S_{D}^Ty_{D}-L_{K,D}f_{
	\lambda _{1}}-\lambda _{1}f_{\lambda _{1}}-\lambda _{2}V_{D}^{T}V_{D}f_{
	\lambda _{1}}\big]
\\
\nonumber
&&=(L_{K,D}+\lambda _{1}I+\lambda _{2}V_{D}^{T}V_{D})^{-1}\Big\{
\frac{1}{|D|}\sum _{z\in D}(y-f_{\lambda _{1}}(\mu _{x}))K_{\mu _{x}}-L_{K}(f_{
	\rho}-f_{\lambda _{1}})-\lambda _{2}V_{D}^{T}V_{D}f_{\lambda _{1}}
\Big\}
\end{eqnarray}
in which $z=(\mu _{x},y)\in X_{\mu}\times Y$. Denote
\begin{equation*}
\Delta _{D}=\frac{1}{|D|}\sum _{z\in D}(y-f_{\lambda _{1}}(\mu _{x}))K_{
	\mu _{x}}-L_{K}(f_{\rho}-f_{\lambda _{1}})
\end{equation*}
and
\begin{equation*}
\mathcal{Q}_{D}=(L_{K,D}+\lambda _{1}I+\lambda _{2}V_{D}^{T}V_{D})^{-1}-(L_{K}+
\lambda _{1}I+\lambda _{2}V_{D}^{T}V_{D})^{-1}.
\end{equation*}
Then we can decompose
$f_{\widehat{D},\lambda _{1},\lambda _{2}}-f_{\lambda _{1}}$ as
%
%e5.1 #&#
\begin{eqnarray}
\nonumber
&&f_{\widehat{D},\lambda _{1},\lambda _{2}}-f_{\lambda _{1}}=\Big[f_{
	\widehat{D},\lambda _{1},\lambda _{2}}-f_{D,\lambda _{1},\lambda _{2}}
\Big]+\mathcal{Q}_{D}\Delta _{D}+(L_{K}+\lambda _{1}I+\lambda _{2}V_{D}^{T}V_{D})^{-1}
\Delta _{D}
\\
&& \quad \quad \quad \quad \quad \quad \quad -(L_{K,D}+\lambda _{1}I+
\lambda _{2}V_{D}^{T}V_{D})^{-1}\lambda _{2}V_{D}^{T}V_{D}f_{\lambda _{1}}.
\label{rep}
%%LEAP%%%\label{eq5.1}
\end{eqnarray}
In the following, we will use the corresponding notations $\mathcal{Q}_{D_{j}}$,
$L_{K,D_{j}}$, $\Delta _{D_{j}}$, $V_{D_{j}}$ involving the sample subset
$\widetilde{D}_{j}$ and its associated mean embedding set $D_{j}$. Then
the corresponding representations for
$f_{\widehat{D}_{j},\lambda _{1},\lambda _{2}}-f_{\lambda _{1}}$ are well-defined
for local data sets. Applying second-order decomposition {\eqref{sd1}} to $\mathcal{Q}_{D}$ with
\begin{equation}
\nonumber
A=L_{K,D}+\lambda _{1}I+\lambda _{2}V_{D}^{T}V_{D},
\end{equation}
\begin{equation}
\nonumber
B=L_{K}+\lambda _{1}I+\lambda _{2}V_{D}^{T}V_{D},
\end{equation}
we have
\begin{eqnarray*}
\nonumber
&&L_{K}^{1/2}\mathcal{Q}_{D}\Delta _{D}=L_{K}^{1/2}(L_{K}+\lambda _{1}I+
\lambda _{2}V_{D}^{T}V_{D})^{-1}(L_{K}-L_{K,D})(L_{K}+\lambda _{1}I+
\lambda _{2}V_{D}^{T}V_{D})^{-1}\Delta _{D}
\\
\nonumber
&&\ \ \ \ \ \ \ \ \ \ \ \ \ \ \ \ \  \ +L_{K}^{1/2}(L_{K}+\lambda _{1}I+
\lambda _{2}V_{D}^{T}V_{D})^{-1}(L_{K}-L_{K,D})(L_{K,D}+\lambda _{1}I+
\lambda _{2}V_{D}^{T}V_{D})^{-1}
\\
\nonumber
&&\ \ \ \ \ \ \ \ \ \ \ \  \ \ \ \ \ \  \ (L_{K}-L_{K,D})(L_{K}+
\lambda _{1}I+\lambda _{2}V_{D}^{T}V_{D})^{-1}\Delta _{D}
\end{eqnarray*}
which can be further decomposed as
\begin{eqnarray*}
&&\{L_{K}^{1/2}(L_{K}+\lambda _{1}I)^{-1/2}\}\{(L_{K}+\lambda _{1}I)^{1/2}(L_{K}+
\lambda _{1}I+\lambda _{2}V_{D}^{T}V_{D})^{-1/2}\}
\\
&& \{(L_{K}+\lambda _{1}I+\lambda _{2}V_{D}^{T}V_{D})^{-1/2}(L_{K}+
\lambda _{1}I)^{1/2}\}\{(L_{K}+\lambda _{1}I)^{-1/2}(L_{K}-L_{K,D})\}
\\
&& \{(L_{K}+\lambda _{1}I+\lambda _{2}V_{D}^{T}V_{D})^{-1/2}\}\{(L_{K}+
\lambda _{1}I+\lambda _{2}V_{D}^{T}V_{D})^{-1/2}(L_{K}+\lambda _{1}I)^{1/2}
\}
\\
&& \{(L_{K}+\lambda _{1}I)^{-1/2}\Delta _{D}\}
\\
&& +\{L_{K}^{1/2}(L_{K}+\lambda _{1}I)^{-1/2}\}\{(L_{K}+\lambda _{1}I)^{1/2}(L_{K}+
\lambda _{1}I+\lambda _{2}V_{D}^{T}V_{D})^{-1/2}\}
\\
&& \{(L_{K}+\lambda _{1}I+\lambda _{2}V_{D}^{T}V_{D})^{-1/2}(L_{K}+
\lambda _{1}I)^{1/2}\}\{(L_{K}+\lambda _{1}I)^{-1/2}(L_{K}-L_{K,D})\}
\\
&& \{(L_{K,D}+\lambda _{1}I+\lambda _{2}V_{D}^{T}V_{D})^{-1}\}\{(L_{K}-L_{K,D})(L_{K}+
\lambda _{1}I)^{-1/2}\}
\\
&& \{(L_{K}+\lambda _{1}I)^{1/2}(L_{K}+\lambda _{1}I+\lambda _{2}V_{D}^{T}V_{D})^{-1/2}
\}\{(L_{K}+\lambda _{1}I+\lambda _{2}V_{D}^{T}V_{D})^{-1/2}(L_{K}+
\lambda _{1}I)^{1/2}\}
\\
&& \{(L_{K}+\lambda _{1}I)^{-1/2}\Delta _{D}\}.
\end{eqnarray*}
After taking RKHS-norms on both sides, we obtain
%
%e5.2 #&#
\begin{eqnarray}
\nonumber
&&\Big\|L_{K}^{1/2}\mathcal{Q}_{D}\Delta _{D}\Big\|_{K}\leq 2^{
	\frac{3}{2}}\frac{\Xi _{D,\lambda _{1}}}{\sqrt{\lambda _{1}}}\Big\|(L_{K}+
\lambda _{1}I)^{-1/2}\Delta _{D}\Big\|_{K}+4
\frac{\Xi _{D,\lambda _{1}}^{2}}{\lambda _{1}}\Big\|(L_{K}+\lambda _{1}I)^{-1/2}
\Delta _{D}\Big\|_{K}
\\
&&\ \ \ \ \ \ \ \ \ \ \ \ \ \ \ \ \ \ \ \ \ \  =\Big(2^{\frac{3}{2}}
\frac{\Xi _{D,\lambda _{1}}}{\sqrt{\lambda _{1}}}+4
\frac{\Xi _{D,\lambda _{1}}^{2}}{\lambda _{1}}\Big)\Big\|(L_{K}+
\lambda _{1}I)^{-1/2}\Delta _{D}\Big\|_{K}
\label{d1}
%%LEAP%%%\label{eq5.2}
\end{eqnarray}
where $\Xi _{D,\lambda _{1}}$ is defined in {\eqref{fldef}}. Considering the
above representation {\eqref{rep}} on mean embedding set $D_{j}$, we know
that
\begin{eqnarray}
\nonumber
&&f_{\widehat{D}_{j},\lambda _{1},\lambda _{2}}-f_{\lambda _{1}}=
\Big[f_{\widehat{D}_{j},\lambda _{1},\lambda _{2}}-f_{D_{j},\lambda _{1},
	\lambda _{2}}\Big]+\mathcal{Q}_{D_{j}}\Delta _{D_{j}}+(L_{K}+\lambda _{1}I+
\lambda _{2}V_{D_{j}}^{T}V_{D_{j}})^{-1}\Delta _{D_{j}}
\\
\nonumber
&&\quad \quad \quad \quad \quad \quad \quad -(L_{K,D_{j}}+\lambda _{1}I+
\lambda _{2}V_{D_{j}}^{T}V_{D_{j}})^{-1}\lambda _{2}V_{D_{j}}^{T}V_{D_{j}}f_{
	\lambda _{1}}.
\end{eqnarray}
Then we have
%
%e5.3 #&#
\begin{eqnarray}
\nonumber
&&\overline{f_{\widehat{D},\lambda _{1},\lambda _{2}}}-f_{\lambda _{1}}=
\sum _{j=1}^{m}\frac{|D_{j}|}{|D|}\Big[f_{\widehat{D}_{j},\lambda _{1},
	\lambda _{2}}-f_{\lambda _{1}}\Big]
\\
\nonumber
&&=\sum _{j=1}^{m}\frac{|D_{j}|}{|D|}\Big[f_{\widehat{D}_{j},\lambda _{1},
	\lambda _{2}}-f_{{D}_{j},\lambda _{1},\lambda _{2}}\Big]+\sum _{j=1}^{m}
\frac{|D_{j}|}{|D|}(L_{K,D_{j}}+\lambda _{1}I+\lambda _{2}V_{D_{j}}^{T}V_{D_{j}})^{-1}
\Delta _{D_{j}}
\\
&&\ \ -\sum _{j=1}^{m}\frac{|D_{j}|}{|D|}(L_{K,D_{j}}+\lambda _{1}I+
\lambda _{2}V_{D_{j}}^{T}V_{D_{j}})^{-1}\lambda _{2}V_{D_{j}}^{T}V_{D_{j}}f_{
	\lambda _{1}}.
\label{rep1}
%%LEAP%%%\label{eq5.3}
\end{eqnarray}
Since $\Delta _{D}=\sum _{j=1}\frac{|D_{j}|}{|D|}\Delta _{D_{j}}$, substraction
between {\eqref{rep1}} and {\eqref{rep}} yields the crucial error decomposition
for two-stage multi-penalty distribution regression distributed learning
scheme {\eqref{disal}} in the following proposition.
%
%p8 #&#
\begin{pro}
%%LEAP%%%\label{prop8}
\label{two-stage_distributed_decomposition}
For the distributed estimator
$\overline{f_{\widehat{D},\lambda _{1},\lambda _{2}}}$ and the estimator
$f_{\widehat{D},\lambda _{1},\lambda _{2}}$ associated with a single data
set $D$, there holds the following two-stage error decomposition:
\begin{eqnarray}
	\nonumber
	&&\overline{f_{\widehat{D},\lambda _{1},\lambda _{2}}}- f_{
		\widehat{D},\lambda _{1},\lambda _{2}}=\sum _{j=1}^{m}
	\frac{|D_{j}|}{|D|}\Big[[f_{\widehat{D}_{j},\lambda _{1},\lambda _{2}}-f_{{D}_{j},
		\lambda _{1},\lambda _{2}}]-[f_{\widehat{D},\lambda _{1},\lambda _{2}}-f_{{D},
		\lambda _{1},\lambda _{2}}]\Big]
	\\
	\nonumber
	&&\ \ \ \ \ \ \ \ \ \ \ \  \ +\sum _{j=1}^{m}\frac{|D_{j}|}{|D|}\big[(L_{K,D_{j}}+
	\lambda _{1}I+\lambda _{2}V_{D_{j}}^{T}V_{D_{j}})^{-1}-(L_{K,D}+
	\lambda _{1}I+\lambda _{2}V_{D}^{T}V_{D})^{-1}\big]\Delta _{D_{j}}
	\\
	\nonumber
	&&\ \ \ \ \ \ \ \ \ \ \ \  \ +\sum _{j=1}^{m}\frac{|D_{j}|}{|D|}\big[(L_{K,D}+\lambda _{1}I+\lambda _{2}V_{D}^{T}V_{D})^{-1}
	\lambda _{2}V_{D}^{T}V_{D}f_{\lambda _{1}}-(L_{K,D_{j}}+
	\lambda _{1}I+\lambda _{2}V_{D_{j}}^{T}V_{D_{j}})^{-1}\lambda _{2}V_{D_{j}}^{T}V_{D_{j}}f_{
		\lambda _{1}}\big].
\end{eqnarray}
\end{pro}
In the following, we always denote the three terms on the right-hand side
of the above error decomposition by $\mathcal{S}_{1}$,
$\mathcal{S}_{2}$, $\mathcal{S}_{3}$.
%
%r3 #&#
\begin{rmk}
\label{rem3}
It can be observed that $\mathcal{S}_{1}$ presents the second-stage sampling effect to the multi-penalty distribution regression-based distributed learning algorithm. $\mathcal{S}_{2}$ is mainly related to one-stage distributed decomposition. $\mathcal{S}_{3}$ shows some additional multi-penalty influence on the proposed distributed learning algorithm. In contrast to former works, the decomposition is able to capture a more general setting when the data are functional or distribution data.
\end{rmk}
For later use, we spit
$\Delta _{D_{j}}=\Delta _{D_{j}}'+\Delta _{D_{j}}''$ in which
%
%e5.4 #&#
\begin{equation}
\Delta _{D_{j}}'=S_{D_{j}}^Ty_{D_{j}}-L_{K,D_{j}}f_{\rho}, \ \Delta _{D_{j}}''=L_{K,D_{j}}(f_{
	\rho}-f_{\lambda _{1}})-L_{K}(f_{\rho}-f_{\lambda _{1}}).
\label{split}
%%LEAP%%%\label{eq5.4}
\end{equation}
We are ready to prove {Theorem~\ref{disrate}} on the optimal learning rates
for distributed learning with multi-penalty distribution regression. In
subsequent proof, we use the notation $f_{1}\lesssim f_{2}$ to denote that
there is an absolute constant $C$ independent of $|D|$ and $m$ such that
$f_{1}\leq C f_{2}$.
\begin{proof}[Proof of {Theorem~\ref{disrate}}]
We aim to bound
$E[\|\mathcal{S}_{1}+\mathcal{S}_{2}+\mathcal{S}_{3}\|_{\rho}]$, in which
$\mathcal{S}_{1}$, $\mathcal{S}_{2}$, $\mathcal{S}_{3}$ are defined as
in {Proposition~\ref{two-stage_distributed_decomposition}}. The inequality
in {\eqref{yin}} implies that,
\begin{eqnarray}
	\nonumber
	&&E\Big[\big \|f_{\widehat{D},\lambda _{1},\lambda _{2}}-f_{D,
		\lambda _{1},\lambda _{2}}\big \|_{\rho}\Big]\lesssim \Big(
	\frac{1}{\lambda _{1}d^{\frac{\alpha }{2}}}+1\Big)\bigg(
	\frac{2\mathcal{B}_{|D|,\lambda _{1}}}{\sqrt{\lambda _{1}}}+1\bigg)^{2}
	\frac{1}{\lambda _{1}^{\frac{1}{2}}d^{\frac{\alpha }{2}}}\bigg[2+
	\frac{\mathcal{B}_{|D|,\lambda _{1}}}{\sqrt{\lambda _{1}}}\bigg(
	\frac{2\mathcal{B}_{|D|,\lambda _{1}}}{\sqrt{\lambda _{1}}}+1\bigg)
	\\
	\nonumber
	&&\ \ \ \ \ \ \ \ \ \ \ \  \ \ \ \ \ \ \ \ \ \ \ \ \ \ \ \ \ \  \ \
	\  \ \ \ \ \ \ \ +(\lambda _{1}+\lambda _{2}c_{V})\lambda _{1}^{r-
		\frac{3}{2}}\bigg(
	\frac{2\mathcal{B}_{|D|,\lambda _{1}}}{\sqrt{\lambda _{1}}}+1\bigg)^{2r-1}
	\bigg],
\end{eqnarray}
and same procedure implies that the above inequality continues to hold
when $D$ is replaced by $D_{j}$. Substituting
$\lambda _{1}=|D|^{-\frac{1}{2r+\beta }}$ into {\eqref{cap}} yields
$\mathcal{N}(\lambda _{1})\leq \mathcal{C}_{0}|D|^{
	\frac{\beta }{2r+\beta }}$. Using the condition
$ m\leq |D|^{\frac{2r-1}{2r+\beta }}$ and $|D_{j}|=|D|/m$,
$j=1,2,...,m$, we have
%
%e5.5 #&#
\begin{equation}
	\frac{\mathcal{N}(\lambda _{1})}{\lambda _{1}|D_{j}|}\leq \mathcal{C}_{0}
	m |D|^{\frac{1-2r}{2r+\beta }}\leq \mathcal{C}_{0}.
	\label{xl}
	%%LEAP%%%\label{eq5.5}
\end{equation}
Then it follows that, for $j=1,2,...,m$,
%
%e5.6 #&#
\begin{equation}
	\nonumber
	\frac{\mathcal{B}_{|D_{j}|,\lambda _{1}}}{\sqrt{\lambda _{1}}}=
	\frac{2\kappa }{\sqrt{\lambda _{1}|D_{j}|}}\bigg(
	\frac{\kappa }{\sqrt{|D_{j}|\lambda _{1}}}+\sqrt{\mathcal{N}(\lambda _{1})}
	\bigg)\leq 2\kappa (\kappa +\sqrt{\mathcal{C}_{0}}).
\end{equation}
Note that by taking $\widetilde{d}=d$, the former estimates {\eqref{e1}} and {\eqref{e2}} have already implied
%
%e5.7 #&#
\begin{equation}
	\nonumber
	\frac{1}{\lambda _{1}d^{\frac{\alpha }{2}}}\leq 1, \
	\frac{1}{\lambda _{1}^{\frac{1}{2}}d^{\frac{\alpha }{2}}}\leq |D|^{-
		\frac{r}{2r+\beta }}, \ (\lambda _{1}+\lambda _{2}c_{V})\lambda _{1}^{r-
		\frac{3}{2}}\leq \frac{3}{2}.
\end{equation}

Now we can estimate $\mathcal{S}_{1}$ in expectation form as
\begin{eqnarray}
	\nonumber
	&&E\Big[\big\|\mathcal{S}_{1}\big\|_{\rho}\Big]\lesssim \sum _{j=1}^{m}
	\frac{|D_{j}|}{|D|}\Big(\frac{1}{\lambda _{1}d^{\frac{\alpha }{2}}}+1
	\Big)\bigg(
	\frac{2\mathcal{B}_{|D_{j}|,\lambda _{1}}}{\sqrt{\lambda _{1}}}+1
	\bigg)^{2}\frac{1}{\lambda _{1}^{\frac{1}{2}}d^{\frac{\alpha }{2}}}
	\bigg[2+
	\frac{\mathcal{B}_{|D_{j}|,\lambda _{1}}}{\sqrt{\lambda _{1}}}\bigg(
	\frac{2\mathcal{B}_{|D_{j}|,\lambda _{1}}}{\sqrt{\lambda _{1}}}+1
	\bigg)
	\\
	\nonumber
	&&\ \ \ \ \ \ \ \ \ \ \ \  \ \ \ \ \ \ \ \ \ \ \ \  +(\lambda _{1}+
	\lambda _{2}c_{V})\lambda _{1}^{r-\frac{3}{2}}\bigg(
	\frac{2\mathcal{B}_{|D_{j}|,\lambda _{1}}}{\sqrt{\lambda _{1}}}+1
	\bigg)^{2r-1}\bigg]
	\\
	\nonumber
	&&\ \ \ \ \ \ \ \ \ \ \ \ \ \ \ \ \ \ +\Big(
	\frac{1}{\lambda _{1}d^{\frac{\alpha }{2}}}+1\Big)\bigg(
	\frac{2\mathcal{B}_{|D|,\lambda _{1}}}{\sqrt{\lambda _{1}}}+1\bigg)^{2}
	\frac{1}{\lambda _{1}^{\frac{1}{2}}d^{\frac{\alpha }{2}}}\bigg[2+
	\frac{\mathcal{B}_{|D|,\lambda _{1}}}{\sqrt{\lambda _{1}}}\bigg(
	\frac{2\mathcal{B}_{|D|,\lambda _{1}}}{\sqrt{\lambda _{1}}}+1\bigg)
	\\
	\nonumber
	&&\ \ \ \ \ \ \ \ \ \ \ \  \ \ \ \ \ \ \ \ \ \ \ \ +(\lambda _{1}+
	\lambda _{2}c_{V})\lambda _{1}^{r-\frac{3}{2}}\bigg(
	\frac{2\mathcal{B}_{|D|,\lambda _{1}}}{\sqrt{\lambda _{1}}}+1\bigg)^{2r-1}
	\bigg]\lesssim |D|^{-\frac{r}{2r+\beta }}.
\end{eqnarray}

We turn to estimate $\mathcal{S}_{2}$. By subtracting and adding the operator
$(L_{K}+\lambda _{1}I+\lambda _{2}V_{D_{j}}^{T}V_{D_{j}})^{-1}$, noting
that $\Delta _{D_{j}}=\Delta _{D_{j}}'+\Delta _{D_{j}}''$, where
$\Delta _{D_{j}}'$, $\Delta _{D_{j}}''$ are given as in {\eqref{split}},
we have the decomposition
%
%e5.8 #&#
\begin{equation}
	\mathcal{S}_{2}=\sum _{j=1}^{m}\frac{|D_{j}|}{|D|}\mathcal{Q}_{D_{j}}
	\Delta _{D_{j}}'+\sum _{j=1}^{m}\frac{|D_{j}|}{|D|}\mathcal{Q}_{D_{j}}
	\Delta _{D_{j}}''-\mathcal{Q}_{D}\Delta _{D}.
	\label{disdecoms2}
	%%LEAP%%%\label{eq5.6}
\end{equation}
Using the procedure of deriving {\eqref{d1}} with $\Delta _{D}$ replaced
by $\Delta _{D_{j}}'$ and $\Delta _{D_{j}}''$, and taking RKHS-norms on both
sides, we have
\begin{eqnarray}
	\nonumber
	&&\Big\|L_{K}^{1/2}\mathcal{Q}_{D_{j}}\Delta _{D_{j}}'\Big\|_{K}^{2}
	\leq \Big(2^{\frac{3}{2}}
	\frac{\Xi _{D_{j},\lambda _{1}}}{\sqrt{\lambda _{1}}}+4
	\frac{\Xi _{D_{j},\lambda _{1}}^{2}}{\lambda _{1}}\Big)^{2}\Big\|(L_{K}+
	\lambda _{1}I)^{-1/2}\Delta _{D_{j}}'\Big\|_{K}^{2}
	\\
	\nonumber
	&&\ \ \ \ \ \  \ \ \ \ \ \  \ \ \ \ \ \  \ \ \ \ \ \ \lesssim \Big(
	\frac{\Xi _{D_{j},\lambda _{1}}}{\sqrt{\lambda _{1}}}+
	\frac{\Xi _{D_{j},\lambda _{1}}^{2}}{\lambda _{1}}\Big)^{2}\Big\|(L_{K}+
	\lambda _{1}I)^{-1/2}\Delta _{D_{j}}'\Big\|_{K}^{2}
\end{eqnarray}
and
%
%e5.9 #&#
\begin{equation}
	\nonumber
	\Big\|L_{K}^{1/2}\mathcal{Q}_{D_{j}}\Delta _{D_{j}}''\Big\|_{K}
	\lesssim \Big(\frac{\Xi _{D_{j},\lambda _{1}}}{\sqrt{\lambda _{1}}}+
	\frac{\Xi _{D_{j},\lambda _{1}}^{2}}{\lambda _{1}}\Big)\Big\|(L_{K}+
	\lambda _{1}I)^{-1/2}\Delta _{D_{j}}''\Big\|_{K}.
\end{equation}
Now we have obtained the error bounds that we need via the decomposition
for $\mathcal{S}_{2}$ in {\eqref{disdecoms2}}. We observe that the bounds
$\Big(\frac{\Xi _{D_{j},\lambda _{1}}}{\sqrt{\lambda _{1}}}+
\frac{\Xi _{D_{j},\lambda _{1}}^{2}}{\lambda _{1}}\Big)\Big\|(L_{K}+
\lambda _{1}I)^{-1/2}\Delta _{D_{j}}'\Big\|_{K}$ and
$\Big(\frac{\Xi _{D_{j},\lambda _{1}}}{\sqrt{\lambda _{1}}}+
\frac{\Xi _{D_{j},\lambda _{1}}^{2}}{\lambda _{1}}\Big)\Big\|(L_{K}+
\lambda _{1}I)^{-1/2}\Delta _{D_{j}}''\Big\|_{K}$ have already been handled
in \cite{sbsecond} for one-stage distributed learning theory. Therefore,
we obtain that when $|D_{j}|=|D|/m$, $\mathcal{S}_{2}$, in fact, shares
the same expected $L_{\rho_{X_\mu}}^2$-norm bounds with the bound of Corollary 3 in
\cite{sbsecond}. Namely,
%
%e5.10 #&#
\begin{equation}
	E_{\mathbf{z}^{|D|}}\Big[\big\|\mathcal{S}_{2}\big\|_{\rho}\Big]
	\lesssim \sqrt{\frac{\mathcal{N}(\lambda _{1})}{\lambda _{1}|D|}}
	\Big(\sqrt{\lambda _{1}}+
	\frac{m\|f_{\rho }-f_{\lambda _{1}}\|_{\rho }}{\sqrt{|D|\lambda _{1}}}
	\Big).
	\label{eq5.7}
\end{equation}
Applying {\eqref{xl}} with $|D_{j}|$ replaced by $|D|/m$, we know that
$\frac{\mathcal{N}(\lambda _{1})}{\lambda _{1}|D|}\lesssim |D|^{
	\frac{1-2r}{2r+\beta }}$, which shows
%
%e5.11 #&#
\begin{equation}
	\nonumber
	\sqrt{\frac{\mathcal{N}(\lambda _{1})}{\lambda _{1}|D|}}\lesssim |D|^{
		\frac{\frac{1}{2}-r}{2r+\beta }}.
\end{equation}
Recall the fact that
$\|f_{\lambda _{1}}-f_{\rho}\|_{\rho}\leq \lambda _{1}^{r}\|g_{\rho}
\|_{\rho}$, then we have
\begin{eqnarray}
	\nonumber
	&&E_{\mathbf{z}^{|D|}}\Big[\big\|\mathcal{S}_{2}\big\|_{\rho}\Big]
	\lesssim \sqrt{\frac{\mathcal{N}(\lambda _{1})}{\lambda _{1}|D|}}
	\lambda _{1}^{r}\Big(\lambda _{1}^{\frac{1}{2}-r}+
	\frac{m}{\sqrt{|D|\lambda _{1}}}\Big)
	\\
	\nonumber
	&&\quad \quad \quad \quad \quad \quad \lesssim |D|^{-
		\frac{r}{2r+\beta }}\Big(|D|^{-\frac{\frac{1}{2}-r}{2r+\beta }}|D|^{
		\frac{\frac{1}{2}-r}{2r+\beta }}+|D|^{\frac{2r-1}{2r+\beta }}|D|^{-
		\frac{1}{2}}|D|^{\frac{\frac{1}{2}}{2r+\beta }}|D|^{
		\frac{\frac{1}{2}-r}{2r+\beta }}\Big)
	\\
	\nonumber
	&&\quad \quad \quad \quad \quad \quad \lesssim |D|^{-
		\frac{r}{2r+\beta }}.
\end{eqnarray}

We finally estimate $\mathcal{S}_{3}$. Note that for any
$j=1,2,...,m$,
\begin{eqnarray}
	\nonumber
	&&\Big\|L_{K}^{1/2}(L_{K,D_{j}}+\lambda _{1}I+\lambda _{2}V_{D_{j}}^{T}V_{D_{j}})^{-1}
	\lambda _{2}V_{D_{j}}^{T}V_{D_{j}}f_{\lambda _{1}}\Big\|_{K}
	\\
	\nonumber
	&&\leq \lambda _{2}\Big\|V_{D_{j}}^{T}V_{D_{j}}\Big\|\Big\|L_{K}^{1/2}(
	\lambda _{1}I+L_{K})^{-1/2}\Big\|\Big\|(\lambda _{1}I+L_{K})^{1/2}(L_{K}+
	\lambda _{1}I+\lambda _{2}V_{D_{j}}^{T}V_{D_{j}})^{-1/2}\Big\|
	\\
	\nonumber
	&&\quad \Big\|(L_{K}+\lambda _{1}I+\lambda _{2}V_{D_{j}}^{T}V_{D_{j}})^{1/2}(L_{K,D_{j}}+
	\lambda _{1}I+\lambda _{2}V_{D_{j}}^{T}V_{D_{j}})^{-1/2}\Big\|
	\\
	\nonumber
	&&\quad \Big\|(L_{K,D_{j}}+\lambda _{1}I+\lambda _{2}V_{D_{j}}^{T}V_{D_{j}})^{-1/2}
	\Big\|\Big\|f_{\lambda _{1}}\Big\|_{K}
	\\
	\nonumber
	&&\leq \lambda _{2}c_{V}\sqrt{2}\mathcal{A}_{D_{j},\lambda _{1},
		\lambda _{2},V_{D_{j}}}^{1/2}\frac{1}{\sqrt{\lambda _{1}}}\|f_{
		\lambda _{1}}\|_{K}.
\end{eqnarray}
Taking expectations on both sides and noting from \cite{s4} the fact that
$\|f_{\lambda _{1}}\|_{K}\leq\kappa ^{2r-1}\|g_{\rho}\|_{\rho}$ when
$r\in [1/2,1]$, we have
\begin{eqnarray}
	\nonumber
	&&E_{\mathbf{z}^{|D_{j}|}}\bigg[\Big\|L_{K}^{1/2}(L_{K,D_{j}}+
	\lambda _{1}I+\lambda _{2}V_{D_{j}}^{T}V_{D_{j}})^{-1}\lambda _{2}V_{D_{j}}^{T}V_{D_{j}}f_{
		\lambda _{1}}\Big\|_{K}\bigg]
	\\
	\nonumber
	&&\lesssim E_{\mathbf{z}^{|D_{j}|}}\Big[\mathcal{A}_{D_{j},\lambda _{1},
		\lambda _{2},V_{D_{j}}}^{1/2}\Big]\lambda _{1}^{2r-\frac{1}{2}}
	\lesssim \bigg(
	\frac{2\mathcal{B}_{|D_{j}|,\lambda _{1}}}{\sqrt{\lambda _{1}}}+1
	\bigg)|D|^{-\frac{2r-\frac{1}{2}}{2r+\beta }}\lesssim |D|^{-
		\frac{r}{2r+\beta }},
\end{eqnarray}
in which the last inequality follows from $\frac{1}{2}\leq r\leq 1$. Similar
procedure implies
%
%e5.12 #&#
\begin{equation}
	\nonumber
	E_{\mathbf{z}^{|D|}}\bigg[\Big\|L_{K}^{1/2}(L_{K,D}+\lambda _{1}I+
	\lambda _{2}V_{D}^{T}V_{D})^{-1}\lambda _{2}V_{D}^{T}V_{D}f_{\lambda _{1}}
	\Big\|_{K}\bigg]\lesssim |D|^{-\frac{r}{2r+\beta }}.
\end{equation}
Finally, we obtain that
%
%e5.13 #&#
\begin{equation}
	\nonumber
	E_{\mathbf{z}^{|D|}}\Big[\big\|\mathcal{S}_{3}\big\|_{\rho}\Big]
	\lesssim |D|^{-\frac{r}{2r+\beta }}.
\end{equation}
Now combine above three estimates for $\mathcal{S}_{1}$,
$\mathcal{S}_{2}$, $\mathcal{S}_{3}$, we arrive at
%
%e5.14 #&#
\begin{equation}
	\nonumber
	E\Big[\big\|\overline{f_{\widehat{D},\lambda _{1},\lambda _{2}}}- f_{
		\widehat{D},\lambda _{1},\lambda _{2}}\big\|_{\rho}\Big]=\mathcal{O}(|D|^{-
		\frac{r}{2r+\beta }}).
\end{equation}
On the other hand, {Theorem~\ref{thm1}} has already shown
%
%e5.15 #&#
\begin{equation}
	\nonumber
	E\Big[\big\|f_{\widehat{D},\lambda _{1},\lambda _{2}}- f_{\rho}\big\|_{
		\rho}\Big]=\mathcal{O}(|D|^{-\frac{r}{2r+\beta }}).
\end{equation}
Minkowski inequality finally implies
\begin{eqnarray*}
	E\Big[\big\|\overline{f_{\widehat{D},\lambda _{1},\lambda _{2}}}-f_{
		\rho}\big\|_{\rho}\Big]=\mathcal{O}(|D|^{-\frac{r}{2r+\beta }}).
\end{eqnarray*}
The proof is complete.
\end{proof}

%s6 #&#
\section{Conclusion}
\label{sec6}

In this paper, we have proposed and studied the multi-penalty distribution
regression algorithm. The learning rates are systematically studied via
integral operator theory. The optimal rates are shown to be achievable
under appropriate conditions. Meanwhile, the learning rates in the hard
learning scenario of $f_{\rho}=L_{K}^{r}(g_{\rho})$, $r\in (0,1/2)$ for
some $ g_{\rho}\in L_{\rho _{X_{\mu}}}^{2}$ are first studied in the literature
on two-stage distribution regression. The results improve the existing
achievable rates in the literature. Moreover, based on the multi-penalty
distribution regression scheme we provided, we propose a new distributed
learning algorithm and derive optimal learning rates for it. It would be
interesting to extend our methods in this work to other settings in learning
theory, such as multi-kernel learning and deep neural network. The potential
application domains are expected to be explored.

%\section*{CRediT authorship contribution statement}

%DOCI
%
%\begin{conflict} % None declared.
%\end{conflict}

%
%\begin{dataavailability}[title=Data availability]
%\end{dataavailability}

%\section*{Declaration of generative AI and AI-assisted technologies in the writing process}

\section*{Acknowledgments}
 The authors would like to thank Professor Ding-Xuan Zhou for his valuable
suggestions on this work.This work was partially supported by the Research Grants Council of the Hong Kong Special Administrative Region, China ( CityU 11202819,11203521).

%\begin{appm}
%\def\the...{}
%\reset{}{}
%\appendix{}
%\appendix*{}
%\end{appm}%
%
%%********** End of text entry *****************%%
%

% structpyb loaded by jolita, 2023-11-23 07:05:54

%
%%\begin{cv}[idref=cv1] %% [fig=..,idref=cv1]
%%\textbf{}
%%\end{cv}
%

\end{document}